\documentclass{article}

\usepackage{microtype}
\usepackage{graphicx}
\usepackage{booktabs} 
\usepackage{multirow}
\usepackage{xcolor} 
\usepackage{array}
\newcommand{\ch}[1]{\ensuremath{\mathrm{#1}}}
\usepackage{adjustbox}
\usepackage{hyperref}
\newcommand{\flame}{FLAME}
\newcommand{\meanstd}[2]{#1 {\scriptsize$\pm$ #2}}

\usepackage[accepted]{icml2026}

\usepackage{amsmath}
\usepackage{amssymb}
\usepackage{amsthm}

\theoremstyle{plain}
\newtheorem{theorem}{Theorem}[section]
\newtheorem{proposition}[theorem]{Proposition}

\theoremstyle{definition}

\theoremstyle{remark}

\icmltitlerunning{FLAME: Physics-Guided Neural Operators for Onboard Methane Detection}

\begin{document}

\twocolumn[
  \icmltitle{
    \texorpdfstring{
      FLAME: Physics-Guided Neural Operators for Onboard Satellite \\ Methane Detection in Hyperspectral Imagery
    }{
      FLAME: Physics-Guided Neural Operators for Onboard Satellite Methane Detection in Hyperspectral Imagery
    }
  }



  \icmlsetsymbol{equal}{*}

  \begin{icmlauthorlist}
    \icmlauthor{Junhyuk Heo}{comp} \quad
    \icmlauthor{Junghwan Park}{comp} \quad
    \icmlauthor{Sangcheol Sim}{comp} \quad
    \icmlauthor{Beomkyu Choi}{comp} \quad
    \icmlauthor{Woojin Cho}{comp} 
  \end{icmlauthorlist}

  \icmlaffiliation{comp}{TelePIX, Seoul, Republic of Korea}

  \icmlcorrespondingauthor{Woojin Cho}{woojin.py@gmail.com}

  \icmlkeywords{Machine Learning, ICML}

  \vskip 0.3in
]



\printAffiliationsAndNotice{}  

\begin{abstract}
Methane is a major driver of near-term climate change, and rapidly identifying its emission sources is a critical climate intervention. Spaceborne hyperspectral imagery is the primary tool for this task, but the volume of data produced by each sensor makes ground-based detection impractical and necessitates onboard detection. Classical methods incur prohibitive computational cost on onboard hardware, while deep learning models are fast but fall short on detection quality. We propose FLAME, a physics-guided neural operator that builds the physics of methane absorption directly into its architecture. On the methane detection benchmark, FLAME achieves the highest detection accuracy among all evaluated methods, reduces the pixel-level false positive rate by nearly $3\times$ over the strongest neural baseline, uses the fewest parameters among learned baselines, and runs within the latency budget of onboard satellite hardware. Codes are available at \url{https://github.com/ROKMC1250/FLAME}.
\end{abstract}

\section{Introduction}
\label{sec:intro}

\begin{figure}[t!]
\centering
\includegraphics[width=\columnwidth]{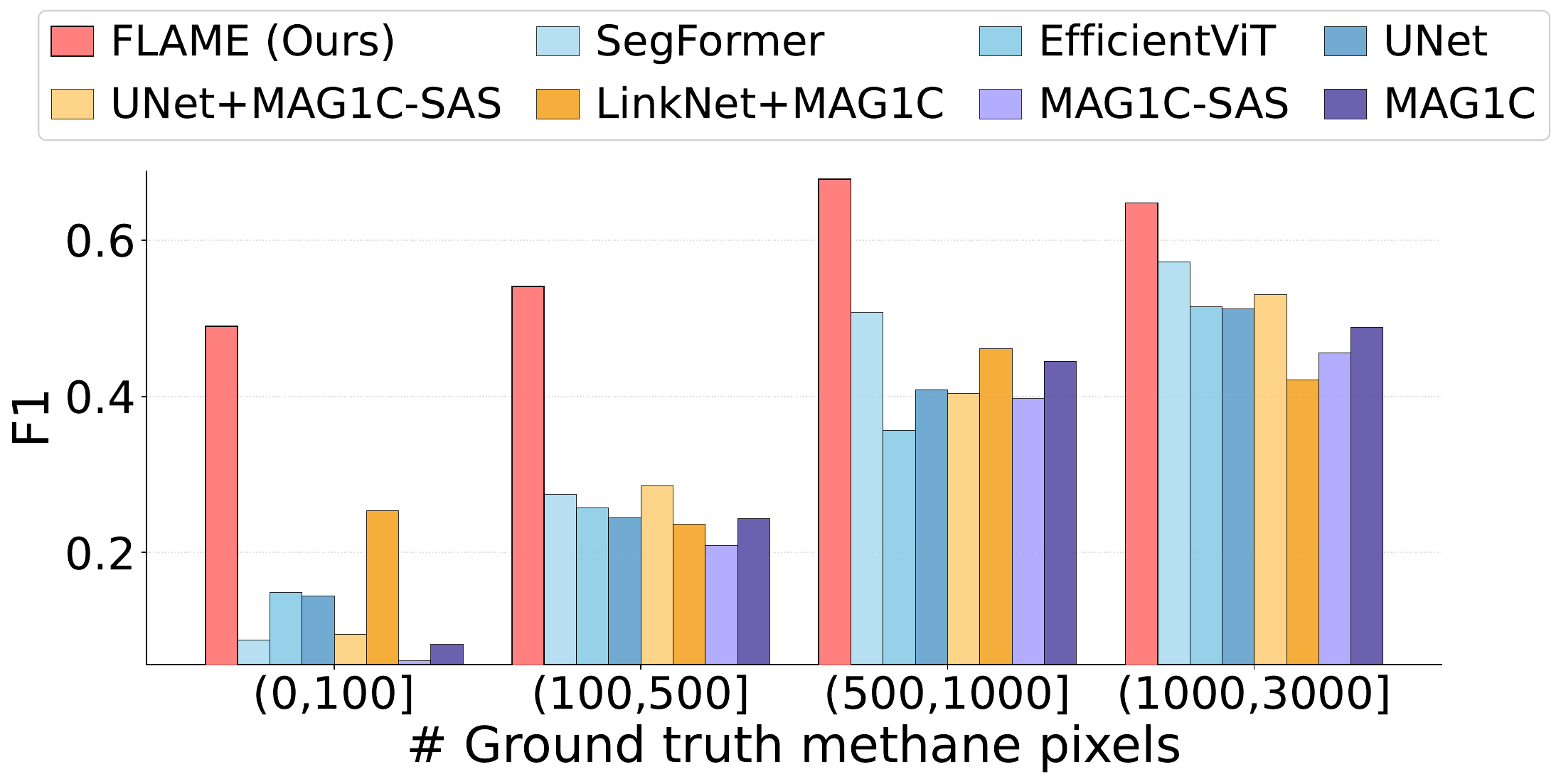}
\caption{\textbf{F1 score across plume sizes.} F1 score is measured on the STARCOP dataset, binned by the number of ground-truth plume pixels per tile. FLAME leads across all plume sizes, with the largest margin in the small-plume regime where weak signals are most easily confused with background.}
\label{fig:plume_analysis_a}
\end{figure}

Methane (\ch{CH_4}) is the second largest anthropogenic driver of
radiative forcing after carbon dioxide, with a global warming
potential roughly 80 times that of \ch{CO_2} over a $20$ year
horizon~\cite{ipcc2021ar6,etminan2016radiative,saunois2019global}.
Because its atmospheric lifetime is only about a decade, near-term
methane reductions translate directly into slowing of global
warming~\cite{shindell2012simultaneously,ocko2021acting}, making
mitigation one of the most cost-effective short-term climate levers.
A disproportionate share of anthropogenic methane comes from a
small number of point sources: large ultra-emitters in the oil and
gas sector account for a substantial fraction of sectoral
emissions~\cite{lauvaux2022global,alvarez2018assessment,duren2019california}.
Identifying and localising these sources is therefore high-leverage,
since a single repaired leak can eliminate emissions equivalent to
those of thousands of vehicles.

Spaceborne hyperspectral imagery offers a scalable tool for such
detection. Methane's distinctive shortwave-infrared absorption
features between $2122$ and $2488\,nm$ are captured by imaging
spectrometers such as AVIRIS-NG and NASA's
EMIT~\cite{thompson2016space}, and upcoming missions including
NASA's Surface Biology and Geology
(SBG)~\cite{cawse2021nasa} and ESA's Copernicus Hyperspectral
Imaging Mission (CHIME)~\cite{nieke2018towards} will extend this
toward near-global coverage. The resulting data volumes, however,
create a delivery bottleneck: hyperspectral tiles reach hundreds of
gigabytes per day per sensor, while downlink bandwidth is limited to
a few ground-station passes~\cite{langer2023robust}. Raw radiance is therefore processed on
the ground days after acquisition, by which time an ultra-emission
event may have vented for its entire duration.

Onboard processing directly addresses this latency bottleneck: detection on
the spacecraft means only plume-containing tiles need full
downlinking, and tasking signals can be issued in minutes. This is
no longer speculative. ESA's $\Phi$-Sat missions have demonstrated
onboard deep-learning inference in
orbit~\cite{giuffrida2021varphi,esposito2019orbit}, and embedded
GPU platforms such as NVIDIA Jetson are being evaluated as payload
computers~\cite{furano2020towards,yost2024state}. The design
constraint thus shifts to accuracy under tight power, memory, and
runtime budgets.

Existing detection methods sit uneasily within these constraints.
Classical matched filters model methane absorption through the
Beer--Lambert law and remain the standard tool for hyperspectral
methane retrieval~\cite{foote2020fast}, but their reliance on
tile- or column-wide background statistics leads to unstable
residuals over heterogeneous scenes and to costly iterative
covariance estimation. Purely data-driven segmentation networks are
fast, but they discard the physical structure of the problem and in
practice over-detect bright or textured surfaces that share no
absorption signature with methane. Hybrid two-stage pipelines retain
the physics but inherit the runtime of the matched filter they sit on
top of. None of these options is simultaneously accurate, fast, and
physically grounded.

We argue that the right place to inject physics is inside the
network, not around it. We propose FLAME (Fourier Learned
Absorption Matched Estimator), a physics-guided neural operator that
preserves the functional form of the log-domain matched filter while
replacing the quantities that make it fragile, namely the tile-wide
background and the global covariance, with pixel-wise estimates
produced by a Fourier-based neural operator. The detection score itself is
computed by a parameter-free inner-product layer, so the
Beer--Lambert structure is baked into the architecture rather than
learned from data. A scheduled auxiliary loss initially aligns the
learned score with a classical matched-filter product and then
decays, allowing the model to surpass its physics teacher while
retaining its inductive bias. 

\paragraph{Contributions.}
\begin{itemize}
    \item \textbf{Physics-guided neural operator.} We propose a new approach to methane detection that integrates the physical signal model into a neural operator, recasting retrieval as a function-to-function mapping rather than coupling physics and neural networks externally as in prior work.
    
    \item \textbf{Strong detection performance with a compact model.} On the STARCOP benchmark, FLAME achieves the highest F1 across all evaluated methods with the fewest parameters among learned baselines, and reduces the pixel-level false positive rate by nearly $3\times$ over the strongest neural baseline.
    
    \item \textbf{Onboard satellite deployment analysis.} We characterise inference time, power draw, and thermal behaviour on three NVIDIA Jetson modules representative of current and next-generation onboard satellite payload computers.
\end{itemize}

\section{Related Works}
\label{sec:related}

\paragraph{\textbf{Onboard Satellite Processing.}} Onboard machine learning for Earth observation has moved from
concept to orbit in recent years. The $\Phi$-Sat-1 mission
demonstrated onboard inference of a cloud-filtering CNN on a
dedicated AI accelerator~\cite{giuffrida2021varphi}, and in-orbit
hyperspectral experiments confirmed feasibility on CubeSat-class
platforms~\cite{esposito2019orbit}. Subsequent work has
benchmarked deep-learning stacks on representative edge
hardware~\cite{ziaja2021benchmarking} and examined the radiation,
power, and thermal constraints that shape payload computer
design~\cite{furano2020towards, yost2024state}. The broader vision
of responsive constellations, in which onboard detections trigger
follow-up tasking across multiple assets, has been articulated for
both disaster response and greenhouse gas
monitoring~\cite{ruuvzivcka2022ravaen,parr2024live}. Within methane
detection specifically, onboard constraints have driven work on
accelerating classical retrievals~\cite{herec2025optimizing} and on
compact learned models with reduced input
bandwidth~\cite{ruuvzivcka2025hyperspectralvits}. FLAME is designed
with these constraints as first-class concerns, trading
matched-filter iteration for a single-pass neural operator.

\paragraph{\textbf{Methane Detection from Hyperspectral Imagery.}} The matched filter has been the standard tool for hyperspectral
methane retrieval for over a decade. This is because methane is
characterized by narrow absorption features in the SWIR range, which
hyperspectral sensors can resolve with dense contiguous bands, whereas
multispectral imagery averages radiance over a small number of broad
bands and can blur these spectral cues. As illustrated in Figure~\ref{fig:ch4_spectrum},
CH$_4$ exhibits distinctive absorption structure across this
region. The matched filter models each pixel as a tile mean perturbed
by a methane-specific absorption signature under the Beer--Lambert
law, whitens the residual by the tile covariance, and projects onto
the methane target
spectrum~\cite{thompson2015real,manolakis2013detection,funk2002clustering}.
MAG1C~\cite{foote2020fast} is the current workhorse in this family,
combining an albedo correction with an iteratively-reweighted $\ell_1$
penalty at the cost of repeated covariance inversions, and variants
extend it with scene-specific enhancement
spectra~\cite{ayasse2019methane}.
Learning-based methods extend the matched-filter family in two directions.
Two-stage pipelines such as HyperSTARCOP~\cite{ruuvzivcka2023semantic}
refine matched-filter score maps with a U-Net~\cite{ronneberger2015u}, and lightweight
variants pair LinkNet~\cite{chaurasia2017linknet} with accelerated
filters such as MAG1C-SAS~\cite{herec2025optimizing}. Such pipelines
inherit the detection ceiling of the underlying filter, since any
plume the filter misses cannot be recovered downstream. End-to-end
alternatives such as MethaneMapper~\cite{kumar2023methanemapper} and
HyperspectralViTs~\cite{ruuvzivcka2025hyperspectralvits}, which adapt
SegFormer~\cite{xie2021segformer} and
EfficientViT~\cite{cai2023efficientvit} to raw hyperspectral input,
remove the matched filter entirely but discard the absorption prior
encoded in the Beer--Lambert law. FLAME keeps the matched-filter
score as an architectural component while replacing its fragile
tile-level quantities with learned pixel-wise estimates.

\begin{figure}[t]
\centering
\includegraphics[width=\columnwidth]{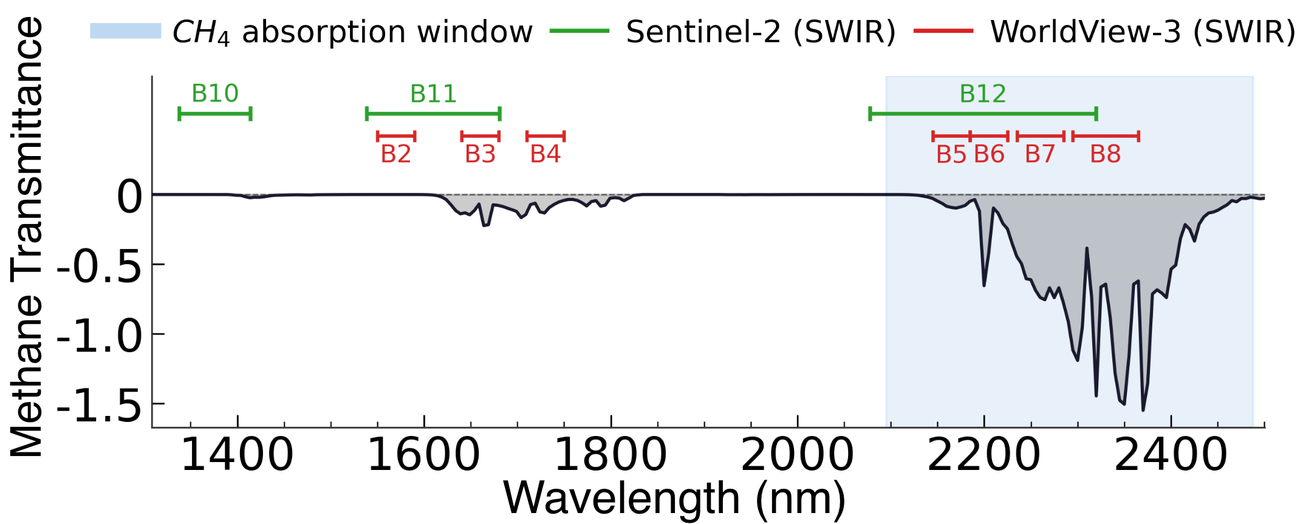}
\caption{Methane absorption spectrum in the SWIR window relevant to plume detection.
The broad SWIR bands of multispectral sensors (Sentinel-2, WorldView-3)
cannot resolve the narrow absorption features near $2122-2488\,nm$.}\label{fig:ch4_spectrum}
\end{figure}

\paragraph{\textbf{Neural Operators for Satellite Data.}} Neural operators learn mappings between function spaces rather than between fixed-dimensional tensors, which aligns well with the continuous spatial and spectral structure of remote-sensing data~\cite{kovachki2023neural}. The Fourier Neural Operator (FNO) parameterises a global convolution in the Fourier domain and has achieved strong results on parametric PDEs~\cite{li2020fourier} and global weather forecasting~\cite{pathak2022fourcastnet,kurth2023fourcastnet}. U-FNO~\cite{wen2022u} adds a U-shaped multi-scale backbone to spectral mixing, producing compact models that capture both global and local structure. Applications to remote sensing so far include tasks such as soil organic carbon estimation~\cite{wong2023image} and climate downscaling~\cite{yang2024fourier}. Our motivation for using a neural operator is that hyperspectral scenes exhibit long-range spatial correlations that global spectral convolutions capture more naturally than local kernels, and that the two unobserved quantities in the log-domain matched filter are both maps from an input function to an output function over the same domain, which matches the native form of neural-operator approximation. FLAME instead integrates the matched-filter signal model into the operator's output, preserving its inner-product structure as an architectural invariant.

\section{Preliminaries}
\label{sec:prelim}
This section reviews the matched filter for methane detection. We first introduce the Beer--Lambert model in the radiance domain and the classical matched filter derived from its linearization, and then present the log-domain formulation that eliminates the linear approximation. The log-domain derivation motivates the FLAME architecture introduced in Section~\ref{sec:flame}.

\subsection{Matched Filter under the Beer--Lambert Law}
\label{sec:mf_prelim}

A hyperspectral sensor measures radiance across $p$ discrete SWIR bands, so we work with $\mathbb{R}^p$-valued spectra throughout. Let $L_i \in \mathbb{R}^p$ denote the radiance observed at pixel $i$ over the $p$ bands, and let $L_i^{\mathcal{B}} \in \mathbb{R}^p$ denote the corresponding methane-free background radiance that would be observed in the absence of a plume. Under the Beer--Lambert law, the two quantities are related through the methane column enhancement $\alpha_i \ge 0$ and the fixed unit absorption spectrum $s \in \mathbb{R}^p$, as follows.
\begin{equation}
    L_i = L_i^{\mathcal{B}} \odot \exp(\alpha_i \, s),
    \label{eq:beer_lambert}
\end{equation}
where $\odot$ denotes element-wise multiplication and $\exp(\cdot)$ is applied element-wise. Following the sign convention of \citet{foote2020fast}, we define $s$ such that $s_j < 0$ in bands of strong methane absorption, which absorbs the sign of the linearization into the spectrum. Classical matched filters operate directly in the radiance domain. Assuming that $\alpha_i \, s$ is small element-wise, a first-order Taylor expansion of Eq.~\eqref{eq:beer_lambert} yields the following approximation, where $L_i^{\mathcal{B}}$ is replaced by the tile-shared sample mean $\mu \in \mathbb{R}^p$ in the second line.
\begin{equation}
    \begin{aligned}
        L_i &\approx L_i^{\mathcal{B}} + \alpha_i \, (L_i^{\mathcal{B}} \odot s), \\
        L_i &\approx \mu + \alpha_i \, (\mu \odot s) + \eta,
    \end{aligned}
    \label{eq:radiance_linearization}
\end{equation}
where $\eta$ is an additive noise term.
Defining the target spectrum $t = \mu \odot s$, where $C \in \mathbb{R}^{p \times p}$ is the background covariance estimated empirically from the tile, the matched-filter estimate of $\alpha_i$ is given as follows.
\begin{equation}
    \hat\alpha_i^{\mathrm{MF}} = \frac{(L_i - \mu)^\top C^{-1}(\mu \odot s)}{(\mu \odot s)^\top C^{-1}(\mu \odot s)}.
    \label{eq:mf_estimate}
\end{equation}

MAG1C~\citep{foote2020fast} extends this estimator with two additional components: an albedo correction $r_i = L_i^\top \mu / \mu^\top \mu$ that compensates for the multiplicative nature of the Beer-Lambert attenuation across pixels of varying brightness, and an iteratively reweighted $\ell_1$ regularization that suppresses already-detected methane pixels during background re-estimation. At iteration $k$, the update is given as follows.
\begin{equation}
    \hat\alpha_i^{(k)} = \max\left( \frac{(L_i - \mu^{(k)})^\top v^{(k)} - w_i^{(k)}}{r_i^{(k)} \cdot m^{(k)}}, \; 0 \right),
    \label{eq:mag1c_iter}
\end{equation}
where $v^{(k)} = (C^{(k)})^{-1}(\mu^{(k)} \odot s)$ is the whitened target and $m^{(k)} = (\mu^{(k)} \odot s)^\top v^{(k)}$ is the normalization constant, both re-estimated after subtracting previous methane estimates, and $w_i^{(k)} = 1 / (r_i^{(k)} (\hat\alpha_i^{(k-1)} + \delta))$ is an adaptive sparsity term with stabilization constant $\delta > 0$.

Three limitations follow from this formulation. First, a single tile-level $\mu$ cannot represent heterogeneous surfaces. Second, computing $v^{(k)}$ requires an explicit covariance inversion $(C^{(k)})^{-1}$, which dominates the runtime on onboard satellite hardware. Third, iteration is unavoidable because the initial $\mu$ is contaminated by methane pixels and must be re-estimated after suppressing previous detections.

\subsection{Log-Domain Matched Filter}
\label{sec:log_mf}
The linear approximation in Eq.~\eqref{eq:radiance_linearization} can be avoided by operating in log space. Taking the element-wise natural logarithm of Eq.~\eqref{eq:beer_lambert} yields the following exact linear relation, to which we attach an additive noise term $\eta_i$ that captures sensor noise and unmodeled deviations.
\begin{equation}
    \ell_i = \ell_i^{\mathcal{B}} + \alpha_i \, s + \eta_i,
    \label{eq:log_beer_lambert}
\end{equation}
where $\ell_i = \log L_i \in \mathbb{R}^p$ is the log-radiance and $\ell_i^{\mathcal{B}} = \log L_i^{\mathcal{B}} \in \mathbb{R}^p$ is the log-background, with $\log(\cdot)$ applied element-wise. Unlike Eq.~\eqref{eq:radiance_linearization}, this relation holds without approximation and remains valid for strong plumes in which $\alpha_i \, s$ is not small. Crucially, the Beer--Lambert attenuation is additive in the log domain rather than multiplicative, so the background radiance no longer scales the target spectrum; as a consequence, the albedo correction $r_i$ present in Eq.~\eqref{eq:mf_estimate} is not needed. The covariance of $\eta_i$, denoted $\mathcal{C}_i \in \mathbb{R}^{p \times p}$, is in general pixel-dependent because sensor noise scales with radiance and its log-space counterpart is therefore heteroscedastic.

Under the observation model in Eq.~\eqref{eq:log_beer_lambert}, the weighted least-squares estimator of $\alpha_i$ is given as follows.
\begin{equation}
    \hat\alpha_i^{\mathrm{log}} = \frac{(\ell_i - \ell_i^{\mathcal{B}})^\top \, \mathcal{C}_i^{-1} \, s}{s^\top \, \mathcal{C}_i^{-1} \, s}.
    \label{eq:log_mf_estimator}
\end{equation}
The component that determines whether a pixel is detected is the inner product in the numerator, which can be written band-wise as follows.
\begin{equation}
    (\ell_i - \ell_i^{\mathcal{B}})^\top \, \mathcal{C}_i^{-1} \, s = \sum_{j=1}^{p} \big[\ell_i - \ell_i^{\mathcal{B}}\big]_j \cdot \big[\mathcal{C}_i^{-1} \, s\big]_j,
    \label{eq:log_mf_numerator}
\end{equation}
where $[\,\cdot\,]_j$ denotes the $j$-th SWIR-band component. This expression is a spectrally weighted projection of the log-residual onto the methane spectrum, whitened by the inverse noise covariance. Two unobserved quantities appear: the log-background $\ell_i^{\mathcal{B}}$, and the whitened target $\mathcal{C}_i^{-1} s$. FLAME replaces both with learned estimates while preserving the inner-product structure of Eq.~\eqref{eq:log_mf_numerator}.

\section{Proposed Methods: FLAME}
\label{sec:flame}
Building on the log-domain formulation in Section~\ref{sec:log_mf}, FLAME estimates the log-background and the whitened target in Eq.~\eqref{eq:log_mf_numerator} using a shared neural operator followed by two lightweight heads.
\subsection{Formulation}
\label{sec:flame_formulation}
We view each hyperspectral tile as a spectral field $\ell : \Omega \to \mathbb{R}^p$ defined over the spatial domain $\Omega \subset \mathbb{R}^2$, with $\ell(x) \in \mathbb{R}^p$ the log-radiance vector at spatial location $x$. FLAME replaces the two unobserved quantities in Eq.~\eqref{eq:log_mf_numerator} with predictions of the same functional form---a log-background field $\hat\ell^{\mathcal{B}} : \Omega \to \mathbb{R}^p$ and a spectral-weight field $\hat w : \Omega \to \mathbb{R}^p$---so that detection is cast as a mapping between spectral fields over the same spatial domain. On the discrete $H \times W$ grid, we identify fields with their evaluations at pixel locations and write $\ell_i = \ell(x_i)$, $\hat\ell_i^{\mathcal{B}} = \hat\ell^{\mathcal{B}}(x_i)$, and $\hat w_i = \hat w(x_i)$.
The backbone $\Phi(\,\cdot\,;\theta_\Phi)$ is a neural operator that maps the input field to a latent feature field; its evaluation at pixel $i$ is given as follows.
\begin{equation}
    \mathbf{z}_i = \big[\Phi(\ell; \theta_\Phi)\big](x_i) \in \mathbb{R}^d,
    \label{eq:backbone}
\end{equation}
with feature dimension $d$. Two point-wise heads $f_{bg}$ and $f_{sw}$ then lift the feature vector to the two physical fields, evaluated pixel-wise as follows.
\begin{equation}
    \hat\ell_i^{\mathcal{B}} = f_{bg}(\mathbf{z}_i; \theta_{bg}) \in \mathbb{R}^p, \quad
    \hat w_i = f_{sw}(\mathbf{z}_i; \theta_{sw}) \in \mathbb{R}^p.
    \label{eq:bg_head}
\end{equation}
Here $\theta_\alpha = (\theta_\Phi, \theta_{bg}, \theta_{sw})$ collects the parameters that enter the physics-guided score, $\theta_{sg}$ parameterizes the segmentation head introduced in Section~\ref{sec:heads}, and $\Theta = (\theta_\alpha, \theta_{sg})$ denotes the full parameter set. The head $f_{bg}$ approximates the log-background $\ell_i^{\mathcal{B}}$ pixel-wise. The head $f_{sw}$ approximates the whitened target through the parameterization $\mathcal{C}_i^{-1} s \to \hat w_i \odot s$, in which the fixed absorption spectrum $s$ is retained as a multiplicative physical prior. Substituting Eq.~\eqref{eq:bg_head} and this parameterization into the log-MF numerator in Eq.~\eqref{eq:log_mf_numerator} yields the FLAME detection score, which we refer to as the physics-guided score and define as follows.
\begin{equation}
    \hat\alpha_i(\ell; \theta_\alpha) = \big(\ell_i - \hat\ell_i^{\mathcal{B}}\big)^\top \big(\hat w_i \odot s\big).
    \label{eq:flame_score}
\end{equation}
Although $\hat\ell_i^{\mathcal{B}}$ and $\hat w_i$ appear without explicit arguments in Eq.~\eqref{eq:flame_score}, both depend on the entire input field $\ell$ through $\Phi$, giving the score a non-local receptive field. The full pipeline runs in a single forward pass that avoids the iterative covariance inversion of Eq.~\eqref{eq:mag1c_iter}, and strictly contains the classical log-domain matched filter as a special case, as shown in Appendix~\ref{app:proposition}.
\begin{figure}[t!]
\centering
\includegraphics[width=0.99\columnwidth]{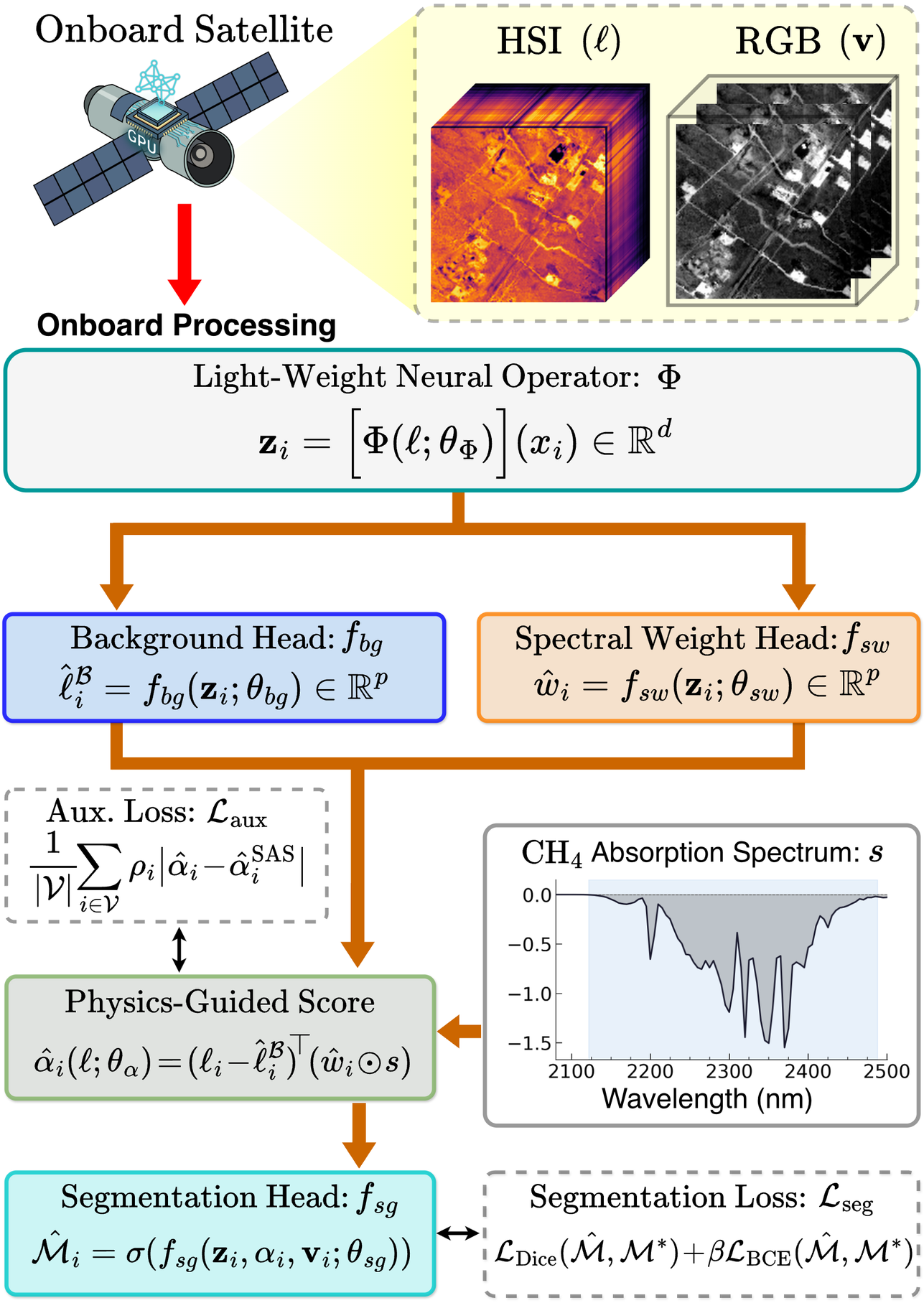}
\caption{\textbf{Overview of FLAME.} A neural operator predicts pixel-wise backgrounds and spectral weights, combined with the fixed \ch{CH_4} prior to yield a physics-guided score and plume mask.}\label{fig:model}
\end{figure}

\subsection{Heads and Score Layer}
\label{sec:heads}

\paragraph{Log-background head.}
The log-background head $f_{bg}$ is implemented as a single $1 \times 1$ convolution, $\hat\ell_i^{\mathcal{B}} = W_{bg} \mathbf{z}_i + b_{bg}$ with $W_{bg} \in \mathbb{R}^{p \times d}$, $b_{bg} \in \mathbb{R}^p$, and parameters $\theta_{bg} = (W_{bg}, b_{bg})$. The bias $b_{bg}$ is initialized to the training-set mean log-spectrum, so the model starts from the shared-mean solution assumed by the classical matched filter and then departs to a pixel-wise estimate as learning proceeds.

\paragraph{Spectral weight head.}
The spectral weight head $f_{sw}$ is also a $1 \times 1$ convolution followed by a softplus activation to enforce non-negativity, $\hat w_i = \mathrm{softplus}(W_{sw} \mathbf{z}_i + b_{sw})$ with $\theta_{sw} = (W_{sw}, b_{sw})$. $\hat w_i$ represents a pixel-adaptive generalization of the diagonal part of $\mathcal{C}_i^{-1}$. Off-diagonal correlations in $\mathcal{C}_i$ are not modeled inside $f_{sw}$, which outputs a vector; instead, they are absorbed by the backbone $\Phi$, which mixes all bands globally before $f_{sw}$ operates on $\mathbf{z}_i$. The explicit factor $s$ in Eq.~\eqref{eq:flame_score} guarantees that the effective spectral weighting aligns with the methane absorption prior regardless of the learned $\hat w_i$.

\paragraph{Score layer.}
The score layer is parameter-free and implements Eq.~\eqref{eq:flame_score} directly. Because Eq.~\eqref{eq:log_beer_lambert} is exact, the score remains structurally unbiased across the full range of plume intensities, including strong emissions for which the radiance-domain linearization in Eq.~\eqref{eq:radiance_linearization} would underestimate $\alpha_i$. For notational convenience, throughout the remainder of the paper we let $\hat\alpha_i$ denote the score after a fixed normalization by a constant $\tau > 0$ and clipping to $[0, \tau_{\max}]$, where $\tau$ is determined from log-domain MAG1C-SAS statistics on the training set. The same normalization is applied to the MAG1C-SAS reference $\hat\alpha_i^{\mathrm{SAS}}$ used for auxiliary supervision in Section~\ref{sec:training}.

\paragraph{Segmentation head.}
The segmentation head $f_{sg}$ incorporates spatial context and implicitly handles the albedo correction $r_i$, the normalization $m$, and the sparsity suppression that appear in Eq.~\eqref{eq:mag1c_iter}. It consumes the concatenation $\mathbf{u}_i = [\mathbf{z}_i, \hat\alpha_i, \mathbf{v}_i] \in \mathbb{R}^{d+4}$ of the backbone features, the normalized score, and the normalized RGB triplet $\mathbf{v}_i \in \mathbb{R}^3$. $f_{sg}$ consists of two $3 \times 3$ convolutional blocks with batch normalization and ReLU, followed by a $1 \times 1$ projection to a single logit. The predicted plume probability is $\hat{\mathcal{M}}_i = \sigma(f_{sg}(\mathbf{u}_i; \theta_{sg}))$, where $\sigma$ denotes the sigmoid function.

\subsection{Training}
\label{sec:training}

All parameters $\Theta = (\theta_\alpha, \theta_{sg})$ are trained jointly with the objective
\begin{equation}
    \mathcal{L}(\Theta) = \mathcal{L}_{\mathrm{seg}}(\Theta) + \gamma(t) \cdot \mathcal{L}_{\mathrm{aux}}(\theta_\alpha),
    \label{eq:total_loss}
\end{equation}
where $t$ indexes the training epoch. The segmentation loss $\mathcal{L}_{\mathrm{seg}}$ supervises the predicted plume mask, while the auxiliary loss $\mathcal{L}_{\mathrm{aux}}$ aligns the physics-guided score with a classical matched-filter teacher and is annealed by $\gamma(t)$.

\paragraph{Segmentation loss.}
We combine Dice loss with positive-weighted binary cross-entropy.
\begin{equation}
    \mathcal{L}_{\mathrm{seg}} = \mathcal{L}_{\mathrm{Dice}}(\hat{\mathcal{M}}, \mathcal{M}^*) + \beta \cdot \mathcal{L}_{\mathrm{BCE}}(\hat{\mathcal{M}}, \mathcal{M}^*),
    \label{eq:seg_loss}
\end{equation}
with $\beta = \min(n_-/n_+, 50)$, where $n_+$ and $n_-$ are the positive and negative pixel counts.

\begin{table*}[t!]
\centering
\caption{\textbf{Quantitative Comparison on the STARCOP test set}. Score columns are reported as mean $\pm$ standard deviation. Best results are in bold, second best are underlined. For Pixel FPR, Time, and Params, lower is better. Pixel FPR is reported as $\times 10^{-4}$. Backbones and decoders are abbreviated in the Method column as MobileNetV2 (MNv2), ResNet18 (R18), ConvUp (CU), and ConvUpStride (CUS). Inference time is measured on an RTX 4090 with batch size 1.}
\label{tab:main}
\small
\setlength{\tabcolsep}{5pt}
\begin{tabular}{ll cccc c cc}
\toprule
& Method & F1 & IoU & Precision & Recall & Pixel FPR & Time (ms) & Params \\
\midrule
\multirow{5}{*}{\rotatebox{90}{\scriptsize Classical}} & CEM & 0.178 & 0.098 & 0.114 & 0.399 & 74.0 & 19.6 & -- \\
 & MF & 0.178 & 0.098 & 0.115 & 0.398 & 74.0 & 21.0 & -- \\
 & ACE & 0.154 & 0.084 & 0.111 & 0.254 & 49.0 & 32.5 & -- \\
 & MAG1C-tile & 0.300 & 0.176 & 0.228 & 0.437 & 35.0 & 343.7 & -- \\
 & MAG1C-SAS & 0.284 & 0.166 & 0.194 & 0.528 & 52.0 & 116.2 & -- \\
\midrule
\multirow{4}{*}{\rotatebox{90}{\scriptsize Two-stage}} & UNet+MAG1C-tile & \meanstd{0.477}{0.036} & \meanstd{0.314}{0.032} & \meanstd{0.320}{0.034} & \meanstd{\underline{0.946}}{0.013} & \meanstd{49.0}{7.0} & 348.1 & 6.60M \\
 & UNet+MAG1C-SAS & \meanstd{0.441}{0.020} & \meanstd{0.283}{0.017} & \meanstd{0.300}{0.022} & \meanstd{0.844}{0.042} & \meanstd{48.0}{8.0} & 120.6 & 6.60M \\
 & LinkNet+MAG1C-tile & \meanstd{0.402}{0.026} & \meanstd{0.252}{0.021} & \meanstd{0.256}{0.023} & \meanstd{\textbf{0.947}}{0.018} & \meanstd{67.0}{9.0} & 347.5 & \underline{0.85M} \\
 & LinkNet+MAG1C-SAS & \meanstd{0.397}{0.017} & \meanstd{0.248}{0.013} & \meanstd{0.260}{0.017} & \meanstd{0.846}{0.027} & \meanstd{58.0}{7.0} & 120.0 & \underline{0.85M} \\
\midrule
\multirow{8}{*}{\rotatebox{90}{\scriptsize End-to-end}} & UNet (MNv2) & \meanstd{0.421}{0.007} & \meanstd{0.267}{0.006} & \meanstd{0.337}{0.023} & \meanstd{0.570}{0.042} & \meanstd{27.0}{5.0} & \underline{4.9} & 6.65M \\
 & UNet (R18) & \meanstd{0.309}{0.020} & \meanstd{0.183}{0.014} & \meanstd{0.218}{0.015} & \meanstd{0.535}{0.032} & \meanstd{46.0}{3.0} & \textbf{4.1} & 14.54M \\
 & SegFormer (base) & \meanstd{0.393}{0.055} & \meanstd{0.246}{0.044} & \meanstd{0.286}{0.085} & \meanstd{0.736}{0.129} & \meanstd{52.0}{23.0} & 6.3 & 3.82M \\
 & SegFormer (CU) & \meanstd{\underline{0.515}}{0.009} & \meanstd{\underline{0.347}}{0.008} & \meanstd{\underline{0.417}}{0.012} & \meanstd{0.679}{0.059} & \meanstd{\underline{23.0}}{3.0} & 7.3 & 4.30M \\
 & SegFormer (CUS) & \meanstd{0.449}{0.022} & \meanstd{0.290}{0.019} & \meanstd{0.314}{0.038} & \meanstd{0.831}{0.116} & \meanstd{46.0}{14.0} & 20.3 & 4.30M \\
 & EfficientViT (base) & \meanstd{0.345}{0.035} & \meanstd{0.209}{0.026} & \meanstd{0.222}{0.021} & \meanstd{0.785}{0.135} & \meanstd{66.0}{12.0} & 7.1 & 4.81M \\
 & EfficientViT (CU) & \meanstd{0.404}{0.003} & \meanstd{0.253}{0.002} & \meanstd{0.262}{0.003} & \meanstd{0.880}{0.009} & \meanstd{59.0}{1.0} & 7.6 & 4.85M \\
 & EfficientViT (CUS) & \meanstd{0.414}{0.077} & \meanstd{0.264}{0.060} & \meanstd{0.279}{0.063} & \meanstd{0.829}{0.041} & \meanstd{55.0}{17.0} & 8.2 & 4.85M \\
\midrule
& \flame{} (Ours) & \meanstd{\textbf{0.608}}{0.005} & \meanstd{\textbf{0.437}}{0.005} & \meanstd{\textbf{0.651}}{0.056} & \meanstd{0.576}{0.035} & \meanstd{\textbf{8.0}}{2.0} & 6.2 & \textbf{0.78M} \\
\bottomrule
\end{tabular}
\end{table*}

\paragraph{Curriculum auxiliary supervision.}
Since the score map is initially uninformative, we stabilize early training by matching $\hat\alpha_i$ against precomputed MAG1C-SAS outputs $\hat\alpha_i^{\mathrm{SAS}}$ as follows.
\begin{equation}
    \mathcal{L}_{\mathrm{aux}}(\theta_\alpha) = \frac{1}{|\mathcal{V}|} \sum_{i \in \mathcal{V}} \rho_i \cdot \big| \hat\alpha_i(\ell; \theta_\alpha) - \hat\alpha_i^{\mathrm{SAS}} \big|,
    \label{eq:aux_loss}
\end{equation}
where $\mathcal{V}$ is the set of valid pixels and $\rho_i = 1 + \lambda_\rho \hat\alpha_i^{\mathrm{SAS}}$ with $\lambda_\rho = 10$. The auxiliary weight follows a half-period cosine decay $\gamma(t) = \frac{1}{2}(1 + \cos(\pi \min(t/T, 1)))$ with $T = 10$. This induces two phases: in Phase 1 $\mathcal{L}_{\mathrm{aux}}$ drives optimization, aligning the backbone with the matched filter, and in Phase 2 $\mathcal{L}_{\mathrm{seg}}$ takes over, driving the model beyond the MAG1C-SAS teacher.

\section{Experiments}

We evaluate FLAME along four axes. First, we benchmark detection accuracy and efficiency against classical, two-stage, and end-to-end baselines on the STARCOP test set. Second, we qualitatively characterize FLAME against representative baselines across plume regimes. Third, we ablate the physics-guided score to isolate its contribution. Fourth, we assess deployability by profiling inference on hardware representative of onboard satellite platforms.

\subsection{Experimental Setup}
\label{sec:setup}
\paragraph{Dataset.}
We evaluate on the STARCOP dataset \cite{ruuvzivcka2023semantic}, which pairs AVIRIS-NG hyperspectral tiles from the Permian Basin campaign with pixel-level methane plume annotations curated on $512 \times 512$ chips. We use 72 SWIR bands in the 2122--2488~nm range together with three RGB bands, follow the original train/test split, and report all metrics on the held-out test set of 342 tiles.
\paragraph{Evaluation Metrics.}
We report pixel-level recall, precision, F1, and IoU, together with the pixel-level false positive rate to quantify spurious detections on regions without plumes. Inference time is measured on a single NVIDIA RTX 4090 with batch size $1$.
\paragraph{Implementation Details.}
We train FLAME with a curriculum that first aligns the physics-guided score with a matched-filter teacher and then transitions to the segmentation objective, and report mean and standard deviation over 3 random seeds. Full optimizer settings, training schedule, and loss configuration are provided in Appendix~\ref{app:impl}.

\begin{figure*}[t]
\centering
\setlength{\tabcolsep}{1pt}
\begin{tabular}{@{}ccccc@{}}
\scriptsize RGB & \scriptsize UNet+MAG1C-SAS & \scriptsize SegFormer (CU) & \scriptsize \flame{} (Ours) & \scriptsize GT \\[2pt]
\includegraphics[width=0.19\textwidth]{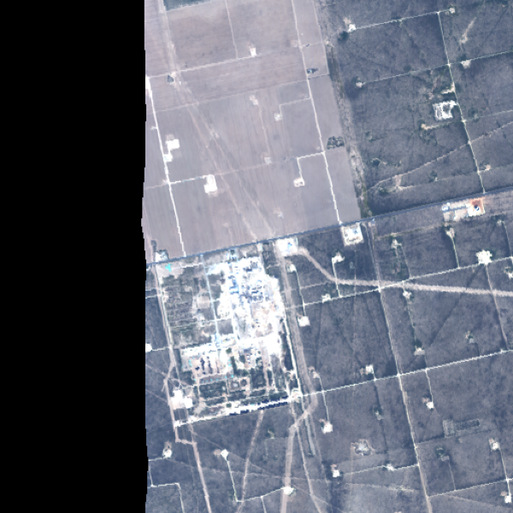} &
\includegraphics[width=0.19\textwidth]{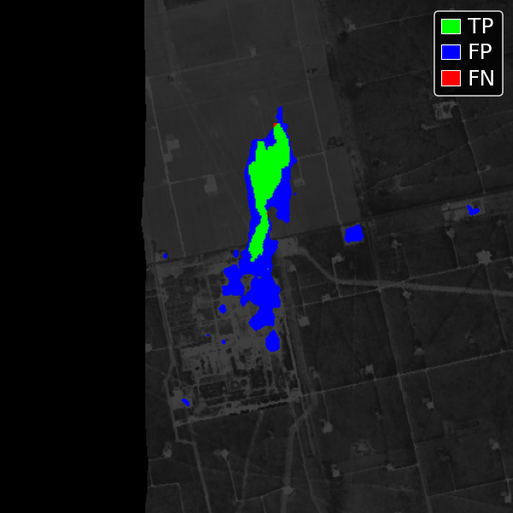} &
\includegraphics[width=0.19\textwidth]{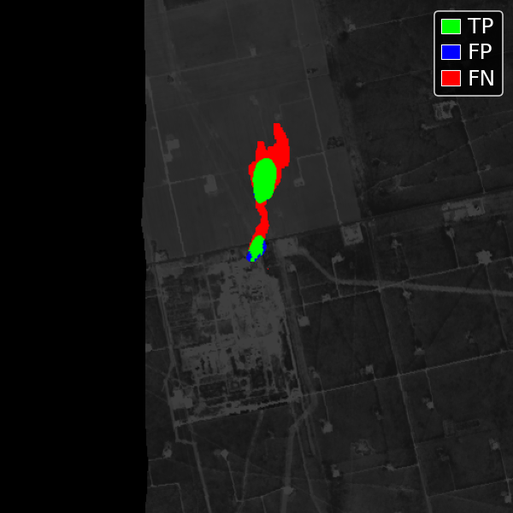} &
\includegraphics[width=0.19\textwidth]{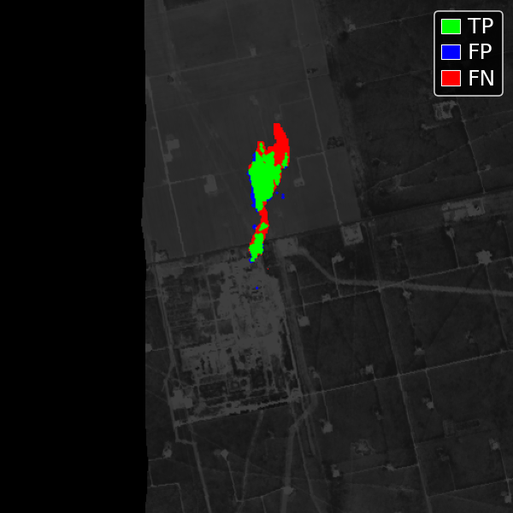} &
\includegraphics[width=0.19\textwidth]{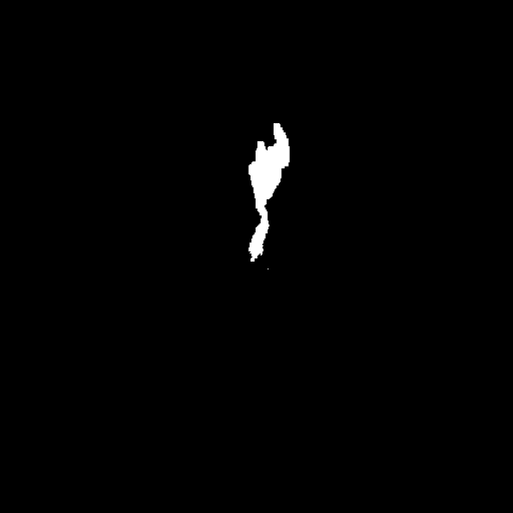} \\[2pt]
\includegraphics[width=0.19\textwidth]{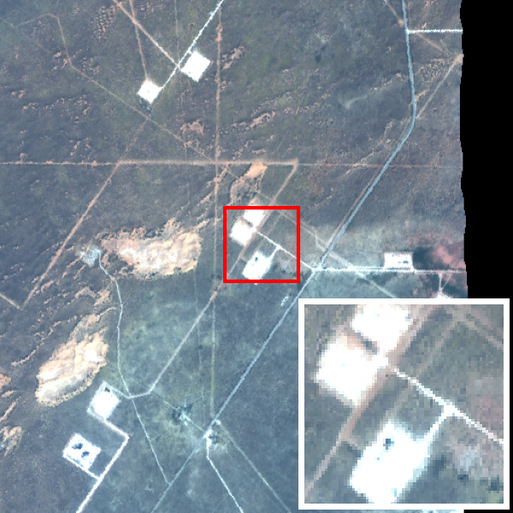} &
\includegraphics[width=0.19\textwidth]{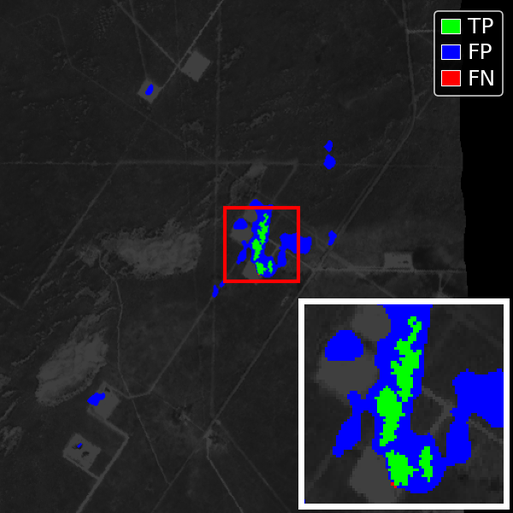} &
\includegraphics[width=0.19\textwidth]{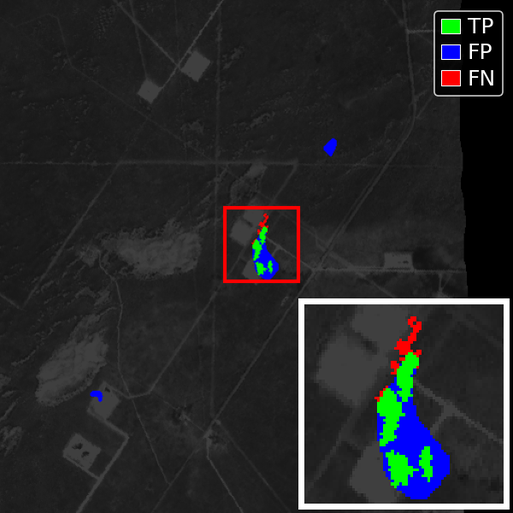} &
\includegraphics[width=0.19\textwidth]{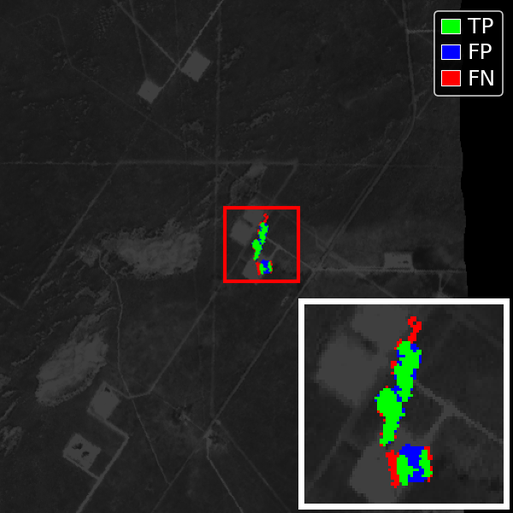} &
\includegraphics[width=0.19\textwidth]{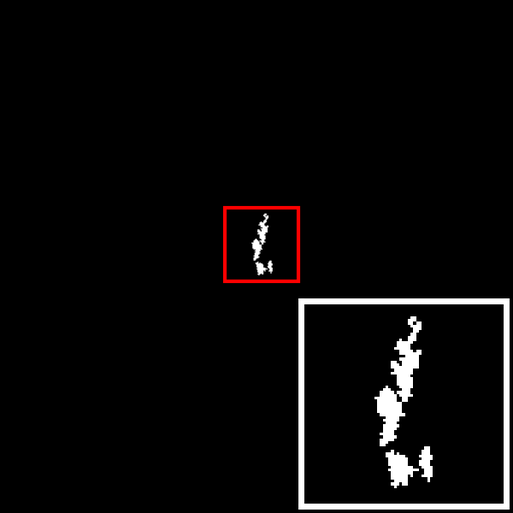} \\[2pt]
\includegraphics[width=0.19\textwidth]{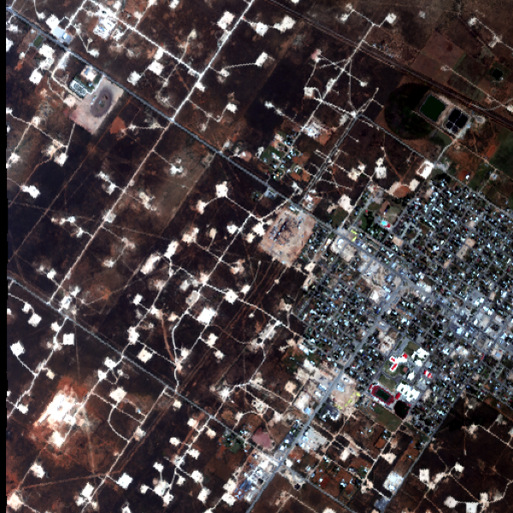} &
\includegraphics[width=0.19\textwidth]{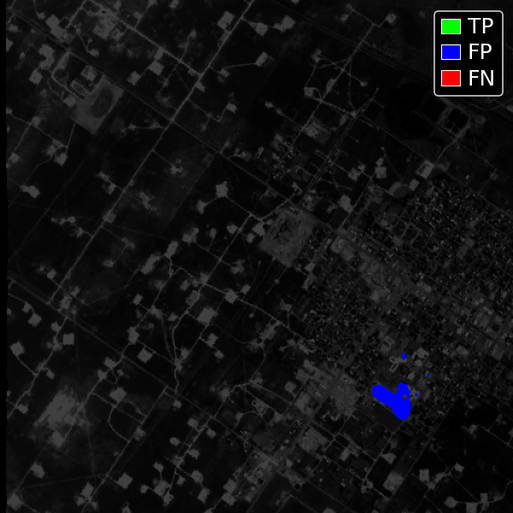} &
\includegraphics[width=0.19\textwidth]{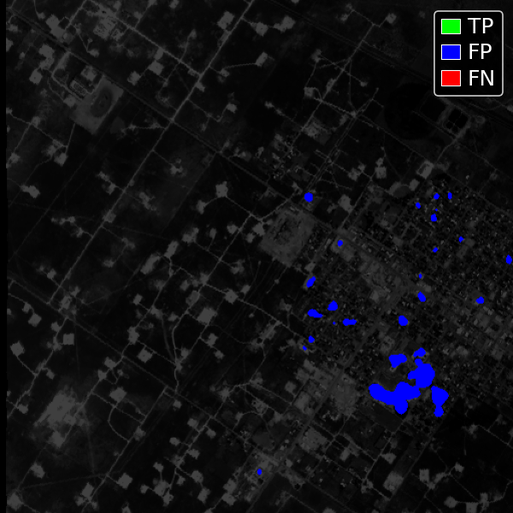} &
\includegraphics[width=0.19\textwidth]{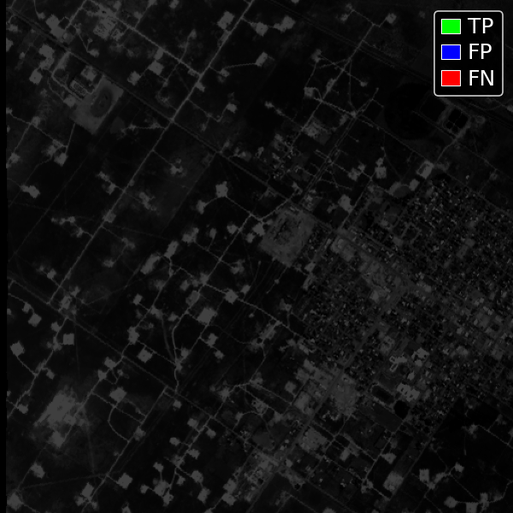} &
\includegraphics[width=0.19\textwidth]{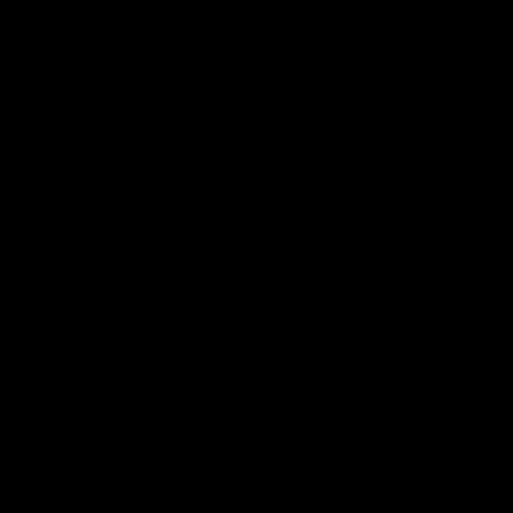} \\
\end{tabular}
\caption{\textbf{Qualitative Comparison on the STARCOP test set}. Three test tiles representing a strong-plume regime (top), a weak-plume regime (middle), and a plume-free regime (bottom), shown for UNet+MAG1C-SAS, SegFormer with the ConvUp decoder, and FLAME. Prediction panels overlay true positives in green, false positives in blue, and false negatives in red on a darkened RGB backdrop. The ground-truth column shows the binary plume mask.}
\label{fig:qualitative2}
\end{figure*}

\paragraph{Baselines.}
We compare against three categories of methods. Classical matched filters include CEM~\cite{kraut2005adaptive}, MF~\cite{manolakis2002detection,funk2002clustering}, ACE~\cite{chang2000constrained}, MAG1C~\cite{foote2020fast}, and MAG1C-SAS~\cite{herec2025optimizing}. Two-stage pipelines combine a matched filter with a segmentation network: the UNet~\cite{ronneberger2015u} variants paired with MAG1C products follow the HyperSTARCOP design introduced in~\cite{ruuvzivcka2023semantic}, while the LinkNet~\cite{chaurasia2017linknet} variants paired with MAG1C-SAS are taken from~\cite{herec2025optimizing}. End-to-end models process raw hyperspectral data without a matched filter, and we evaluate SegFormer~\cite{xie2021segformer} and EfficientViT~\cite{cai2023efficientvit} under the HyperspectralViTs adaptations proposed in~\cite{ruuvzivcka2025hyperspectralvits}, along with UNet~\cite{ronneberger2015u} baselines. Classical matched-filter scores are reproduced from~\cite{herec2025optimizing}, while all learned baselines are trained under the same protocol as FLAME.

\subsection{Comparison with Existing Methods}
\label{sec:fpr_analysis}
We compare FLAME against the classical, two-stage, and end-to-end baselines introduced in Section~\ref{sec:setup}.

\paragraph{Detection performance.}
FLAME achieves the highest F1 of 0.608 and IoU of 0.437 in Table~\ref{tab:main}, improving over the strongest two-stage pipeline UNet+MAG1C-tile by 13.1 F1 points while eliminating the matched-filter computation, and over the strongest end-to-end baseline SegFormer with the ConvUp decoder by 9.3 F1 points. Its precision of 0.651, the highest among all methods, indicates substantially cleaner masks than any learned alternative. FLAME's recall of 0.576 is lower than that of the two-stage pipelines, which exceed 0.84, but the high recall of those pipelines is achieved by predicting overly large plume masks rather than by recovering more true methane pixels, as the qualitative comparison in Figure~\ref{fig:qualitative2} makes visible.

\paragraph{False positive behavior.}
FLAME's pixel-level false positive rate of $8\times 10^{-4}$ is the lowest among learned methods, roughly threefold below the next best end-to-end model SegFormer with the ConvUp decoder at $2.3\times 10^{-3}$ and an order of magnitude below two-stage pipelines, whose rates span $4.8\times 10^{-3}$ to $6.7\times 10^{-3}$. This advantage is preserved across plume sizes, with the smallest degradation on tiles containing fewer than 100 plume pixels. The corresponding per-bin F1 in Figure~\ref{fig:plume_analysis_a} shows that FLAME also leads in F1 across all plume-size bins, so the low FPR reflects cleaner predictions rather than conservative thresholding.

\paragraph{Efficiency trade-off.}
FLAME runs in 6.2~ms per tile on an RTX 4090, competitive to the fastest end-to-end models and roughly 18 times faster than MAG1C-SAS. With 0.78M parameters, it is smaller than LinkNet's 0.85M and roughly an order of magnitude smaller than U-Net variants. A bubble chart visualization of the joint speed, accuracy, and size trade-off is provided in Appendix~\ref{app:fpr}.

\subsection{Qualitative Analysis}
\label{sec:qualitative}
Figure~\ref{fig:qualitative2} compares FLAME against the two-stage baseline UNet+MAG1C-SAS and the end-to-end baseline SegFormer on three tiles spanning a strong plume, a weak plume, and an urban scene without plumes.

In the strong-plume tile in the top row, all three methods localize the plume core, but the failure modes differ. UNet+MAG1C-SAS scatters false positives across the surrounding area, SegFormer under-detects portions of the plume itself, shown as red false negatives, and FLAME produces a tight mask that aligns with the plume boundary. 

The contrast widens in the weak-plume tile in the middle row, where the two baselines exhibit opposing failure modes. UNet+MAG1C-SAS spreads false positives across the surrounding region, while SegFormer misses much of the plume. The inset of the middle row makes this gap visible at higher resolution, where FLAME tracks the fine, fragmented structure of the plume that the baselines either oversmooth into surrounding clutter or fail to detect. This pattern is consistent with FLAME's pixel-wise modeling of the background and noise statistics. Classical pipelines share a single background and covariance across the tile, so pixels whose true background deviates from the tile mean leave residuals that are easily mistaken for weak methane signals. FLAME instead predicts $\hat\ell^B_i$ and $\hat w_i$ per pixel, so background variations are absorbed into the background model itself and no longer leak into the score for weak plumes. The log-domain formulation further removes the first-order Taylor approximation of the Beer--Lambert law that classical matched filters rely on, ensuring that strong plumes are also treated under the same exact signal model as weak ones. The urban tile without plumes in the bottom row confirms this. Both baselines generate spatially extended false detections on building rooftops and other man-made structures whose spectral signatures are easily confused with methane absorption, while FLAME suppresses the response and yields a clean prediction.


These observations explain the F1 gain and pixel-FPR reduction reported in Table~\ref{tab:main}, as FLAME suppresses the excessive false positives of two-stage pipelines while avoiding the incomplete plume masks often produced by end-to-end models. Further qualitative results, including per-pixel probability maps for every method, are provided in Appendix~\ref{app:additional_qual}.

\begin{table}[t]
\centering
\caption{\textbf{Effect of the physics-guided score}. The w/o physics-guide row ablates the score layer, with backbone features passed directly to the segmentation head. The w/ physics-guide row corresponds to the full FLAME model. All metrics are pixel-level on the STARCOP test set from a single training seed.}
\label{tab:ablation}
\footnotesize
\begin{tabular}{@{}l c c c c@{}}
\toprule
Variant & F1 & IoU & Precision & Recall \\
\midrule
w/o physics-guide & 0.409 & 0.257 & 0.437 & 0.384 \\
w/\phantom{o} physics-guide & \textbf{0.603} & \textbf{0.432} & \textbf{0.596} & \textbf{0.611} \\
\bottomrule
\end{tabular}%
\end{table}

\subsection{Effect of the Physics-Guided Score}
The physics-guided score in Eq.~\eqref{eq:flame_score} is the core inductive bias of FLAME. To isolate its contribution, we compare the full model against an ablation that removes the score layer entirely, so the backbone features feed directly into the segmentation head without the physics signal $\hat\alpha_i$.

Table~\ref{tab:ablation} shows that adding the physics-guided score raises F1 from 0.409 to 0.603 and IoU from 0.257 to 0.432. The improvement is larger than the gap between FLAME and any learned baseline in Table~\ref{tab:main}, indicating that the parameter-free log-domain inner product, rather than the backbone capacity alone, is what drives FLAME's detection performance. Appendix~\ref{app:physics_guided} further decomposes the score into its learned components and analyzes how it is shaped by the training curriculum and refined by the segmentation head.

\subsection{Onboard Satellite Deployment}
\label{sec:edge}
We profile FLAME on three NVIDIA Jetson platforms representative of the hardware currently considered for onboard satellite deployment: Jetson Orin NX 16GB, Jetson AGX Orin 64GB, and Jetson AGX Thor. Each platform runs FLAME with identical trained weights on a fixed set of STARCOP tiles in FP32 precision, with the full measurement protocol described in Appendix~\ref{sec:edge_hardware}.

Table~\ref{tab:jetson} reports the measurements. Jetson Orin NX processes one $512 \times 512$ tile in 225.5~ms at 9.5~W, within the 15~W cap of this module. Jetson AGX Orin reduces the latency to 117.6~ms at 7.0~W of actual draw, less than a quarter of the 30~W cap, so the workload leaves substantial compute headroom on this class of module. Jetson AGX Thor reaches 36.4~ms at 59.5~W, roughly half of the 120~W cap. All three platforms thermally equilibrate between 41.5 and 51.8~$^\circ$C, indicating sustained operation rather than transient peak performance. These power envelopes fall within the budgets typical of small-satellite payload computers, making FLAME directly deployable on the onboard hardware class that current and upcoming hyperspectral missions are expected to carry.
\begin{table}[t]
\centering
\caption{\textbf{Onboard inference on Jetson platforms}. Time is reported per tile, and TDP is the catalog power cap of each module.}
\label{tab:jetson}
\footnotesize
\setlength{\tabcolsep}{2pt}
\begin{tabular}{@{}l rrrr@{}}
\toprule
Platform & Time (ms) & Power (W) & Temp ($^\circ$C) & TDP (W) \\
\midrule
Jetson Orin NX   & 225.5 & 9.5  & 51.8 & 15  \\
Jetson AGX Orin  & 117.6 & 7.0  & 47.5 & 30  \\
Jetson AGX Thor  & 36.4  & 59.5 & 41.5 & 120 \\
\bottomrule
\end{tabular}
\end{table}

\section{Conclusion}

In this work, we proposed FLAME, a physics-guided neural operator for onboard methane plume detection in hyperspectral imagery. FLAME preserves the log-domain matched-filter structure of methane retrieval while replacing fragile tile-level statistics with pixel-wise log-background and spectral-weight estimates learned by a compact Fourier-based neural operator. By embedding the Beer--Lambert absorption prior directly into a parameter-free score layer, FLAME combines the physical reliability of classical matched filters with the speed and adaptability of learned segmentation models. Experiments on the STARCOP benchmark show that this design improves detection accuracy, substantially reduces false positives, and remains lightweight enough for onboard satellite deployment. These results suggest that physics-guided neural operators provide a practical path toward accurate, efficient, and deployable methane monitoring from future hyperspectral satellite missions.

\paragraph{Impact Statements}
This work targets faster onboard detection of methane point sources from satellite hyperspectral imagery. Since methane is a potent short-lived greenhouse gas driven by a few emitters, locating these sources cuts the delay from emission to response from days to minutes.



\bibliography{icml2026}
\bibliographystyle{icml2026}

\newpage
\appendix
\onecolumn

\section{FLAME Generalizes the Log-Domain Matched Filter}
\label{app:proposition}
In Section~\ref{sec:flame} we claimed that FLAME strictly contains the classical log-domain matched filter as a special case of its hypothesis class. This appendix formalizes the claim and exhibits an explicit parameter setting that realizes the reduction.
\begin{proposition}[FLAME generalizes the log-domain matched filter]
\label{prop:flame_generalizes_mf}
Fix $\bar\ell \in \mathbb{R}^p$ and a vector $\omega^2$ with strictly positive entries, and define the diagonal log-noise covariance $C = \mathrm{diag}(\omega^2)$. There exists a score-level parameter setting $\theta_\alpha = (\theta_\Phi, \theta_{bg}, \theta_{sw})$ such that, for every log-radiance tile $\ell$ and every pixel $i$,
\begin{equation}
    \hat\alpha_i(\ell;\theta_\alpha) \;=\; (\ell_i - \bar\ell)^\top C^{-1} s,
    \label{eq:app_reduction}
\end{equation}
which is the detection-relevant inner product in the numerator of Eq.~\eqref{eq:log_mf_estimator} under a tile-shared background $\ell_i^{B} \equiv \bar\ell$ and a pixel-independent diagonal covariance $C_i \equiv C$.
\end{proposition}

\begin{proof}
We construct an explicit parameter setting that realizes the reduction. Set
\begin{equation}
    W_{bg} = 0, \quad b_{bg} = \bar\ell, \quad W_{sw} = 0, \quad b_{sw} = \log\!\big(\exp(1/\omega^2) - \mathbf{1}\big),
    \label{eq:app_construction}
\end{equation}
where all operations on $\omega^2$ are element-wise. The choice of $b_{sw}$ is well-defined because softplus is a continuous, strictly increasing bijection from $\mathbb{R}$ onto $(0,\infty)$ and $\omega^2$ has strictly positive entries. This setting yields
\begin{equation}
    \hat\ell_i^B = \bar\ell, \qquad \hat w_i = 1/\omega^2,
    \label{eq:app_heads}
\end{equation}
for every pixel. Any $\theta_\Phi$ is admissible because both heads are independent of $z_i$. Substituting Eq.~\eqref{eq:app_heads} into Eq.~\eqref{eq:flame_score} gives
\begin{equation}
    \hat\alpha_i(\ell;\theta_\alpha) = (\ell_i - \bar\ell)^\top \big((1/\omega^2) \odot s\big),
    \label{eq:app_substitution}
\end{equation}
and since $C^{-1} s = (1/\omega^2) \odot s$, this is Eq.~\eqref{eq:app_reduction}.
\end{proof}

\paragraph{Strict containment.}
The inclusion is strict because any $W_{bg} \neq 0$ produces a log-background that varies across pixels and cannot be reproduced by any tile-shared $\bar\ell$, and any $W_{sw} \neq 0$ yields a pixel-dependent spectral weight that no single diagonal $C$ can match. Together with the single-pass evaluation of the score layer in Eq.~\eqref{eq:flame_score}, these two extensions address the three limitations identified in Section~\ref{sec:mf_prelim} in a one-to-one manner, with a learned $\hat\ell_i^{B}$ replacing the tile-shared background, a learned $\hat w_i$ replacing the pixel-independent covariance, and a single forward pass replacing the iterative refinement of Eq.~\eqref{eq:mag1c_iter}.
\paragraph{Scalar normalization and exactness.}
The full estimator of Eq.~\eqref{eq:log_mf_estimator} differs from Eq.~\eqref{eq:app_reduction} by the data-independent scalar $1/(s^\top C^{-1} s)$, which rescales the detection map uniformly and does not affect which pixels are detected. In FLAME this global rescaling is applied through the fixed $\tau$-normalization of Section~\ref{sec:heads}. Crucially, Proposition~\ref{prop:flame_generalizes_mf} rests on the exact identity in Eq.~\eqref{eq:log_beer_lambert} rather than on the first-order linearization in Eq.~\eqref{eq:radiance_linearization}, so the matched-filter behavior FLAME inherits at $\theta_\alpha$ remains unbiased across the full range of plume intensities, including the strong ultra-emitters for which the radiance-domain linearization is known to underestimate $\alpha_i$.

\clearpage
\section{Dataset and Implementation Details}
\label{app:impl}

\paragraph{Dataset.}
We followed the train/test partition released with STARCOP~\cite{ruuvzivcka2023semantic}, using $3{,}425$ tiles for training and $342$ tiles for held-out evaluation, with no separate validation split. Plume tiles and tiles without plumes are nearly balanced at the tile level, but the positive class is sparse at the pixel level, with $0.26\%$ of training pixels carrying the plume label, rising to $0.52\%$ within tiles that contain plumes. The test split follows the same regime. All tiles are stored as $512\times512$ rasters, and AVIRIS-NG scene crops smaller than this footprint are zero-padded, with the padded margins excluded from training and evaluation through a per-pixel valid mask. \autoref{tab:dataset_stats} summarizes the tile- and pixel-level class balance of each split.

\begin{table}[h]
\centering
\caption{Tile- and pixel-level statistics for the STARCOP partition used in this work. The Plume \% column reports the fraction of pixels labelled positive among pixels that are not padded, and the value in parentheses restricts this fraction to tiles that contain plumes.}
\label{tab:dataset_stats}
\begin{tabular}{lcccc}
\toprule
Split & Tiles & Plume tiles & No-plume tiles & Plume \% (within plume tiles) \\
\midrule
Train & 3{,}425 & 1{,}712 & 1{,}713 & 0.26\% (0.52\%) \\
Test  & 342    & 166    & 176    & 0.21\% (0.44\%) \\
\bottomrule
\end{tabular}
\end{table}

\paragraph{Model architecture.}
The FLAME backbone is a U-FNO~\cite{wen2022u} with hidden width $d=14$ and $L_F=3$ FNO blocks followed by $L_U=3$ U-FNO blocks. Each spectral mixing layer retains $m_1=m_2=12$ Fourier modes per spatial axis. \autoref{tab:flame_arch} summarizes these settings.

\begin{table}[h]
\centering
\caption{Architectural hyperparameters of FLAME.}
\label{tab:flame_arch}
\begin{tabular}{lc}
\toprule
Hyperparameter & Value \\
\midrule
Backbone hidden width $d$ & 14 \\
Fourier modes $m_1, m_2$ & 12, 12 \\
FNO blocks $L_F$ & 3 \\
U-FNO blocks $L_U$ & 3 \\
\bottomrule
\end{tabular}
\end{table}

\paragraph{Training.}
We trained FLAME for 50 epochs with AdamW~\cite{loshchilov2017decoupled}, using initial learning rate $\eta_0 = 2\times 10^{-3}$ and weight decay $10^{-4}$, under a cosine schedule decaying to $\eta_{\min} = 10^{-6}$, with gradient clipping to $\ell_2$-norm $1.0$ and mixed-precision training. The global batch size was 24 and inputs were processed at the native $512\times512$ resolution. We applied random horizontal and vertical flips and uniform $\{0,90,180,270\}^{\circ}$ rotations jointly to the input maps, the binary label, and the cached MAG1C-SAS teacher map, and the valid mask $m$ was broadcast against every loss term and metric counter so that padded margins did not contribute to gradient or evaluation. The BCE component of the segmentation loss is reweighted per minibatch with $\beta = \min(N_{-}/N_{+},\, 50)$, where $N_{+}$ and $N_{-}$ are the positive and negative pixel counts under the valid mask, and the Dice term is masked by $m$ before summing. We set the score normalization constant to $\tau = 1750$. For FLAME and every learned baseline we used a fixed decision rule throughout training and inference, $\sigma(\hat{y}) > 0.5$ followed by a $3\times3$ cross-shaped morphological opening, without per-model threshold tuning. All learned models were trained on three NVIDIA RTX 4090 GPUs.

\paragraph{Baselines.}
We compare against three categories of methods. The classical matched filters use the canonical method-specific thresholds of~\cite{herec2025optimizing}, with $0.004$ for CEM~\cite{kraut2005adaptive} and MF~\cite{manolakis2002detection,funk2002clustering}, $0.03$ for ACE~\cite{chang2000constrained}, and $300$ for both MAG1C~\cite{foote2020fast} and MAG1C-SAS~\cite{herec2025optimizing}. The two-stage pipelines, namely UNet paired with MAG1C and LinkNet paired with MAG1C-SAS, retain the HyperSTARCOP normalization protocol of~\cite{herec2025optimizing}. The end-to-end baselines are UNet with a MobileNet-V2 encoder~\cite{sandler2018mobilenetv2}, SegFormer-B0~\cite{xie2021segformer}, and EfficientViT-B1~\cite{cai2023efficientvit}. All learned baselines share FLAME's training data, augmentation pipeline, valid-pixel mask, decision threshold, and three-seed protocol, with each retaining the optimizer recipe and learning-rate scale of its underlying architecture. Table~\ref{tab:baseline_hp} lists the resulting settings.

\begin{table}[h]
\centering
\caption{Hyperparameters of the end-to-end learned baselines trained in our environment. All models are initialized randomly and share FLAME's training recipe except for the initial learning rate.}
\label{tab:baseline_hp}
\begin{tabular}{lcccc}
\toprule
Method & Backbone & Initial LR & Epochs & Batch \\
\midrule
UNet                  & MobileNet-V2          & $2.0\times10^{-3}$ & 50 & 24 \\
SegFormer             & MiT-B0                & $1.2\times10^{-4}$ & 50 & 24 \\
EfficientViT          & EfficientViT-B1       & $1.0\times10^{-3}$ & 50 & 24 \\
\midrule
FLAME (ours)          & UFNO + dual heads     & $2.0\times10^{-3}$ & 50 & 24 \\
\bottomrule
\end{tabular}
\end{table}

\paragraph{Identical evaluation.}
Every row of Table~\ref{tab:main} is evaluated on the same $342$-tile STARCOP test partition described in Table~\ref{tab:dataset_stats}, with the same pixel-level metric definitions for recall, precision, F1, IoU, and FPR computed against all valid background pixels, and with the same per-row threshold rule, namely the canonical thresholds of~\cite{herec2025optimizing} for classical filters and the $\sigma>0.5$ rule followed by morphological opening for segmentation networks. As a self-consistency check, we verified that our local re-runs of the classical filters reproduce the reference F1 values of~\cite{herec2025optimizing} within $\pm 0.005$.

\section{Neural Operator Backbone Comparison}
\label{app:backbone}
\begin{table}[ht]
\centering
\caption{Neural operator backbone comparison within the FLAME physics-head framework. The physics heads, segmentation head, and training schedule are kept fixed, so this comparison isolates the effect of the operator backbone from the physics-guided formulation. Pixel FPR is reported as $\times 10^{-4}$.}
\label{tab:ablation2}
\small
\begin{tabular}{@{}l c c c c c c r@{}}
\toprule
Variant & F1 & IoU & Precision & Recall & Pixel FPR & Params (M) & Time (ms) \\
\midrule
\flame{} (UFNO) & \textbf{0.603} & \textbf{0.432} & 0.596 & \underline{0.611} & 10.0 & \underline{0.78} & \underline{6.2} \\
FNO & 0.531 & 0.361 & \underline{0.604} & 0.474 & \underline{7.0} & 0.79 & \textbf{5.6} \\
F-FNO & 0.356 & 0.216 & 0.371 & 0.342 & 14.0 & 2.39 & 18.1 \\
Tucker-FNO & \underline{0.571} & \underline{0.399} & 0.506 & \textbf{0.654} & 15.0 & 2.16 & 16.9 \\
UNO & 0.263 & 0.151 & 0.240 & 0.291 & 22.0 & 2.54 & 8.2 \\
WNO & 0.517 & 0.349 & \textbf{0.629} & 0.440 & \textbf{6.0} & \textbf{0.77} & 26.7 \\
\bottomrule
\end{tabular}
\end{table}

\paragraph{Backbone comparison.}
We compare different neural-operator backbones while keeping the physics heads, segmentation head, and training schedule fixed. As shown in Table~\ref{tab:ablation2}, UFNO leads in F1 and IoU while remaining among the most parameter-efficient and fastest options. WNO matches its parameter count and achieves the highest precision and lowest pixel FPR, but its inference time is roughly four times higher. FNO is the fastest at the cost of seven F1 points, and the larger backbones, namely F-FNO, Tucker-FNO, and UNO, increase parameter count and latency without a corresponding improvement in detection performance. UFNO therefore offers the best joint trade-off across detection accuracy, parameter count, and inference speed, which is why we adopt it as the FLAME backbone.

\paragraph{Backbone configurations.}
Table~\ref{tab:backbone_arch} lists the architectural hyperparameters of each backbone evaluated in Table~\ref{tab:ablation2}. All variants share the FLAME physics heads, segmentation head, training schedule, and input pipeline, so the metrics in Table~\ref{tab:ablation2} isolate the effect of the backbone design. For FLAME (UFNO), the value $3 + 3$ in the Layers column denotes three FNO blocks followed by three U-FNO blocks.

\begin{table}[ht]
\centering
\caption{Architectural configuration of each neural operator backbone in Table~\ref{tab:ablation2}.}
\label{tab:backbone_arch}
\small
\begin{tabular}{@{}l c c l l@{}}
\toprule
Variant & Layers & Width & Modes / sub-bands & Notes \\
\midrule
FLAME (UFNO) & $3 + 3$ & 14 & $m_1 = m_2 = 12$ & SE block and mini-UNet local path \\
FNO & 4 & 14 & $m_1 = m_2 = 32$ & canonical formulation \\
F-FNO & 8 & 48 & $m = 32$ per axis & factorized 1D spectral convolution \\
Tucker-FNO & 4 & rank 64 & $m = 32$ per axis & Tucker decomposition \\
UNO & 7 & $w = 8, f = 0.75$ & $m = [24, 14, 8, 8, 8, 14, 22]$ & U-shape channel schedule \\
WNO & 3 & 14 & db6, level 6 & DWT sub-band weighting \\
\bottomrule
\end{tabular}
\end{table}

\clearpage

\section{Additional Quantitative Results}
\label{app:fpr}

\begin{figure}[ht]
\centering
\begin{minipage}[t]{0.48\columnwidth}
    \centering
    \includegraphics[width=\linewidth]{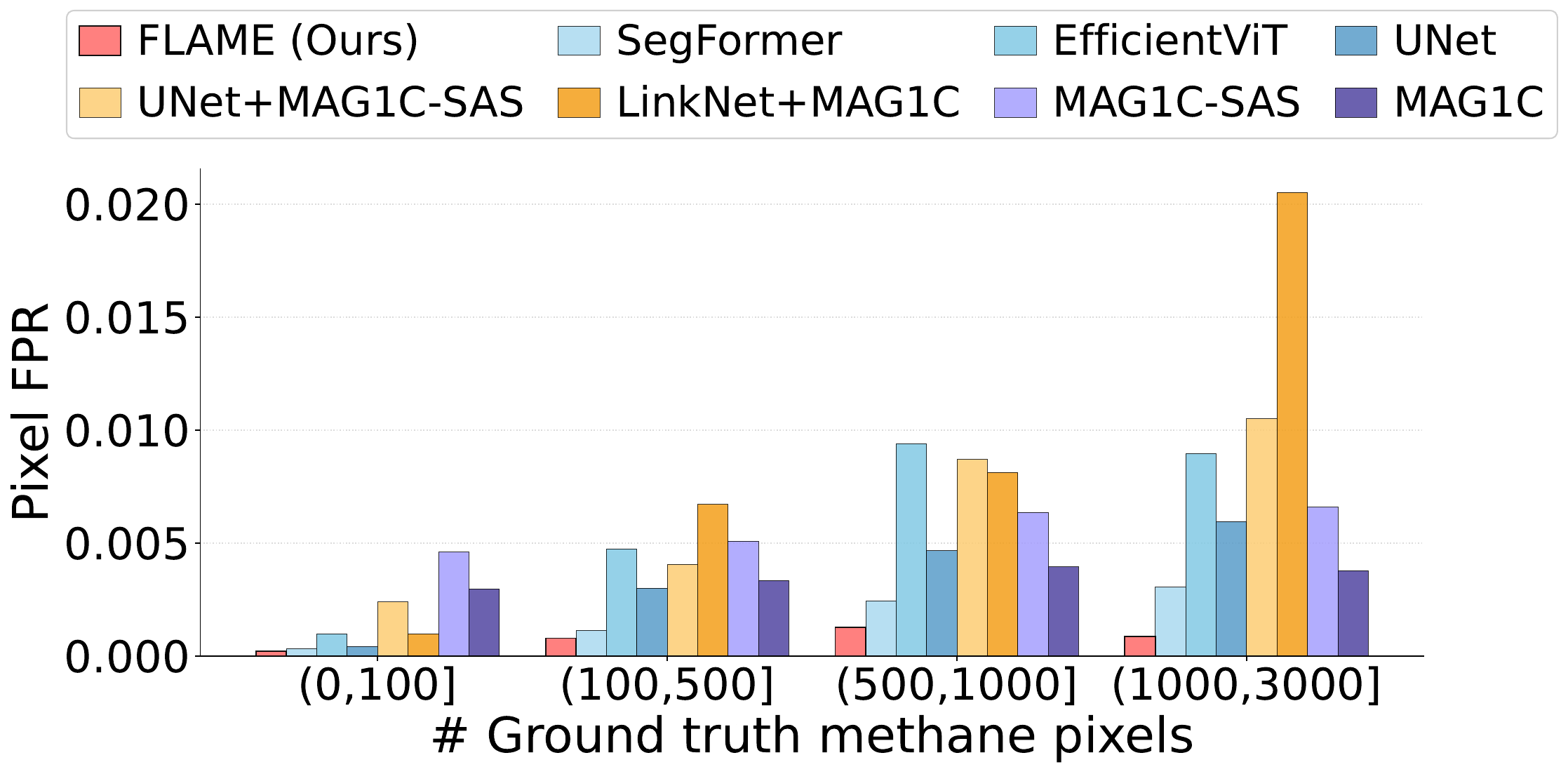}
    \caption{Pixel-level false positive rate across test tiles grouped by the number of ground-truth methane pixels.}
    \label{fig:plume_analysis_b}
\end{minipage}\hfill
\begin{minipage}[t]{0.48\columnwidth}
    \centering
    \includegraphics[width=\linewidth]{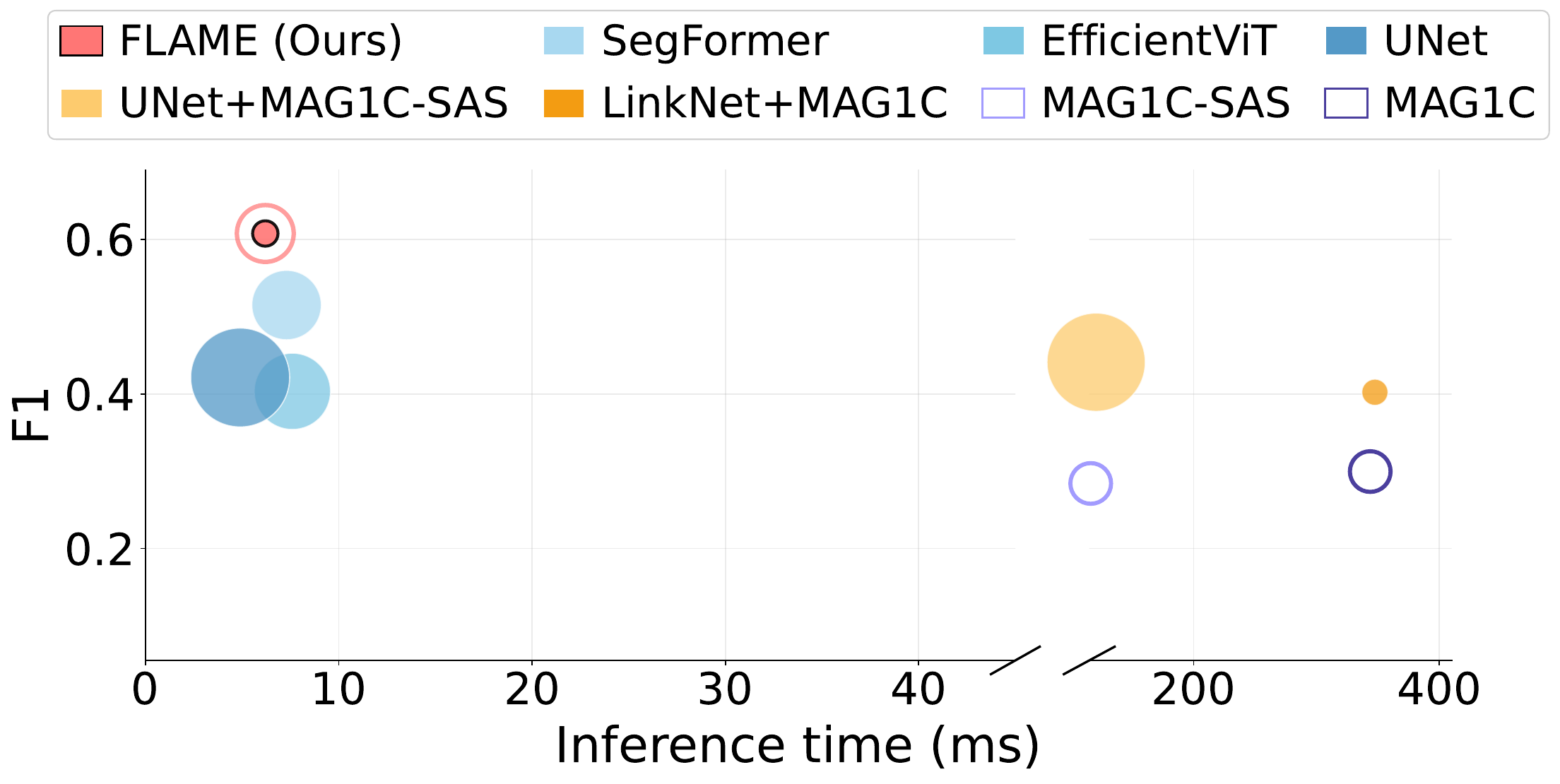}
    \caption{Comparison of speed, accuracy, and size across the evaluated methods. Marker size indicates the number of trainable parameters.}
    \label{fig:bubble}
\end{minipage}
\end{figure}

\paragraph{False-positive behavior across plume sizes.}
The plume-size analysis in the main text reports F1 scores across bins defined by the number of ground-truth methane pixels. Here we provide the complementary false-positive analysis. This analysis is important because methane plume detection is highly imbalanced at the pixel level, and an apparent gain in F1 can be caused by overly large predicted masks that improve recall while increasing false alarms over background regions. Figure~\ref{fig:plume_analysis_b} reports the pixel-level false positive rate across the same plume-size bins used in the main text. FLAME maintains a low false-positive rate across plume regimes, including larger plume cases where several segmentation baselines tend to over-extend the plume mask into surrounding background. This behavior indicates that the physics-guided score improves plume localization without simply expanding predictions around candidate plume regions. The result also supports the qualitative observations in Appendix~\ref{app:additional_qual}, where FLAME produces more compact predictions around plume boundaries and fewer activations on scenes without plumes.

\paragraph{Speed, accuracy, and size trade-off.}
Figure~\ref{fig:bubble} compares the evaluated methods in terms of inference latency, detection accuracy, and parameter count. This analysis complements the main hardware profiling results by showing where each method lies in the broader accuracy-efficiency landscape. Classical matched-filter methods have no trainable parameters but require expensive statistical estimation, while two-stage pipelines inherit the cost of their methane enhancement product. End-to-end learned models are fast, but they do not explicitly preserve the methane absorption prior. FLAME occupies a favorable trade-off region by combining single-pass inference with a compact physics-guided architecture.

\clearpage

\section{Onboard Satellite Deployment Details}
\label{sec:edge_hardware}
This appendix expands on the onboard satellite deployment results in Section~\ref{sec:edge}.
\paragraph{Hardware.}
\label{sec:edge_hardware_config}
Table~\ref{tab:jetson_hw} reports the configuration of the three NVIDIA Jetson modules used in Section~\ref{sec:edge}. The three modules span the compute and power tiers that current and next-generation onboard satellite payload computers are expected to provide. Jetson Orin NX 16GB represents the low-power tier suitable for cubesat-class platforms, Jetson AGX Orin 64GB represents the production-grade mid-range tier suitable for current onboard inference, and Jetson AGX Thor represents the next-generation high-performance tier. The TDP cap column reports the configured power budget under which the latency, power, and temperature measurements in Table~\ref{tab:jetson} were taken.

\begin{table}[h]
\centering
\caption{Hardware configuration of the NVIDIA Jetson platforms used in Section~\ref{sec:edge}.}
\label{tab:jetson_hw}
\footnotesize
\setlength{\tabcolsep}{4pt}
\begin{tabular}{@{}l l l r r@{}}
\toprule
Platform & CPU & GPU & Memory & TDP cap \\
\midrule
Jetson Orin NX   & $8\times$ A78AE @ 2.0 GHz           & Ampere    & 16 GB  & 15 W  \\
Jetson AGX Orin  & $12\times$ A78AE @ 2.2 GHz          & Ampere    & 64 GB  & 30 W \\
Jetson AGX Thor       & $14\times$ Neoverse-V3AE @ 2.6 GHz  & Blackwell & 128 GB & 120 W\\
\bottomrule
\end{tabular}
\end{table}

\paragraph{Measurement protocol.}
\label{sec:edge_hardware_protocol}
All measurements use the same trained FLAME weights and a fixed batch of five STARCOP test tiles. For latency, we issue 50 untimed warm-up forwards and then 200 timed repetitions over the five-tile set per platform, and report the mean per-tile wall-clock time. For power and temperature, on-device sensors are sampled every 100~ms throughout the timed window, and we report mean and peak values per power rail and thermal zone. The values in Table~\ref{tab:jetson} correspond to the total module power and to the GPU thermal zone. Each platform's TDP cap is set to the value listed in Table~\ref{tab:jetson_hw} before profiling.


\section{Limitations and Future Work}
\label{app:limitations}

\paragraph{Single-dataset evaluation.}
All results are reported on the STARCOP benchmark, following the evaluation protocol of MAG1C-SAS~\cite{herec2025optimizing}. The physics-guided score in Eq.~\eqref{eq:flame_score} depends only on the methane absorption spectrum and the log-radiance field, both of which are sensor-specific but conceptually transferable, so we expect the architecture to adapt to other SWIR spectrometers given matched training data. 

\paragraph{Absence of in-orbit validation.}
The onboard satellite deployment characterization in Section~\ref{sec:edge} is performed on NVIDIA Jetson platforms under laboratory conditions. Ground-based profiling does not capture the environmental stressors encountered in low Earth orbit, including total ionizing dose, single event effects, and thermal cycling induced by repeated eclipse transitions. The vacuum environment also eliminates atmospheric convection, so in-orbit heat dissipation must rely entirely on conduction to the spacecraft thermal control system and radiative emission to deep space, which can shift the steady-state operating temperatures reported in Table~\ref{tab:jetson}. Radiation-induced bit flips can degrade neural network inference in ways that depend on weight encoding and arithmetic precision, and the FP32 PyTorch configuration we profile has not been evaluated under such conditions. Validation on radiation-tolerant payload computers, and ultimately an in-orbit demonstration, is left for future work.

\clearpage

\section{Analyses of the Physics-Guided Score}
\label{app:physics_guided}

\begin{figure*}[ht]
\centering
\includegraphics[width=0.70\textwidth]{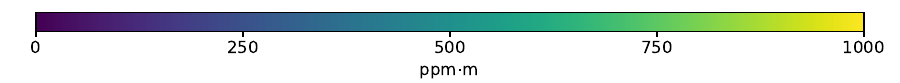}
\vspace{2pt}
\setlength{\tabcolsep}{1pt}
\renewcommand{\arraystretch}{0.9}
\begin{tabular}{ccccc}
\scriptsize RGB & \scriptsize MAG1C-SAS & \scriptsize \flame{} (Phase 1) & \scriptsize \flame{} (Phase 2) & \scriptsize GT \\[2pt]
\includegraphics[width=0.1800\textwidth]{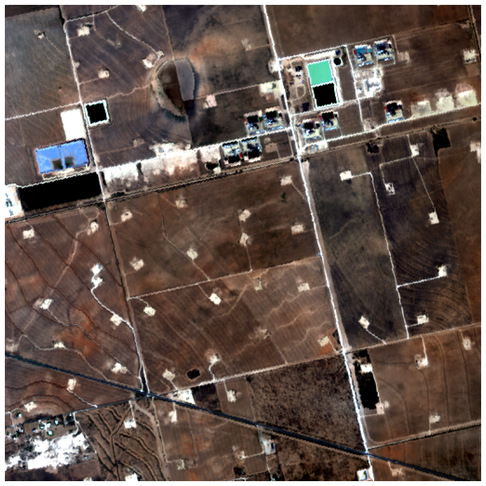} &
\includegraphics[width=0.1800\textwidth]{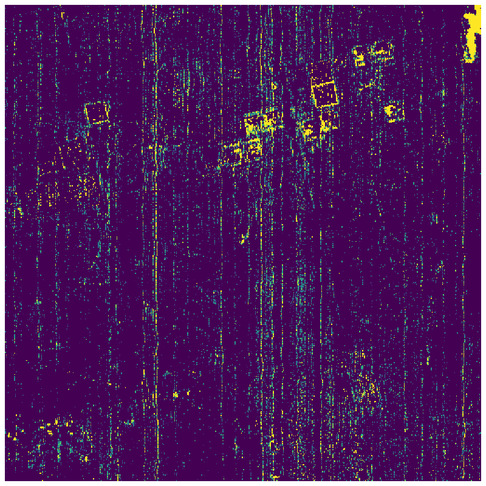} &
\includegraphics[width=0.1800\textwidth]{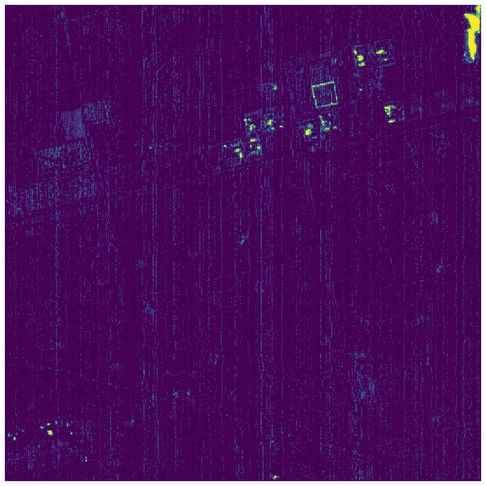} &
\includegraphics[width=0.1800\textwidth]{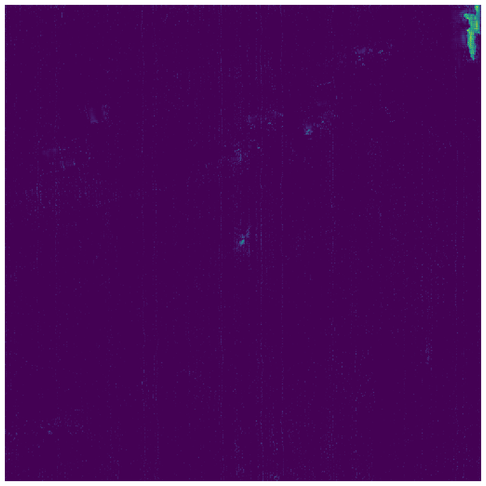} &
\includegraphics[width=0.1800\textwidth]{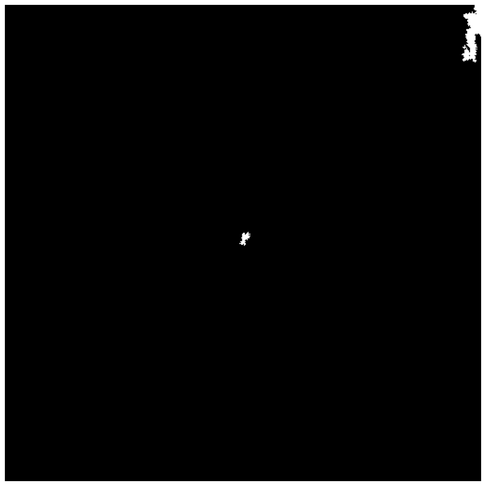} \\[2pt]
\includegraphics[width=0.1800\textwidth]{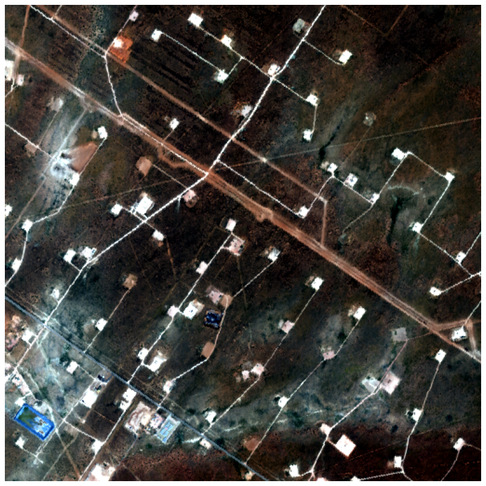} &
\includegraphics[width=0.1800\textwidth]{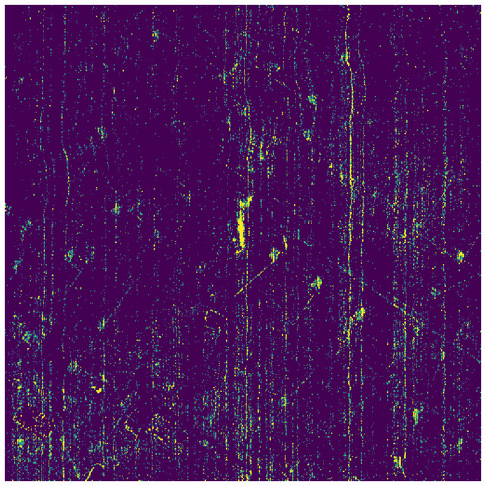} &
\includegraphics[width=0.1800\textwidth]{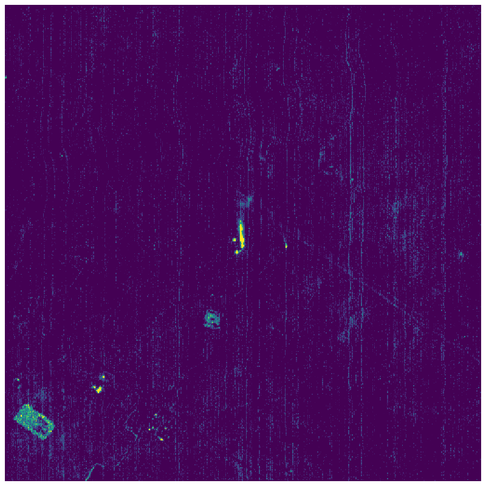} &
\includegraphics[width=0.1800\textwidth]{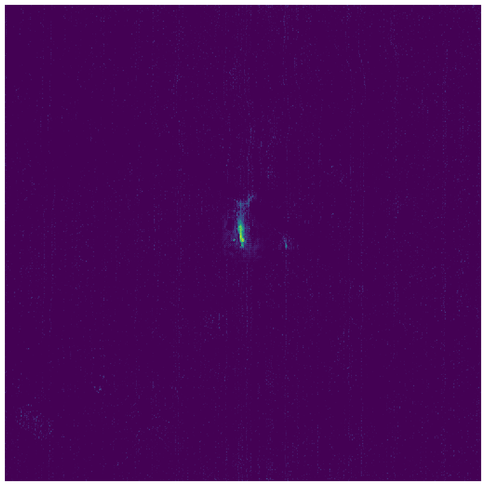} &
\includegraphics[width=0.1800\textwidth]{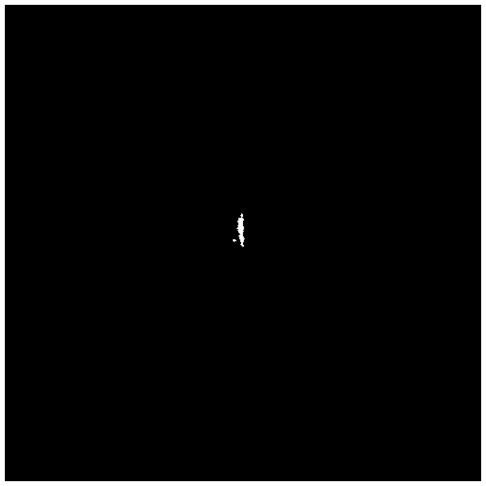} \\[2pt]
\includegraphics[width=0.1800\textwidth]{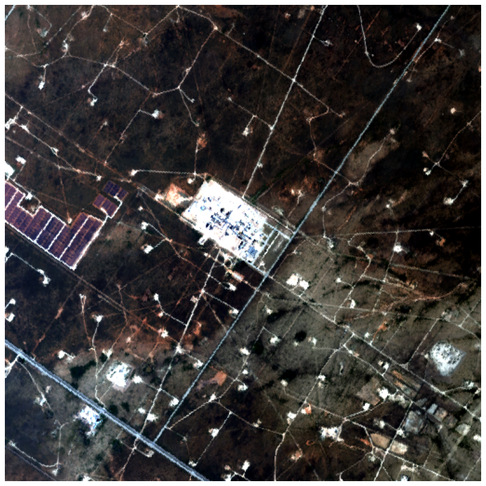} &
\includegraphics[width=0.1800\textwidth]{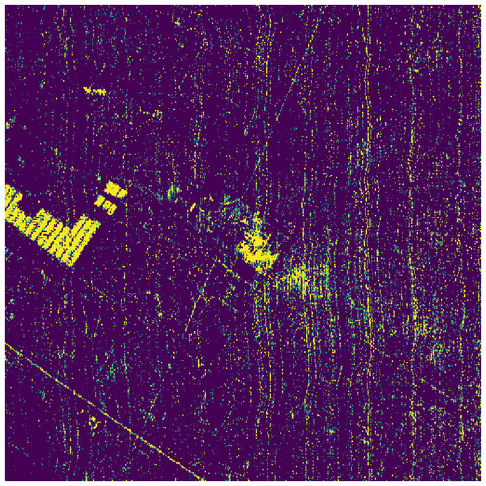} &
\includegraphics[width=0.1800\textwidth]{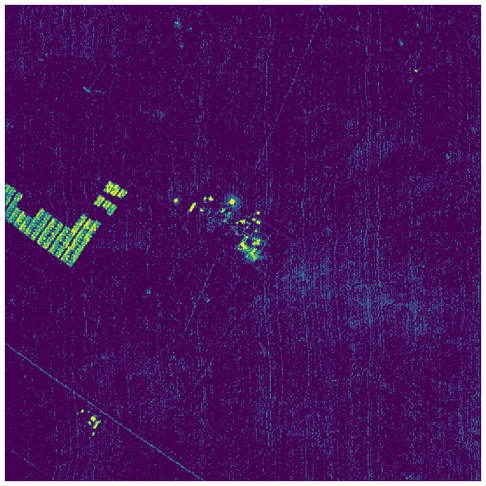} &
\includegraphics[width=0.1800\textwidth]{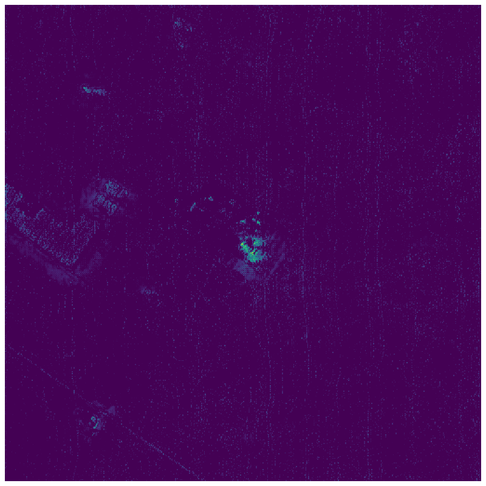} &
\includegraphics[width=0.1800\textwidth]{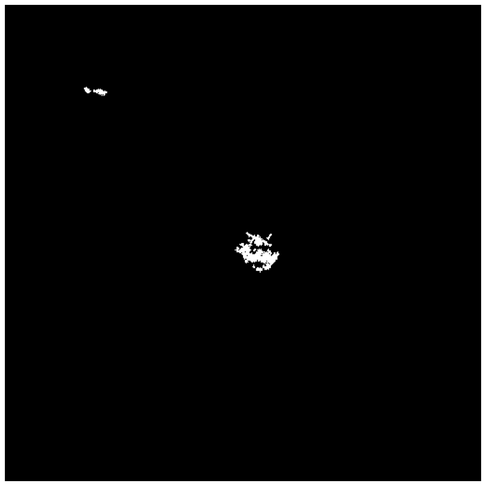} \\
\end{tabular}
\caption{Evolution of the FLAME physics score across the two-phase training curriculum described in the main text. Columns show the RGB tile, the MAG1C-SAS reference, the FLAME physics score after Phase 1, the FLAME physics score after Phase 2, and the ground-truth plume mask.}
\label{fig:curriculum_a}
\end{figure*}

\paragraph{Curriculum evolution.}
As described in the main text, the curriculum first aligns the physics score with the MAG1C-SAS teacher and then shifts the optimization toward the segmentation objective. This design gives the score layer a stable physics-aligned initialization while allowing the final model to depart from the teacher when ground-truth supervision provides a better plume boundary. Figure~\ref{fig:curriculum_a} visualizes this transition, with the Phase~$1$ and Phase~$2$ columns corresponding to the model state at epoch~$T=10$, when $\gamma(t)$ first reaches zero, and at the final epoch~$50$, respectively. The Phase~1 physics score follows the MAG1C-SAS response, while the Phase~2 physics score becomes more selective around annotated plume regions. This selectivity matters because MAG1C-SAS responds to a wide range of albedo variations and surface confounders, and the Phase~2 physics score suppresses these responses while preserving the plume signal, which directly contributes to the false-positive reduction reported in the main text.

\clearpage

\begin{figure*}[t]
\centering
\includegraphics[width=0.70\textwidth]{outputs/figures/fig_shared_colormap_ppm.pdf}
\vspace{3pt}

\newlength{\tilew}
\setlength{\tilew}{0.180\textwidth}
\newlength{\rowlabelw}
\setlength{\rowlabelw}{0.024\textwidth}

\newcommand{\rowlabel}[1]{%
  \makebox[\rowlabelw][c]{%
    \rotatebox[origin=c]{90}{\scriptsize #1}%
  }%
}

\setlength{\tabcolsep}{1.5pt}
\renewcommand{\arraystretch}{0.95}

\begin{tabular}{@{}c@{\hspace{1pt}}*{5}{c}@{}}
&
\scriptsize Tile 1 &
\scriptsize Tile 2 &
\scriptsize Tile 3 &
\scriptsize Tile 4 &
\scriptsize Tile 5
\\[2pt]

\adjustbox{valign=c}{\rowlabel{RGB image}} &
\adjustbox{valign=c}{\includegraphics[width=\tilew]{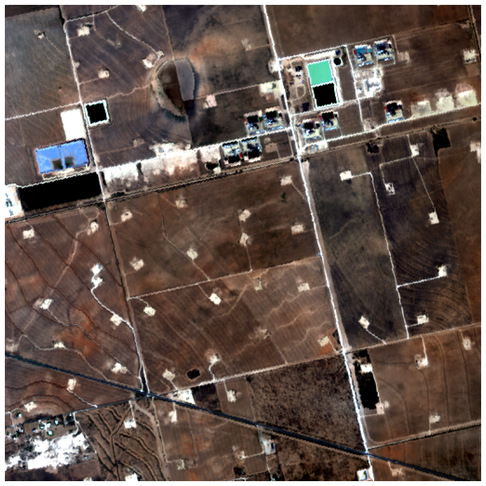}} &
\adjustbox{valign=c}{\includegraphics[width=\tilew]{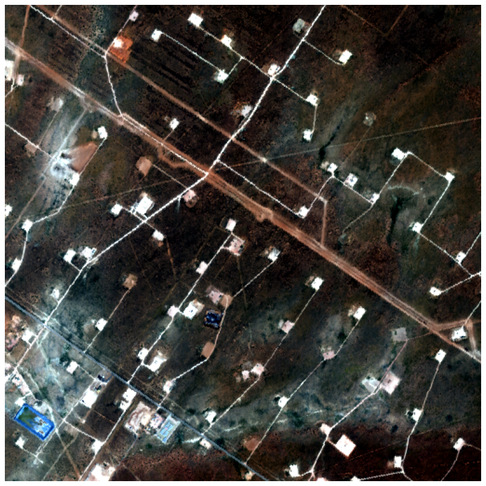}} &
\adjustbox{valign=c}{\includegraphics[width=\tilew]{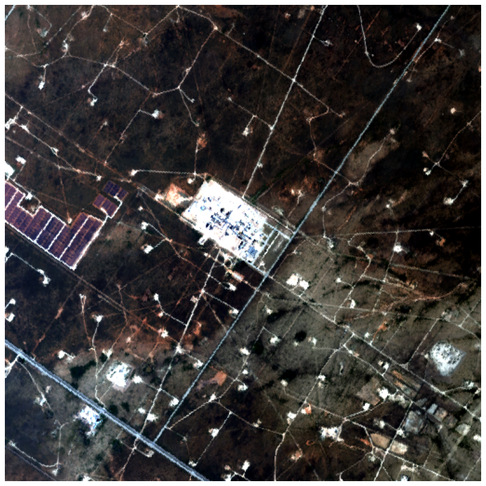}} &
\adjustbox{valign=c}{\includegraphics[width=\tilew]{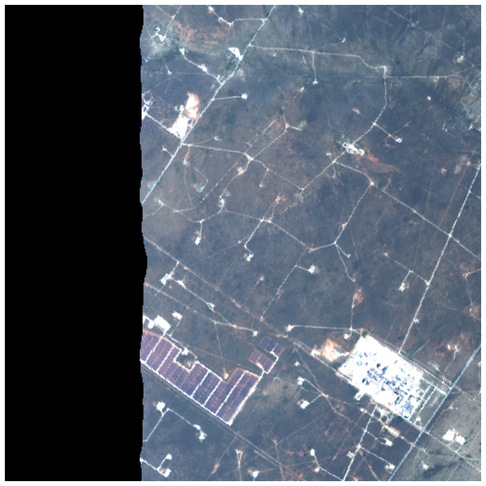}} &
\adjustbox{valign=c}{\includegraphics[width=\tilew]{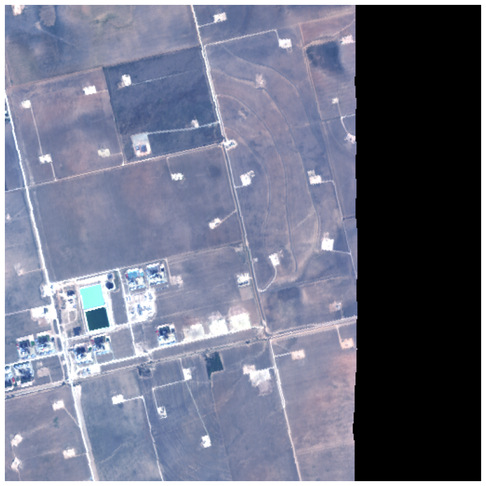}}
\\[3pt]

\adjustbox{valign=c}{\rowlabel{Physics score}} &
\adjustbox{valign=c}{\includegraphics[width=\tilew]{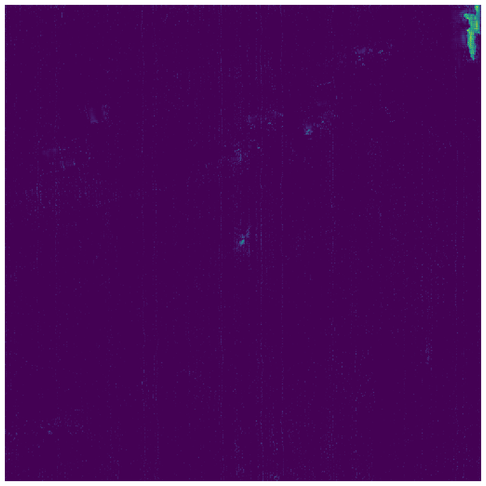}} &
\adjustbox{valign=c}{\includegraphics[width=\tilew]{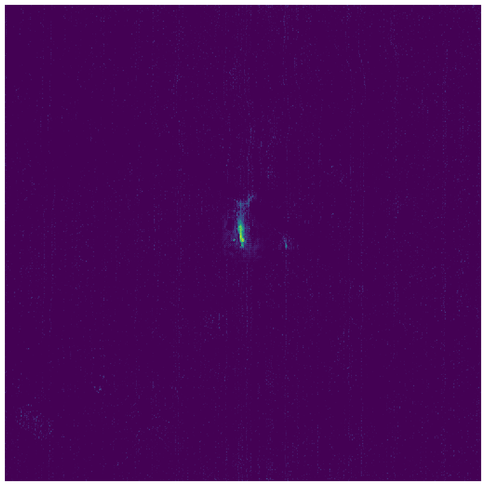}} &
\adjustbox{valign=c}{\includegraphics[width=\tilew]{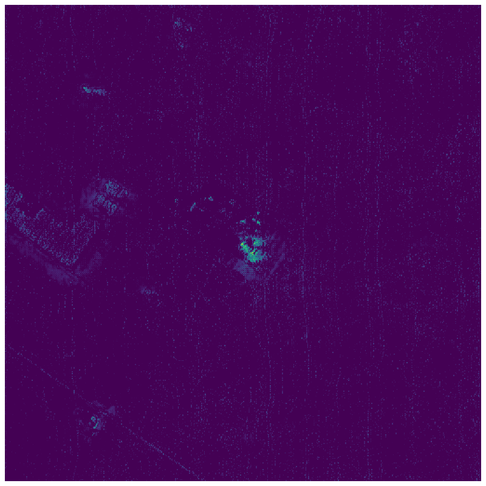}} &
\adjustbox{valign=c}{\includegraphics[width=\tilew]{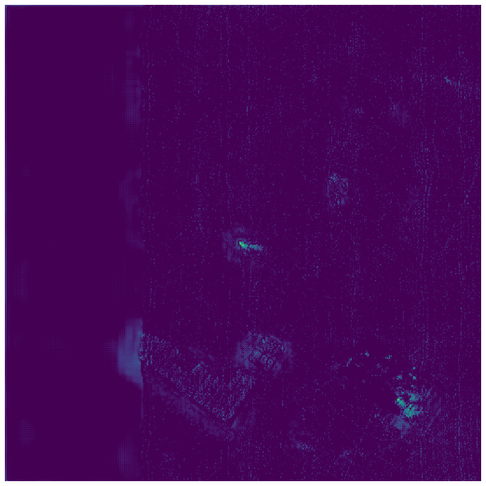}} &
\adjustbox{valign=c}{\includegraphics[width=\tilew]{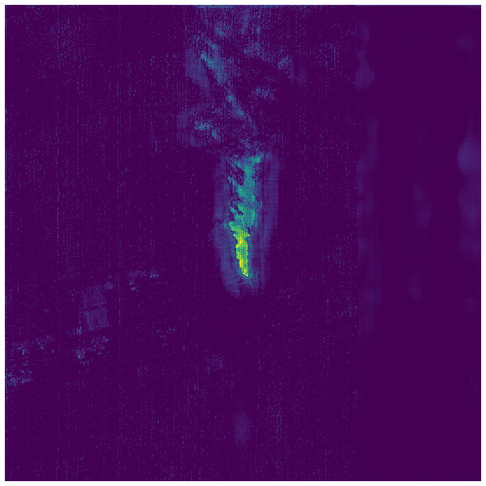}}
\\[3pt]

\adjustbox{valign=c}{\rowlabel{Probability map}} &
\adjustbox{valign=c}{\includegraphics[width=\tilew]{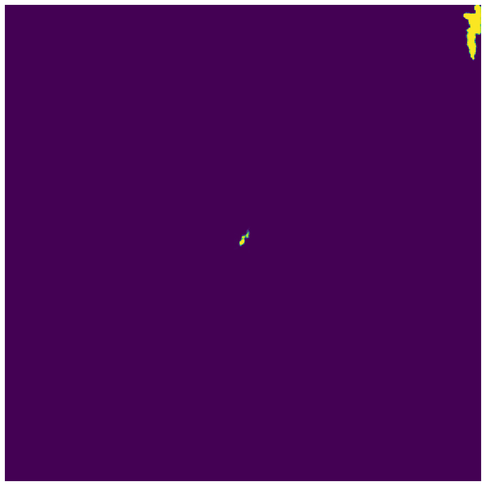}} &
\adjustbox{valign=c}{\includegraphics[width=\tilew]{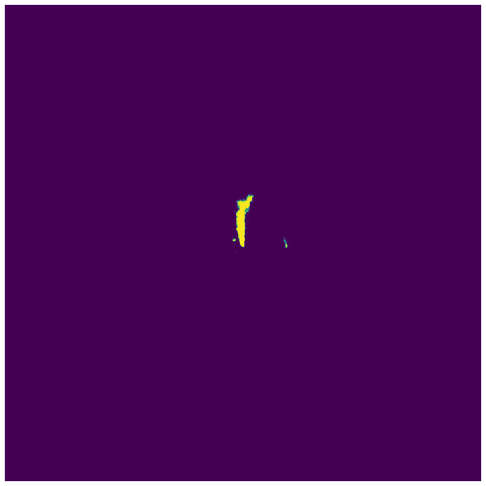}} &
\adjustbox{valign=c}{\includegraphics[width=\tilew]{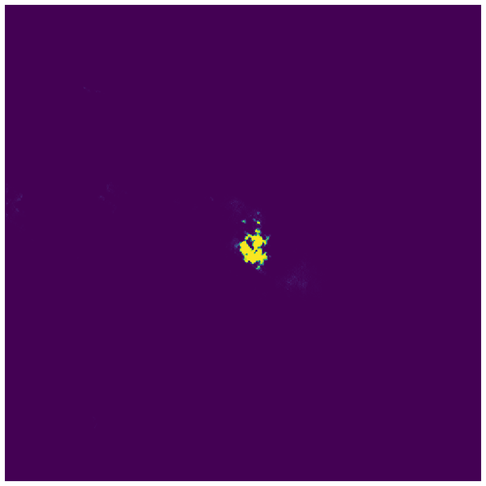}} &
\adjustbox{valign=c}{\includegraphics[width=\tilew]{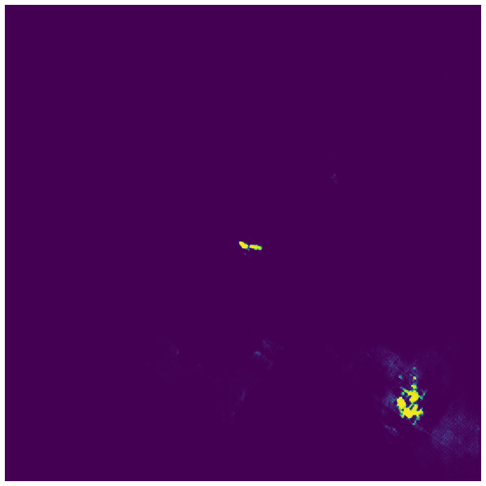}} &
\adjustbox{valign=c}{\includegraphics[width=\tilew]{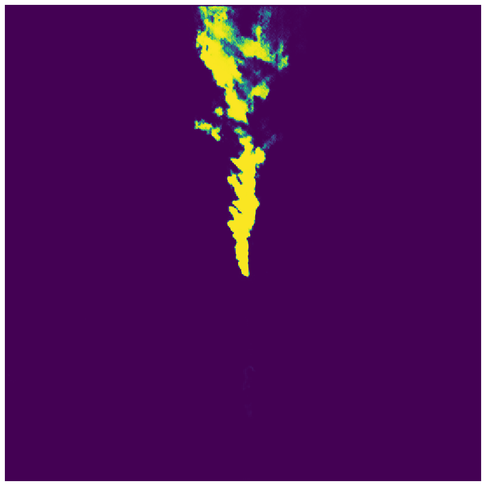}}
\\
\end{tabular}

\caption{Comparison between the FLAME physics score and the probability map produced by the segmentation head. Each column corresponds to a single test tile.}
\label{fig:curriculum_b}
\end{figure*}

\paragraph{Score-to-mask refinement.}
The physics score is an intermediate methane evidence map, while the segmentation head converts it into the probability map that feeds the final binary mask through thresholding. Figure~\ref{fig:curriculum_b} compares the physics score and the probability map on the same tiles. The segmentation head sharpens the response and calibrates the probability while preserving the spatial structure of the physics score. This indicates that the location and shape of the final binary mask are already determined at the physics score stage, and that the segmentation head acts as a refinement step rather than an independent decision module, which is consistent with the score-layer ablation in the main text.

\clearpage

\begin{figure*}[t]
\centering
\includegraphics[width=0.70\textwidth]{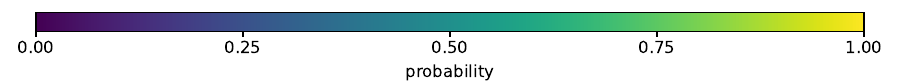}
\vspace{2pt}
\setlength{\tabcolsep}{1pt}
\renewcommand{\arraystretch}{0.9}
\begin{tabular}{cccccc}
\scriptsize RGB & \scriptsize No physics score & \scriptsize Fixed bg, fixed weights & \scriptsize Fixed bg, learned weights & \scriptsize \flame{} (Ours) & \scriptsize GT \\[2pt]
\includegraphics[width=0.1550\textwidth]{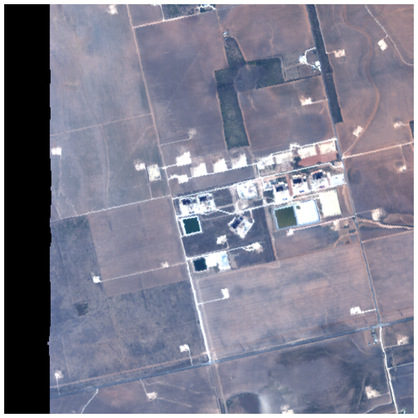} &
\includegraphics[width=0.1550\textwidth]{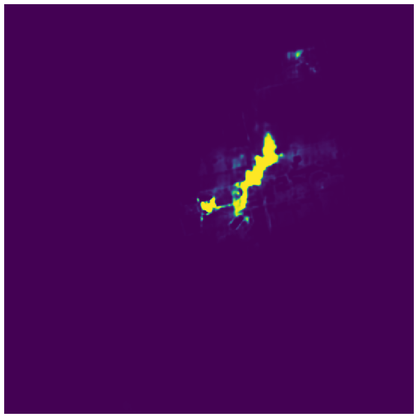} &
\includegraphics[width=0.1550\textwidth]{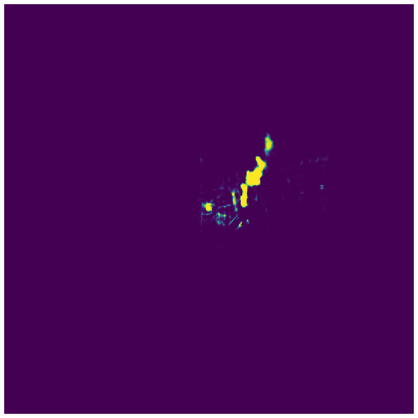} &
\includegraphics[width=0.1550\textwidth]{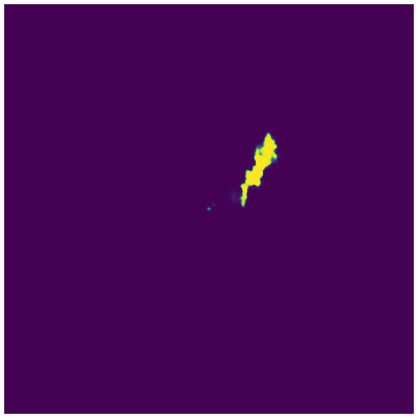} &
\includegraphics[width=0.1550\textwidth]{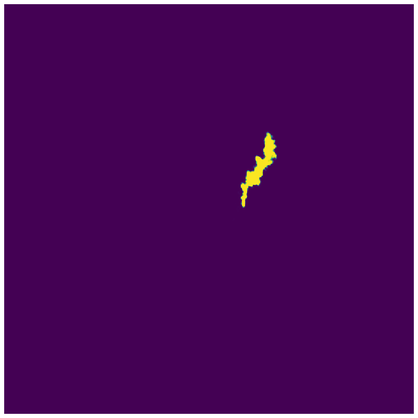} &
\includegraphics[width=0.1550\textwidth]{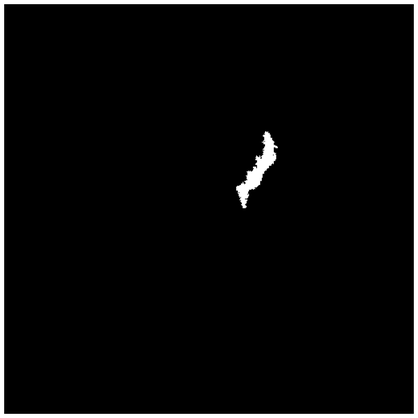} \\[2pt]
\includegraphics[width=0.1550\textwidth]{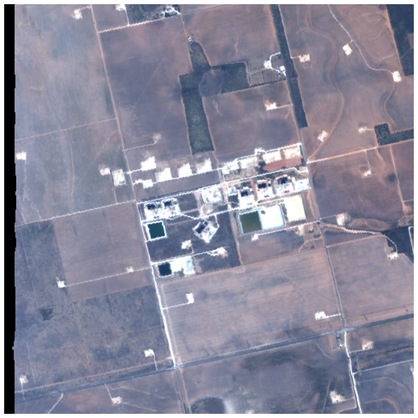} &
\includegraphics[width=0.1550\textwidth]{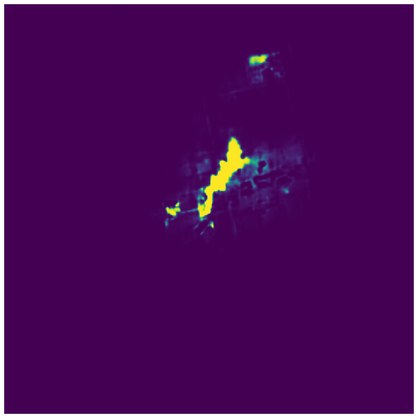} &
\includegraphics[width=0.1550\textwidth]{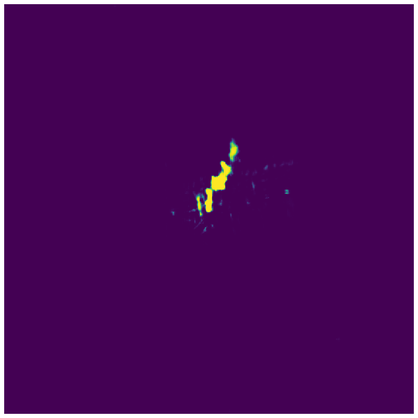} &
\includegraphics[width=0.1550\textwidth]{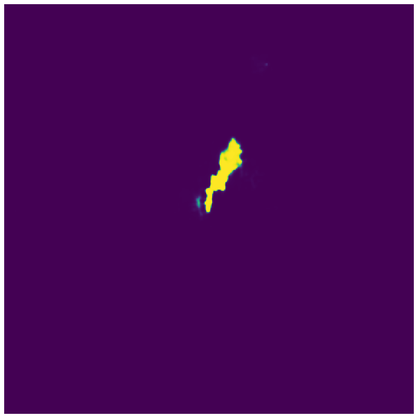} &
\includegraphics[width=0.1550\textwidth]{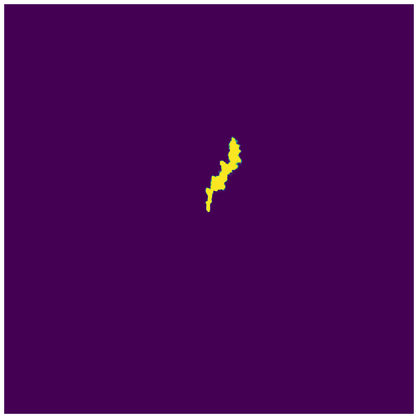} &
\includegraphics[width=0.1550\textwidth]{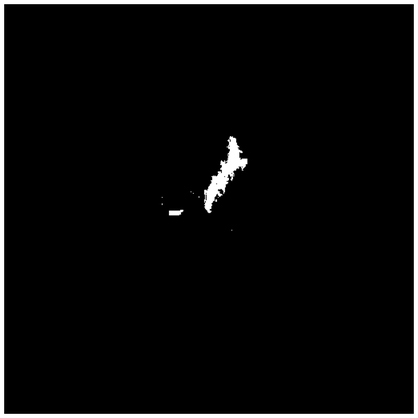} \\[2pt]
\includegraphics[width=0.1550\textwidth]{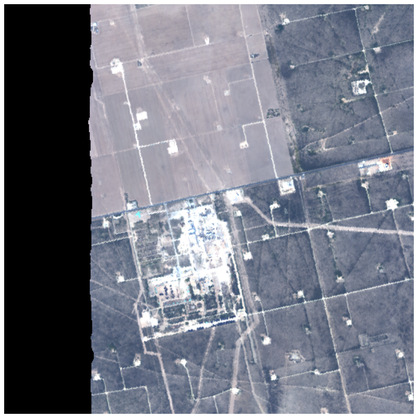} &
\includegraphics[width=0.1550\textwidth]{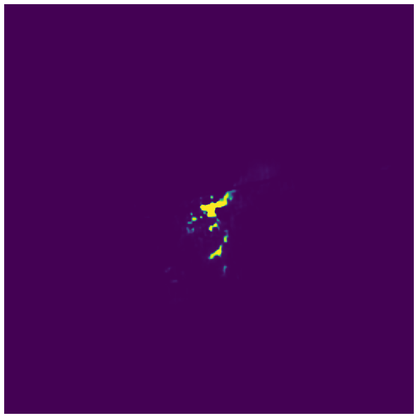} &
\includegraphics[width=0.1550\textwidth]{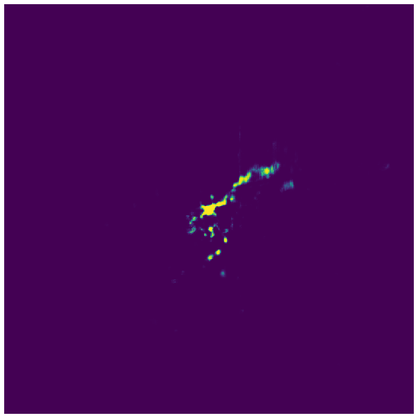} &
\includegraphics[width=0.1550\textwidth]{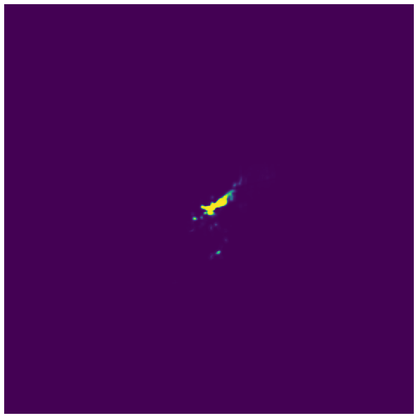} &
\includegraphics[width=0.1550\textwidth]{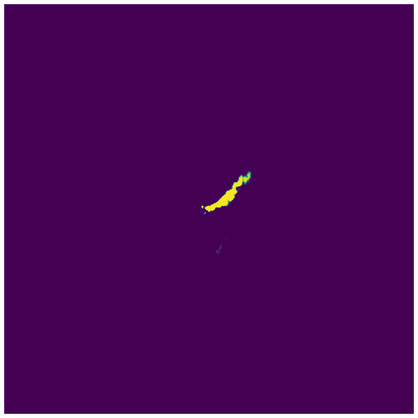} &
\includegraphics[width=0.1550\textwidth]{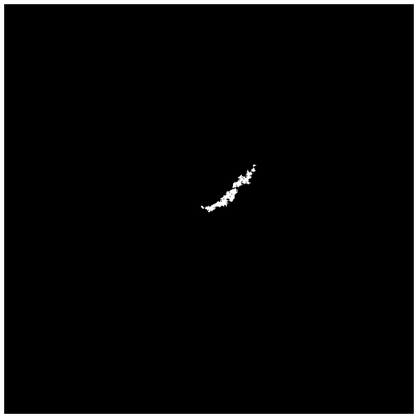} \\
\end{tabular}
\caption{Probability maps produced by physics-integration variants of FLAME. Columns show RGB, no physics score, fixed background with fixed spectral weights, fixed background with learned spectral weights, full FLAME, and the ground-truth plume mask, where bg denotes background.}
\label{fig:score_progression}
\end{figure*}

\begin{table}[t!]
\centering
\caption{Component-wise ablation of the physics score. The full FLAME model learns both the pixel-wise log-background and the pixel-wise spectral weighting, while the variants progressively remove these learned components.}
\label{tab:component_ablation}
\small
\begin{tabular}{@{}l c c c c@{}}
\toprule
Variant & Recall & Precision & F1 & IoU \\
\midrule
No physics score & 0.3841 & 0.4367 & 0.4087 & 0.2568 \\
Fixed background, fixed weights & 0.4340 & 0.5543 & 0.4868 & 0.3217 \\
Fixed background, learned weights & 0.4988 & \textbf{0.6702} & 0.5720 & 0.4005 \\
FLAME (Ours) & \textbf{0.6110} & 0.5957 & \textbf{0.6032} & \textbf{0.4319} \\
\bottomrule
\end{tabular}
\end{table}

\paragraph{Component-wise ablation.}
The score-layer ablation in the main text compares FLAME with and without the physics score. Here we further decompose the physics score itself to isolate the contribution of each learned physical field. We analyze four variants. The first variant uses no physics score and trains the neural backbone as a segmentation model, which corresponds to the w/o physics-guide condition in the main text. The second variant uses a fixed background and fixed spectral weights, recovering a closed-form matched-filter score. The third variant fixes the background but learns the spectral weights, isolating the contribution of pixel-adaptive spectral weighting. The fourth variant learns both the pixel-wise log-background and the pixel-wise spectral weights, which is the full FLAME model and corresponds to the w/ physics-guide condition in the main text. Table~\ref{tab:component_ablation} reports the corresponding detection metrics, and Figure~\ref{fig:score_progression} shows the probability maps. Each added component yields a clear F1 gain. Introducing a physics score lifts F1 from 0.4087 to 0.4868, learning the spectral weights raises it to 0.5720, and additionally learning the log-background reaches 0.6032. The largest jump comes from learning the spectral weights, but the background head is still needed to recover the recall lost in the third variant, where precision is high but plume coverage remains incomplete. This pattern supports the use of both learned physical fields rather than either component alone.

\clearpage

\section{Additional Qualitative Results}
\label{app:additional_qual}

This appendix complements the qualitative analysis in Section~\ref{sec:qualitative}.

\begin{figure*}[ht]
\centering
\setlength{\tabcolsep}{1pt}
\begin{tabular}{@{}ccccc@{}}
\scriptsize RGB & \scriptsize UNet+MAG1C-SAS & \scriptsize SegFormer (CU) & \scriptsize \flame{} (ours) & \scriptsize GT \\[2pt]
\includegraphics[width=0.17\textwidth]{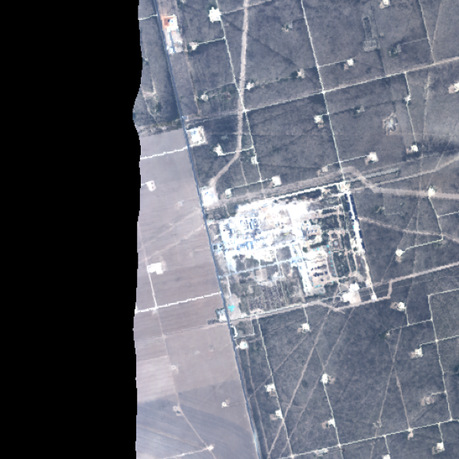} &
\includegraphics[width=0.17\textwidth]{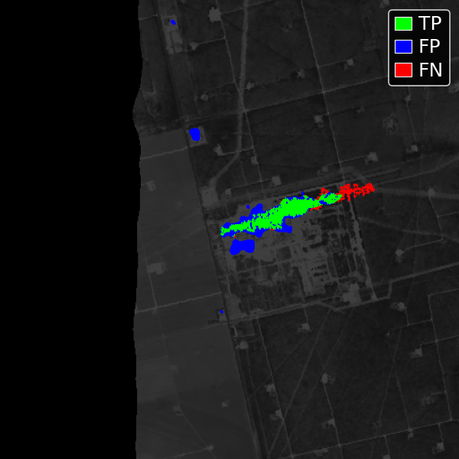} &
\includegraphics[width=0.17\textwidth]{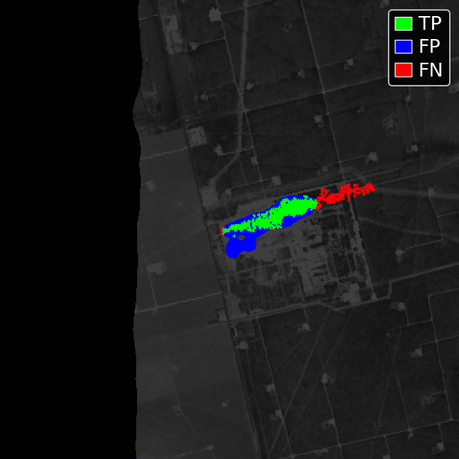} &
\includegraphics[width=0.17\textwidth]{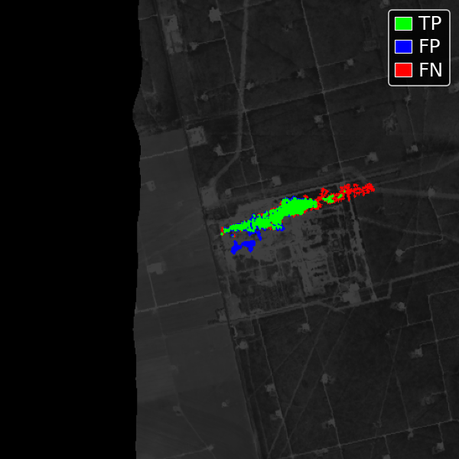} &
\includegraphics[width=0.17\textwidth]{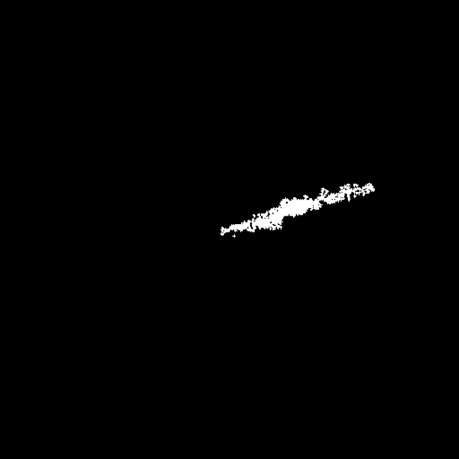} \\[2pt]
\includegraphics[width=0.17\textwidth]{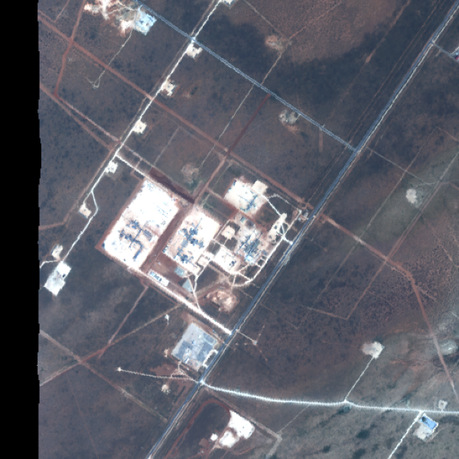} &
\includegraphics[width=0.17\textwidth]{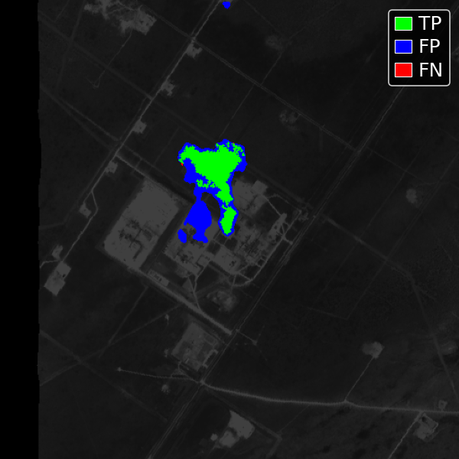} &
\includegraphics[width=0.17\textwidth]{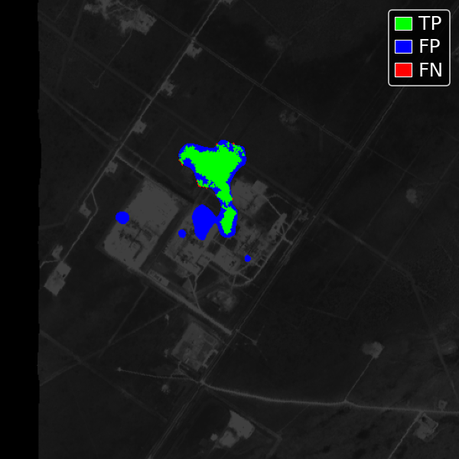} &
\includegraphics[width=0.17\textwidth]{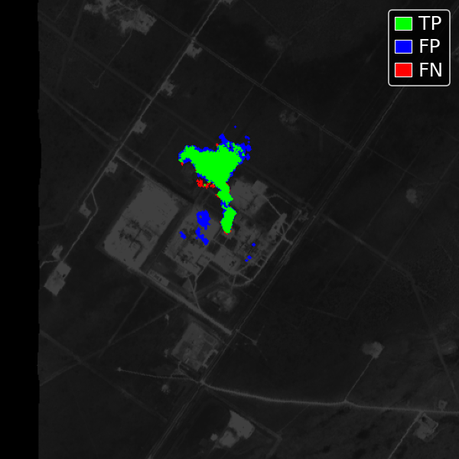} &
\includegraphics[width=0.17\textwidth]{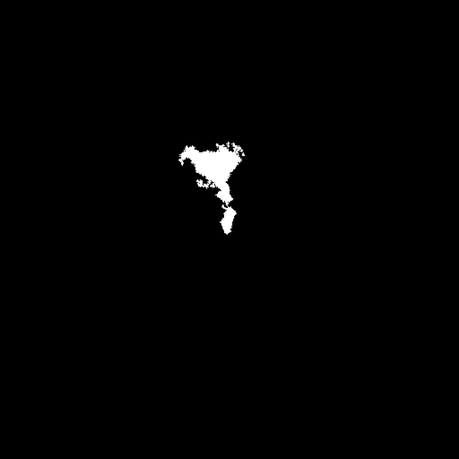} \\[2pt]
\includegraphics[width=0.17\textwidth]{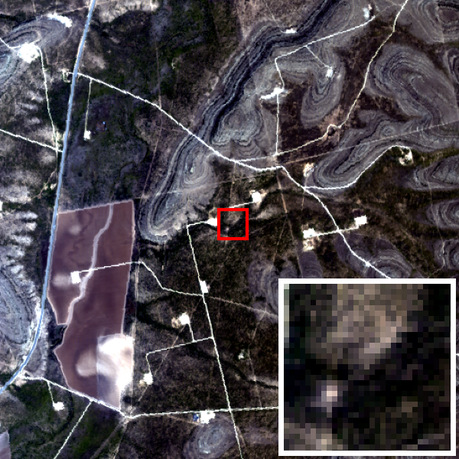} &
\includegraphics[width=0.17\textwidth]{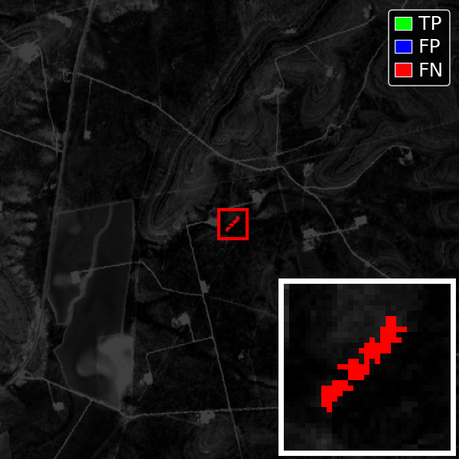} &
\includegraphics[width=0.17\textwidth]{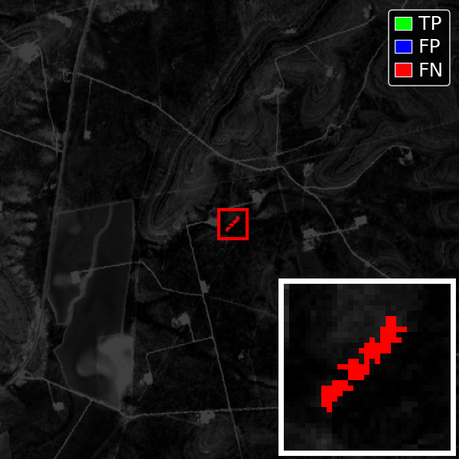} &
\includegraphics[width=0.17\textwidth]{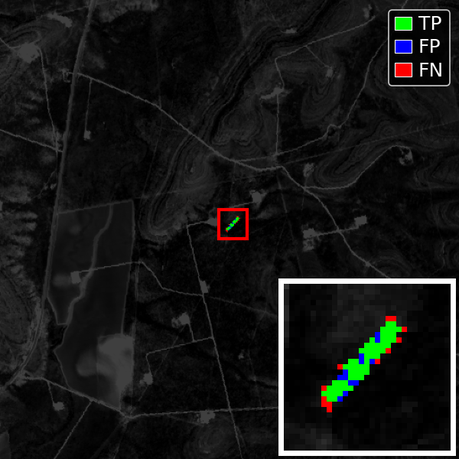} &
\includegraphics[width=0.17\textwidth]{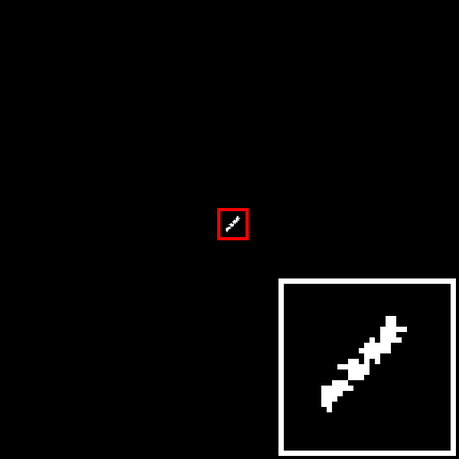} \\[2pt]
\includegraphics[width=0.17\textwidth]{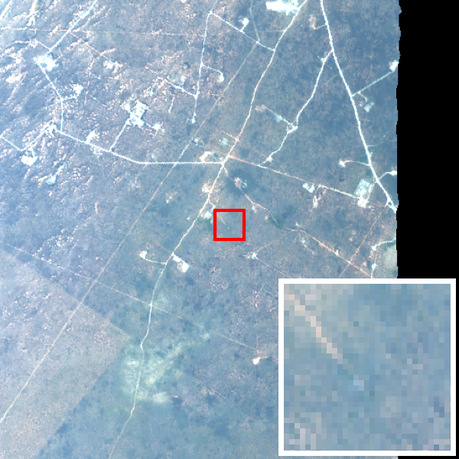} &
\includegraphics[width=0.17\textwidth]{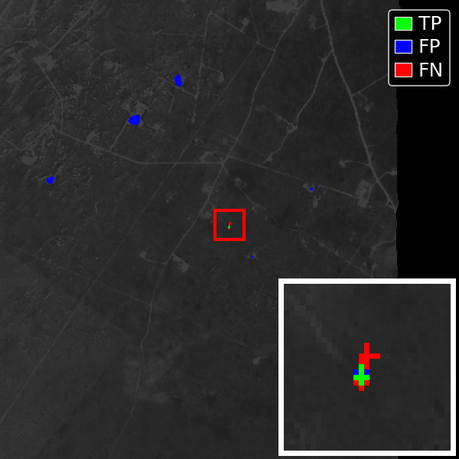} &
\includegraphics[width=0.17\textwidth]{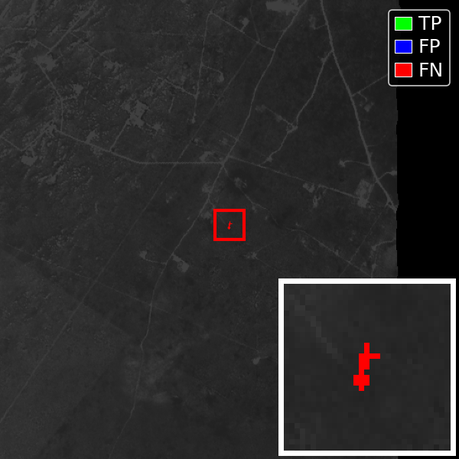} &
\includegraphics[width=0.17\textwidth]{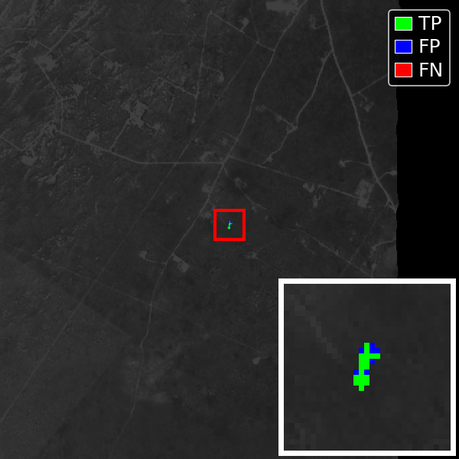} &
\includegraphics[width=0.17\textwidth]{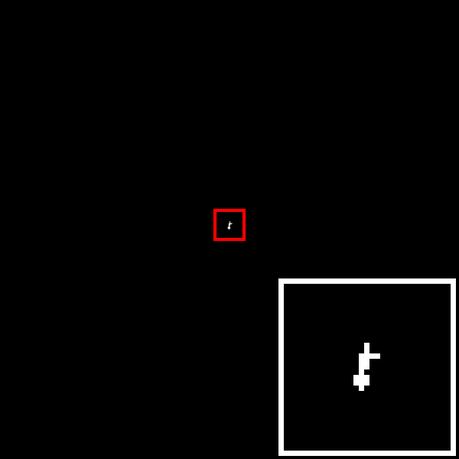} \\[2pt]
\includegraphics[width=0.17\textwidth]{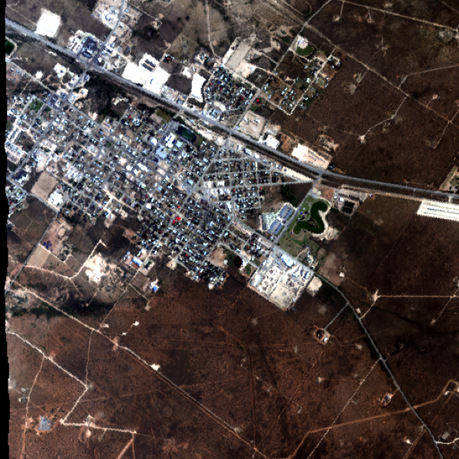} &
\includegraphics[width=0.17\textwidth]{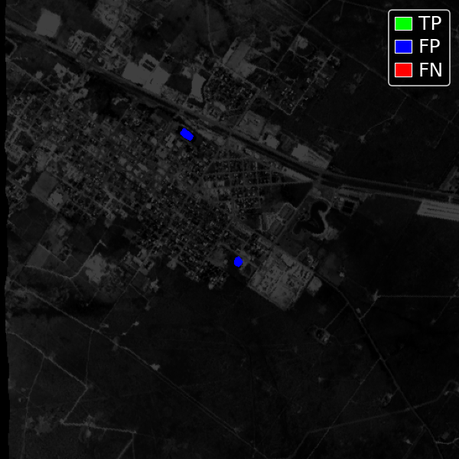} &
\includegraphics[width=0.17\textwidth]{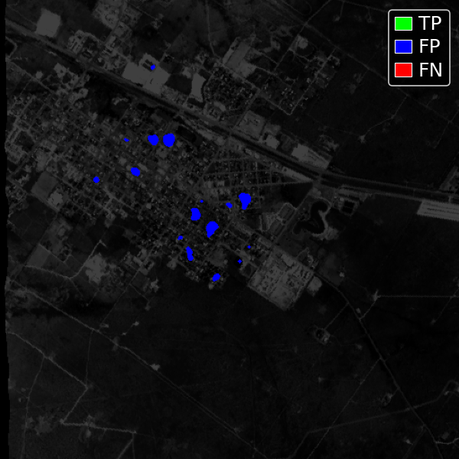} &
\includegraphics[width=0.17\textwidth]{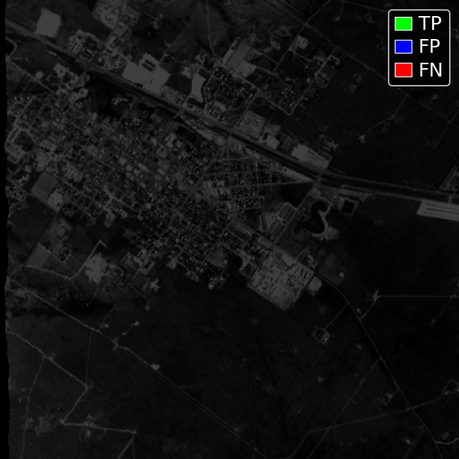} &
\includegraphics[width=0.17\textwidth]{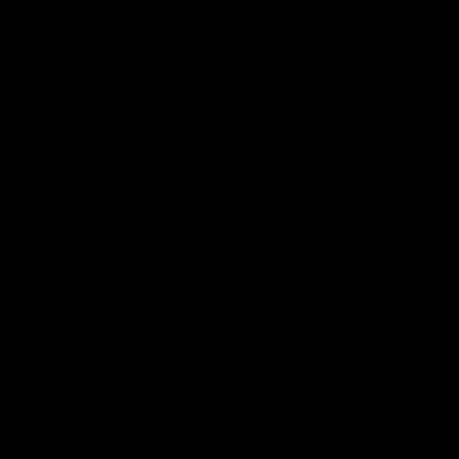} \\
\end{tabular}
\caption{Additional qualitative comparison on STARCOP test tiles. Columns show RGB, UNet+MAG1C-SAS, SegFormer with the ConvUp decoder, \flame{}, and the ground-truth mask. Prediction panels overlay true positives in green, false positives in blue, and false negatives in red on a darkened RGB backdrop. From top to bottom, rows show two strong-plume tiles, a weak-plume tile with a zoomed inset of the plume region, an over-extension case, and a plume-free urban tile.}
\label{fig:qualitative}
\end{figure*}

\paragraph{Qualitative Comparison} Across the regimes shown in Figure~\ref{fig:qualitative}, the dominant failure mode of UNet+MAG1C-SAS is over-prediction around the plume, visible as broad blue halos that extend the predicted mask beyond the annotated region; SegFormer with the ConvUp decoder either misses parts of the plume body, leaving red false negatives along the plume, or scatters small detections in plume-free areas. \flame{} produces masks that remain close to the annotated plume in the strong and weak regimes, recovers the elongated plume morphology in the zoomed weak-plume tile, and stays nearly empty on the plume-free urban scene where surface clutter triggers detections in the other models.

\clearpage

\begin{figure*}[t!]
\centering
\includegraphics[width=0.70\textwidth]{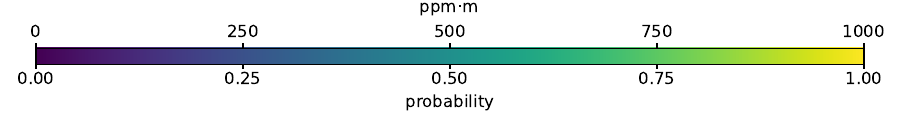}
\vspace{2pt}

\setlength{\rowlabelw}{0.024\textwidth}

\newcommand{\rowlabel}[1]{%
  \makebox[\rowlabelw][c]{%
    \rotatebox[origin=c]{90}{\scriptsize #1}%
  }%
}

\setlength{\tabcolsep}{1pt}
\renewcommand{\arraystretch}{0.95}

\begin{tabular}{@{}c@{\hspace{1pt}}*{5}{c}@{}}
&
\scriptsize RGB &
\scriptsize Probability map &
\scriptsize Final binary mask &
\scriptsize GT &
\scriptsize TP/FP/FN
\\[2pt]

\adjustbox{valign=c}{\rowlabel{MAG1C-tile}} &
\adjustbox{valign=c}{\includegraphics[width=0.1700\textwidth]{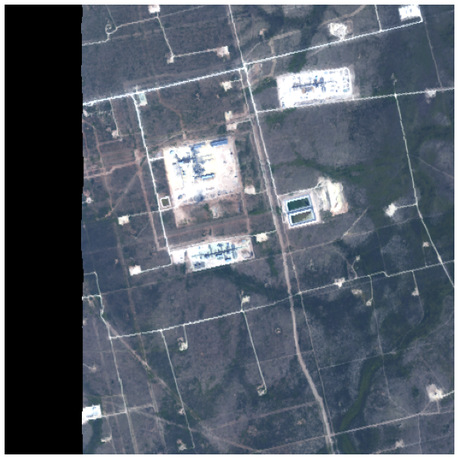}} &
\adjustbox{valign=c}{\includegraphics[width=0.1700\textwidth]{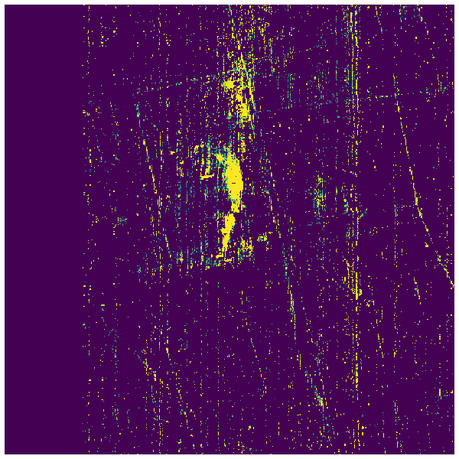}} &
\adjustbox{valign=c}{\includegraphics[width=0.1700\textwidth]{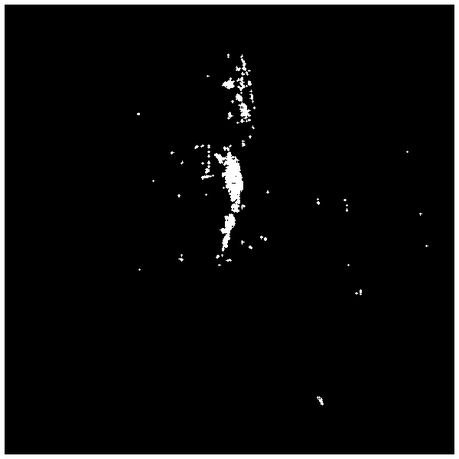}} &
\adjustbox{valign=c}{\includegraphics[width=0.1700\textwidth]{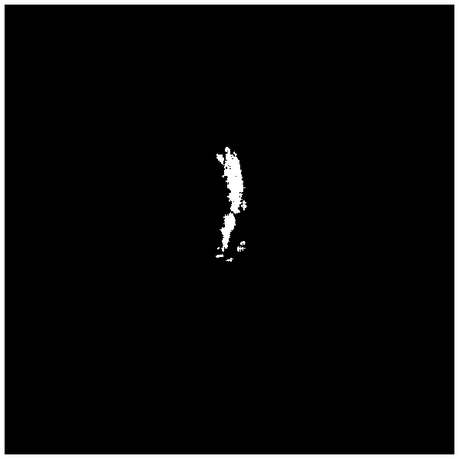}} &
\adjustbox{valign=c}{\includegraphics[width=0.1700\textwidth]{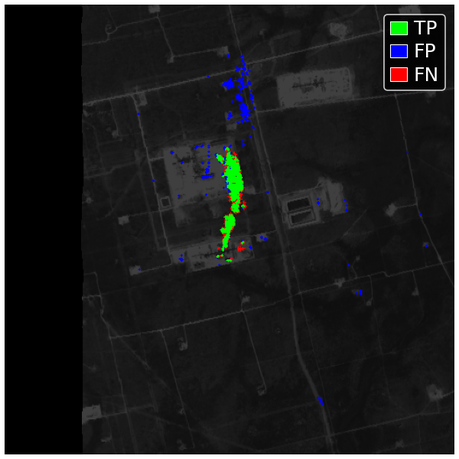}}
\\[2pt]

\adjustbox{valign=c}{\rowlabel{UNet+MAG1C-SAS}} &
\adjustbox{valign=c}{\includegraphics[width=0.1700\textwidth]{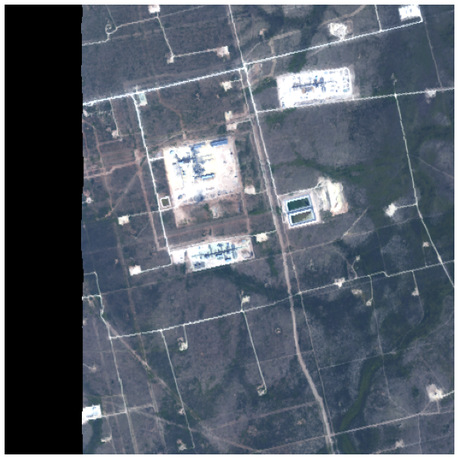}} &
\adjustbox{valign=c}{\includegraphics[width=0.1700\textwidth]{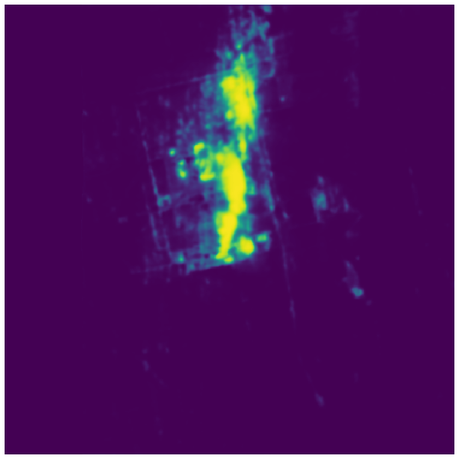}} &
\adjustbox{valign=c}{\includegraphics[width=0.1700\textwidth]{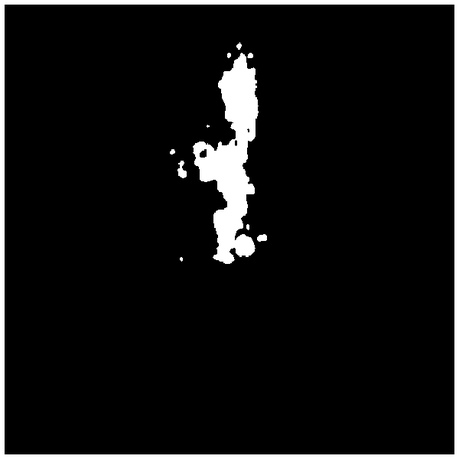}} &
\adjustbox{valign=c}{\includegraphics[width=0.1700\textwidth]{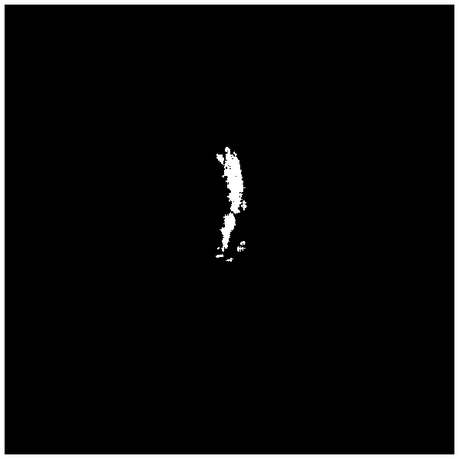}} &
\adjustbox{valign=c}{\includegraphics[width=0.1700\textwidth]{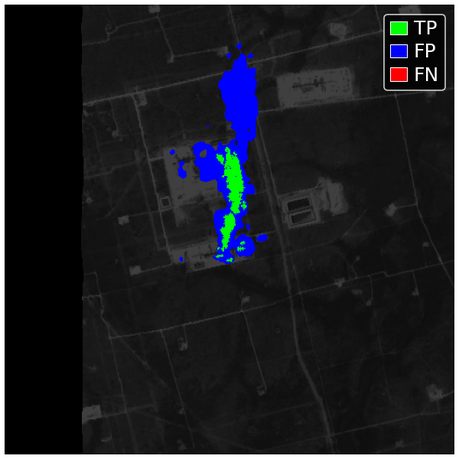}}
\\[2pt]

\adjustbox{valign=c}{\rowlabel{UNet (MNv2)}} &
\adjustbox{valign=c}{\includegraphics[width=0.1700\textwidth]{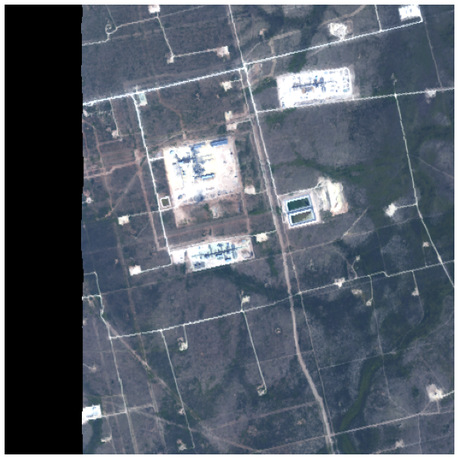}} &
\adjustbox{valign=c}{\includegraphics[width=0.1700\textwidth]{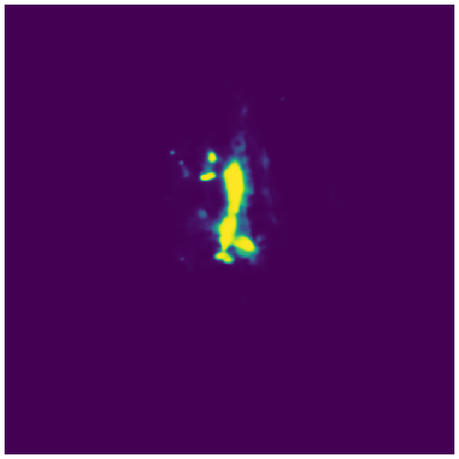}} &
\adjustbox{valign=c}{\includegraphics[width=0.1700\textwidth]{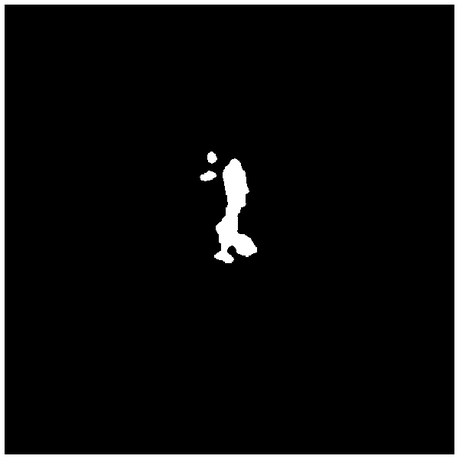}} &
\adjustbox{valign=c}{\includegraphics[width=0.1700\textwidth]{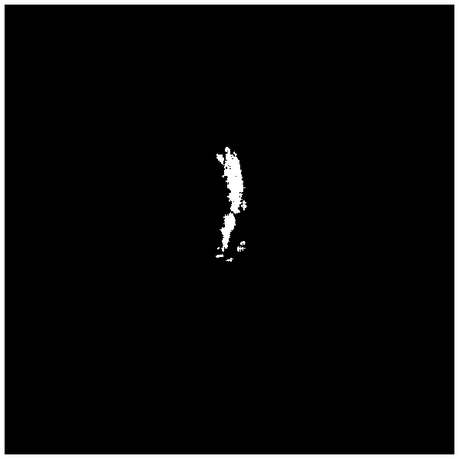}} &
\adjustbox{valign=c}{\includegraphics[width=0.1700\textwidth]{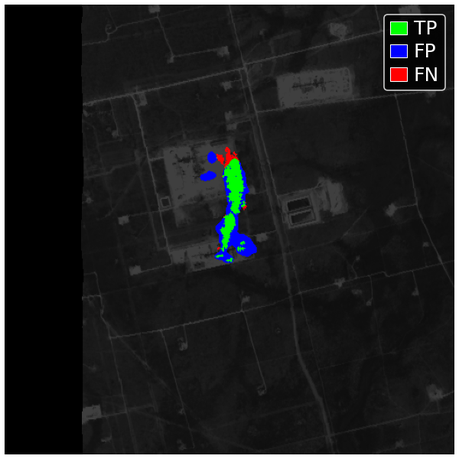}}
\\[2pt]

\adjustbox{valign=c}{\rowlabel{SegFormer (CU)}} &
\adjustbox{valign=c}{\includegraphics[width=0.1700\textwidth]{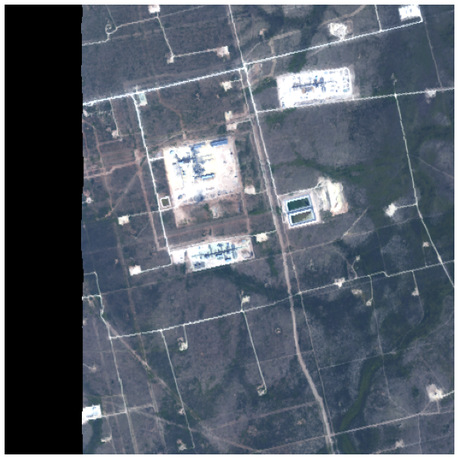}} &
\adjustbox{valign=c}{\includegraphics[width=0.1700\textwidth]{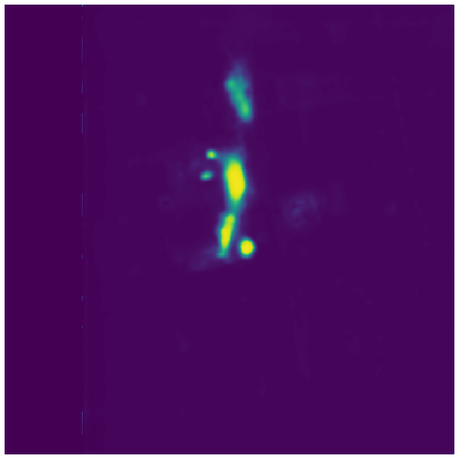}} &
\adjustbox{valign=c}{\includegraphics[width=0.1700\textwidth]{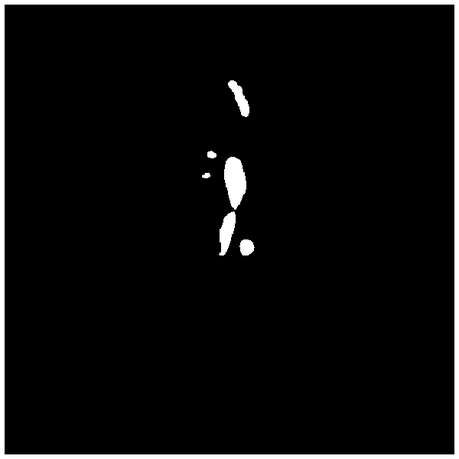}} &
\adjustbox{valign=c}{\includegraphics[width=0.1700\textwidth]{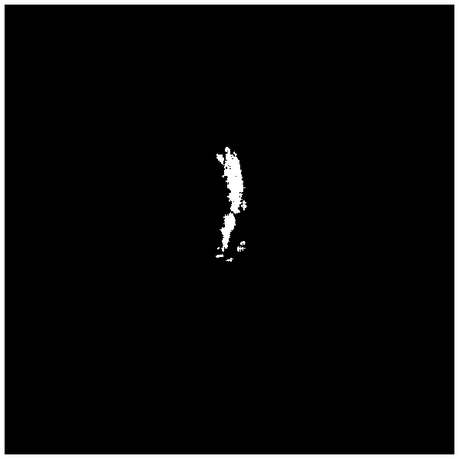}} &
\adjustbox{valign=c}{\includegraphics[width=0.1700\textwidth]{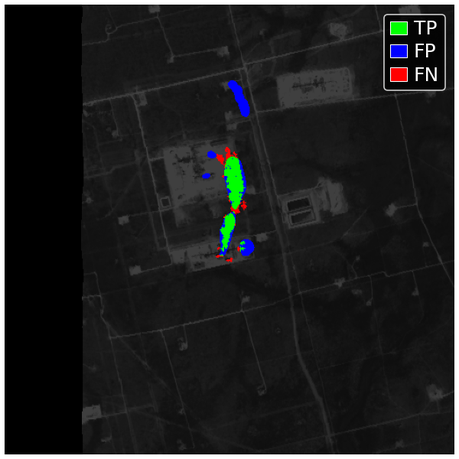}}
\\[2pt]

\adjustbox{valign=c}{\rowlabel{\flame{} (Ours)}} &
\adjustbox{valign=c}{\includegraphics[width=0.1700\textwidth]{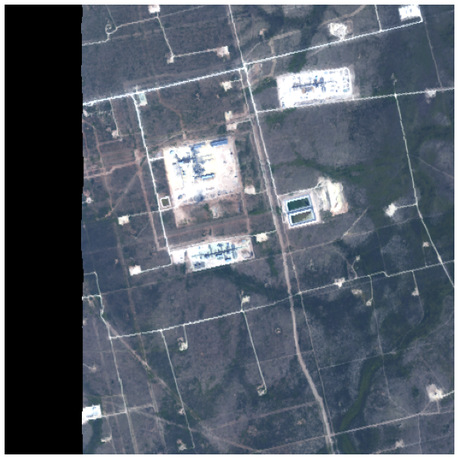}} &
\adjustbox{valign=c}{\includegraphics[width=0.1700\textwidth]{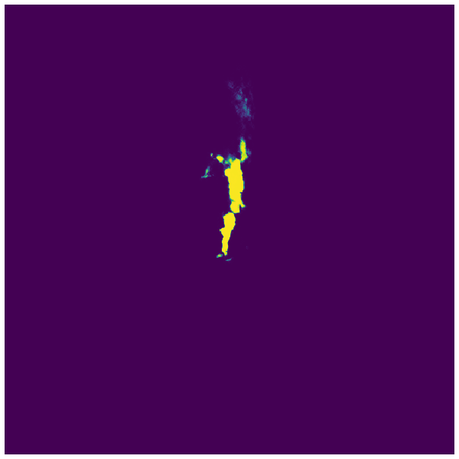}} &
\adjustbox{valign=c}{\includegraphics[width=0.1700\textwidth]{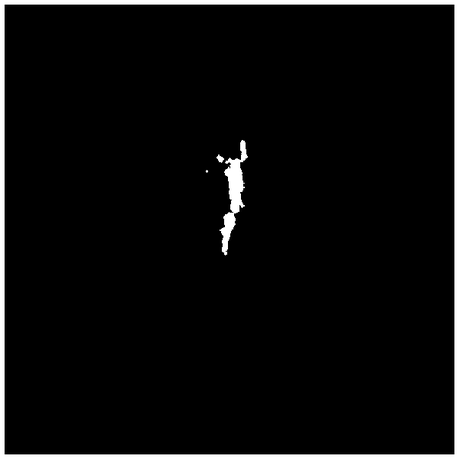}} &
\adjustbox{valign=c}{\includegraphics[width=0.1700\textwidth]{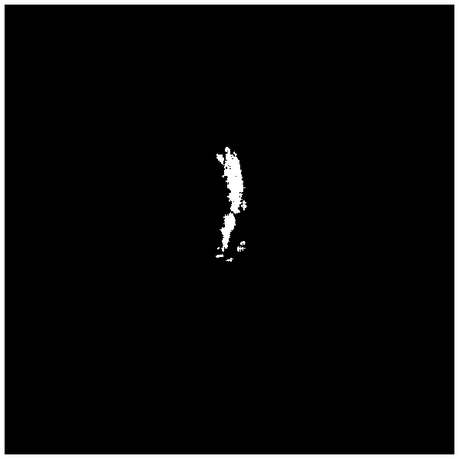}} &
\adjustbox{valign=c}{\includegraphics[width=0.1700\textwidth]{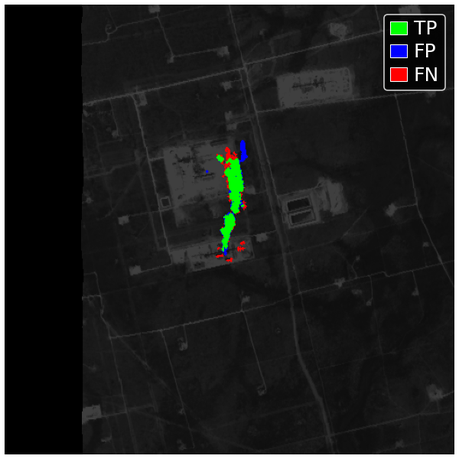}}
\\
\end{tabular}

\caption{Per-pixel outputs on a strong-plume test tile. Columns show the RGB tile, the probability map, the final binary mask, the ground-truth plume mask, and the TP/FP/FN overlay. From top to bottom, rows show MAG1C-tile, UNet+MAG1C-SAS, UNet with the MobileNetV2 encoder, SegFormer with the ConvUp decoder, and \flame{}. For the MAG1C-tile row, the probability-map column instead reports the matched-filter enhancement in ppm$\cdot$m using the top color bar; all learned models use sigmoid probabilities with the bottom color bar.}
\label{fig:qualitative_baselines_strong}
\end{figure*}

\paragraph{Detailed Strong-Plume Case} In Figure~\ref{fig:qualitative_baselines_strong}, the MAG1C-tile score activates broadly across the tile, with strong responses on roads and field boundaries that share a directional radiance pattern with the plume. UNet+MAG1C-SAS inherits parts of this background structure and produces a wide diffuse score around the plume body. The end-to-end UNet and SegFormer scores are concentrated near the plume but leak into adjacent surfaces, which translates into the over-extended binary masks visible in the third and fourth rows. The \flame{} score collapses onto a thin region that closely follows the annotated plume, and its binary mask correspondingly contains the smallest amount of background activation among the methods shown.

\clearpage

\begin{figure*}[t!]
\centering
\includegraphics[width=0.70\textwidth]{outputs/figures/fig_shared_colormap.pdf}
\vspace{2pt}

\setlength{\rowlabelw}{0.024\textwidth}

\newcommand{\rowlabel}[1]{%
  \makebox[\rowlabelw][c]{%
    \rotatebox[origin=c]{90}{\scriptsize #1}%
  }%
}

\setlength{\tabcolsep}{1pt}
\renewcommand{\arraystretch}{0.95}

\begin{tabular}{@{}c@{\hspace{1pt}}*{5}{c}@{}}
&
\scriptsize RGB &
\scriptsize Probability map &
\scriptsize Final binary mask &
\scriptsize GT &
\scriptsize TP/FP/FN
\\[2pt]

\adjustbox{valign=c}{\rowlabel{MAG1C-tile}} &
\adjustbox{valign=c}{\includegraphics[width=0.1700\textwidth]{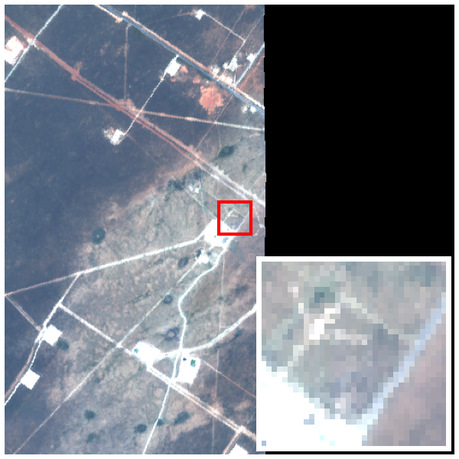}} &
\adjustbox{valign=c}{\includegraphics[width=0.1700\textwidth]{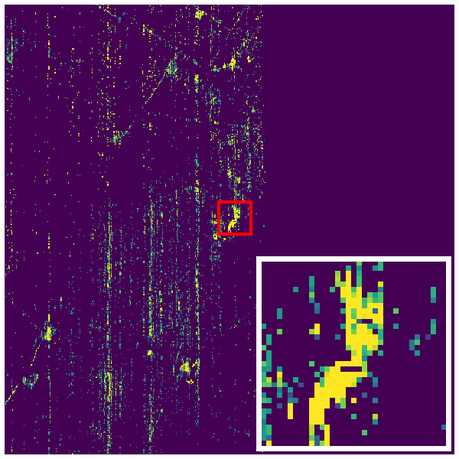}} &
\adjustbox{valign=c}{\includegraphics[width=0.1700\textwidth]{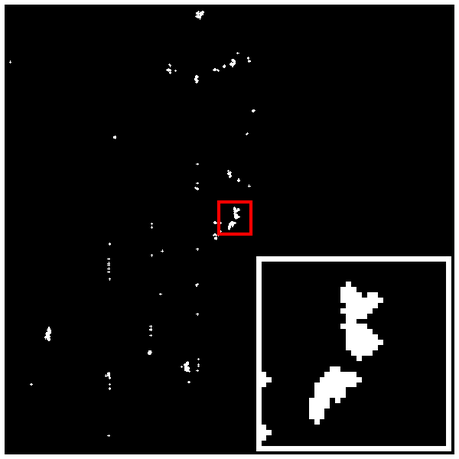}} &
\adjustbox{valign=c}{\includegraphics[width=0.1700\textwidth]{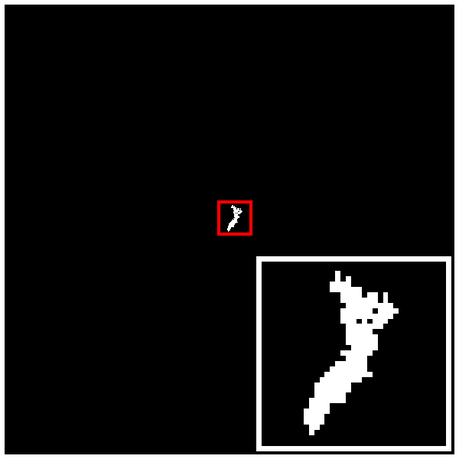}} &
\adjustbox{valign=c}{\includegraphics[width=0.1700\textwidth]{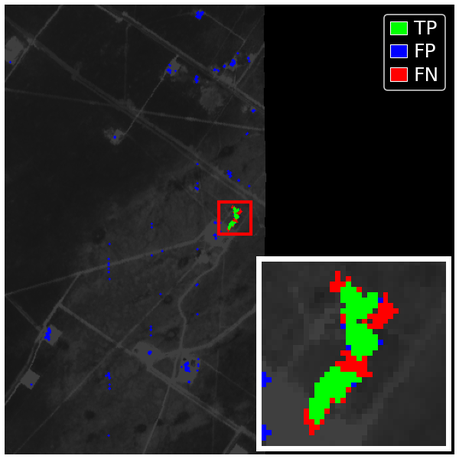}}
\\[2pt]

\adjustbox{valign=c}{\rowlabel{UNet+MAG1C-SAS}} &
\adjustbox{valign=c}{\includegraphics[width=0.1700\textwidth]{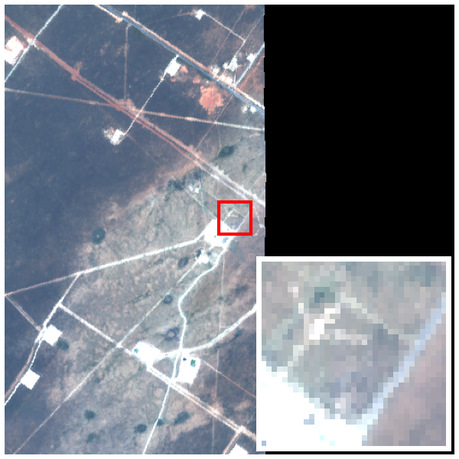}} &
\adjustbox{valign=c}{\includegraphics[width=0.1700\textwidth]{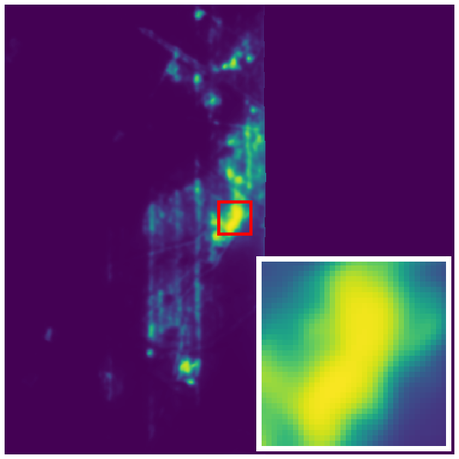}} &
\adjustbox{valign=c}{\includegraphics[width=0.1700\textwidth]{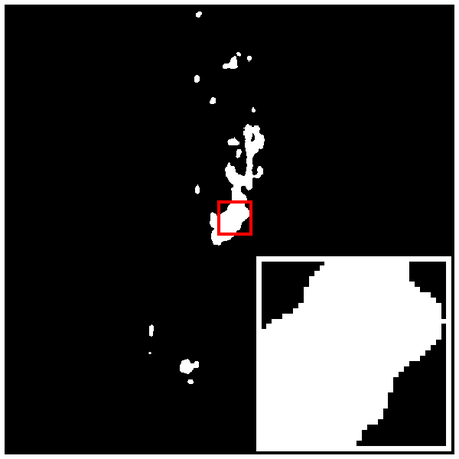}} &
\adjustbox{valign=c}{\includegraphics[width=0.1700\textwidth]{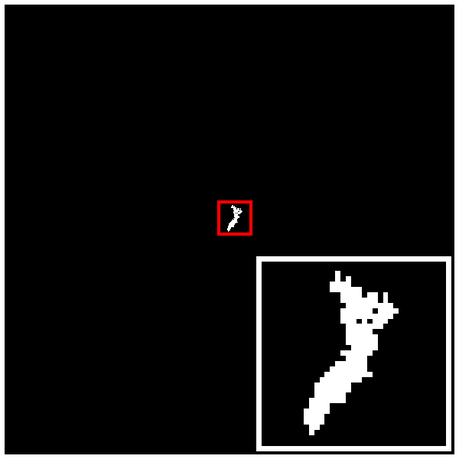}} &
\adjustbox{valign=c}{\includegraphics[width=0.1700\textwidth]{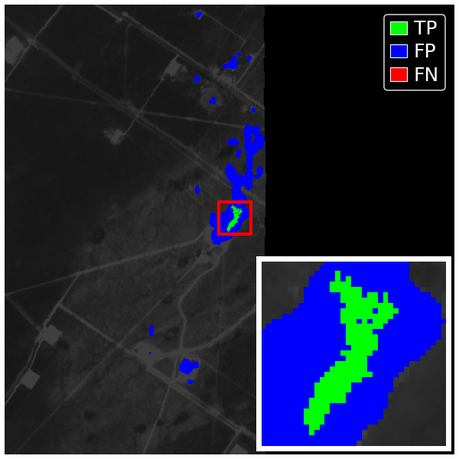}}
\\[2pt]

\adjustbox{valign=c}{\rowlabel{UNet (MNv2)}} &
\adjustbox{valign=c}{\includegraphics[width=0.1700\textwidth]{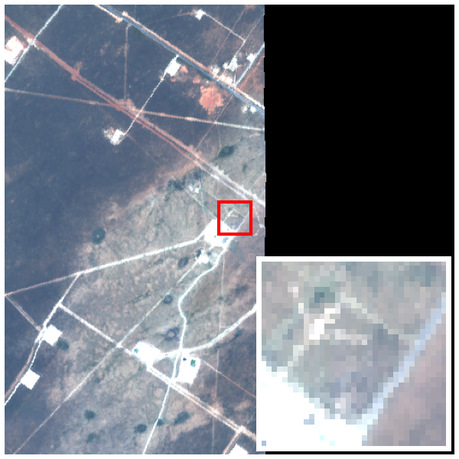}} &
\adjustbox{valign=c}{\includegraphics[width=0.1700\textwidth]{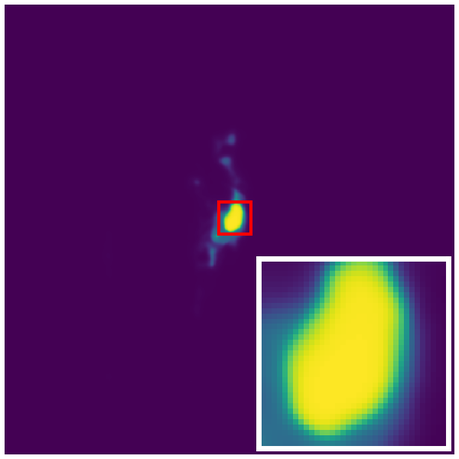}} &
\adjustbox{valign=c}{\includegraphics[width=0.1700\textwidth]{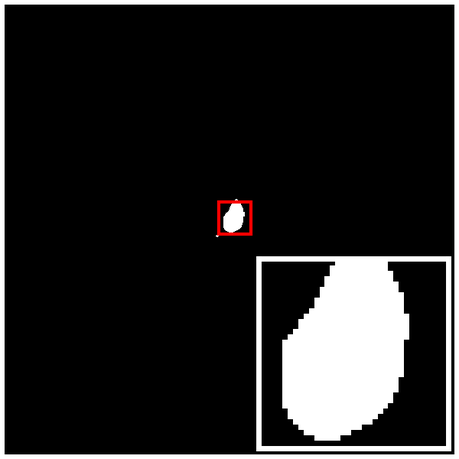}} &
\adjustbox{valign=c}{\includegraphics[width=0.1700\textwidth]{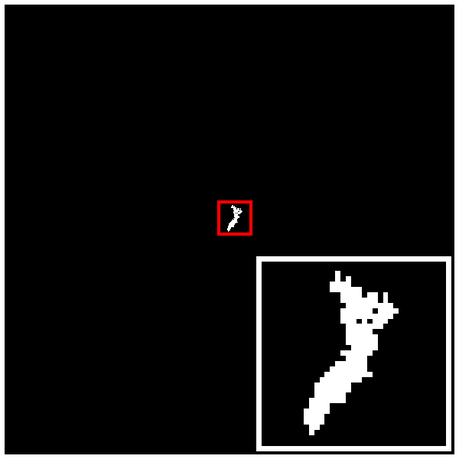}} &
\adjustbox{valign=c}{\includegraphics[width=0.1700\textwidth]{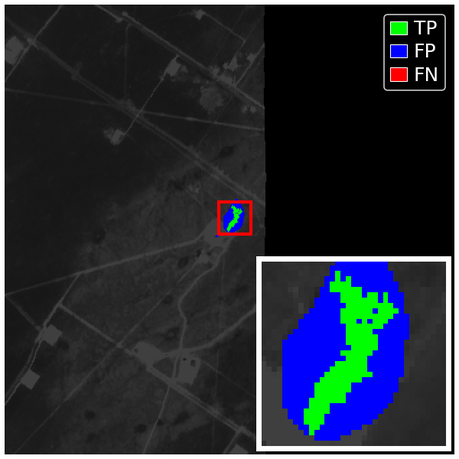}}
\\[2pt]

\adjustbox{valign=c}{\rowlabel{SegFormer (CU)}} &
\adjustbox{valign=c}{\includegraphics[width=0.1700\textwidth]{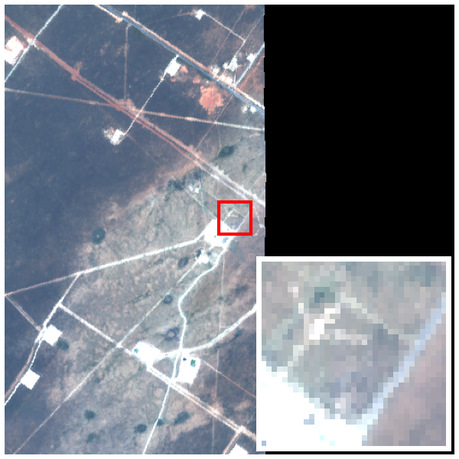}} &
\adjustbox{valign=c}{\includegraphics[width=0.1700\textwidth]{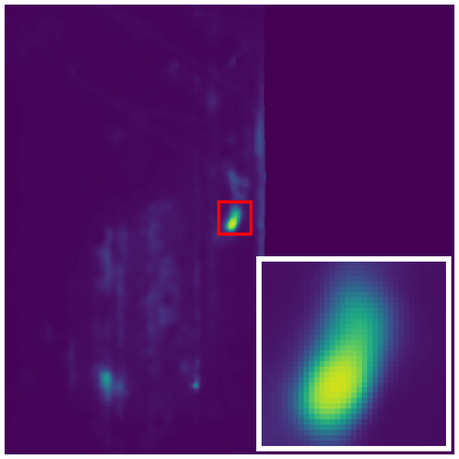}} &
\adjustbox{valign=c}{\includegraphics[width=0.1700\textwidth]{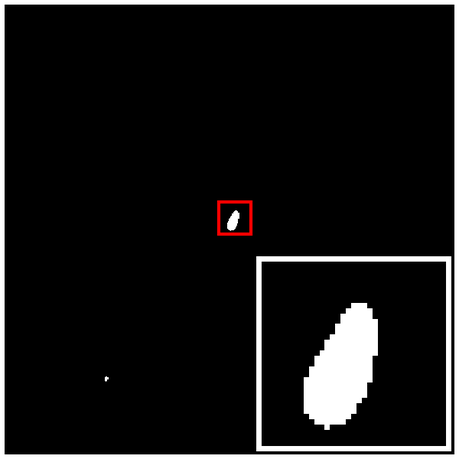}} &
\adjustbox{valign=c}{\includegraphics[width=0.1700\textwidth]{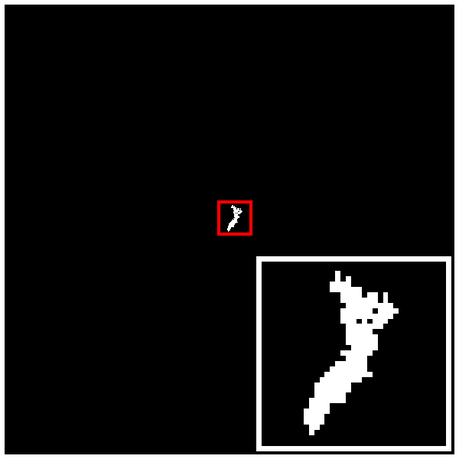}} &
\adjustbox{valign=c}{\includegraphics[width=0.1700\textwidth]{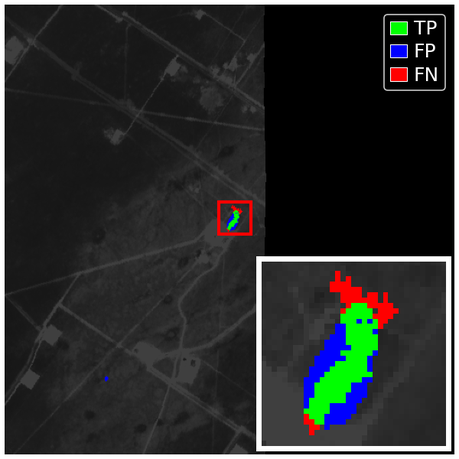}}
\\[2pt]

\adjustbox{valign=c}{\rowlabel{\flame{} (Ours)}} &
\adjustbox{valign=c}{\includegraphics[width=0.1700\textwidth]{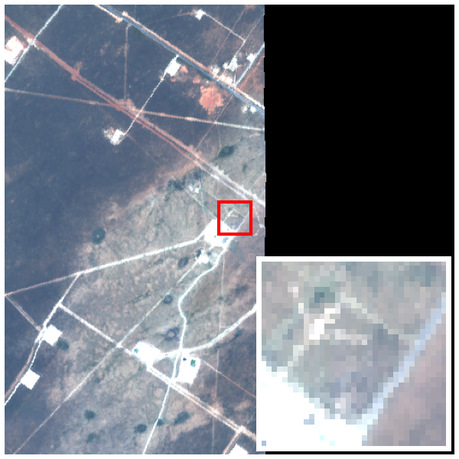}} &
\adjustbox{valign=c}{\includegraphics[width=0.1700\textwidth]{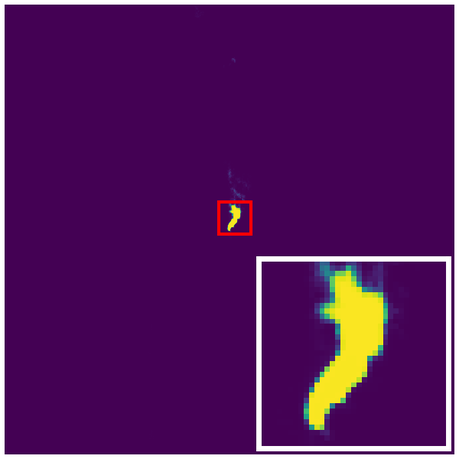}} &
\adjustbox{valign=c}{\includegraphics[width=0.1700\textwidth]{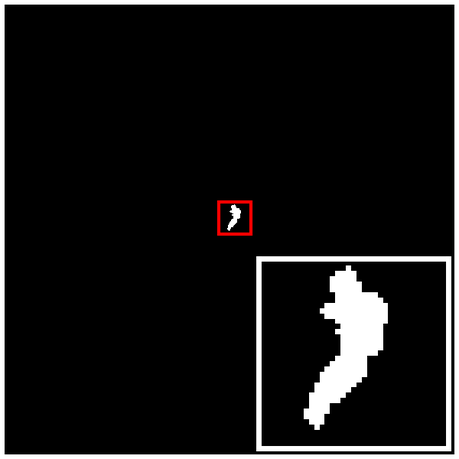}} &
\adjustbox{valign=c}{\includegraphics[width=0.1700\textwidth]{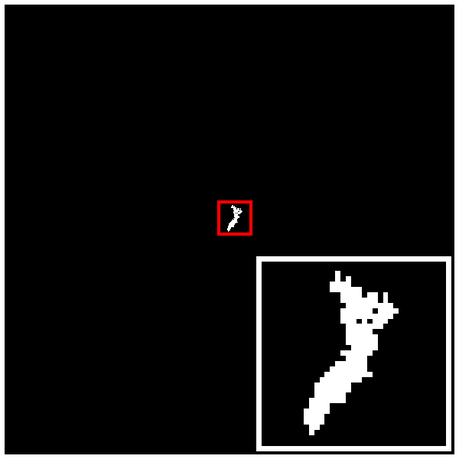}} &
\adjustbox{valign=c}{\includegraphics[width=0.1700\textwidth]{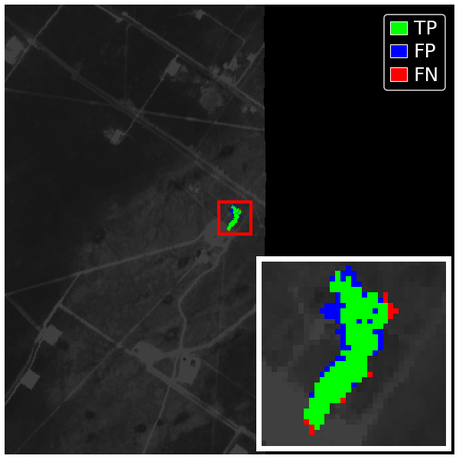}}
\\
\end{tabular}

\caption{Per-pixel outputs on a weak-plume test tile. Row and column layouts follow Figure~\ref{fig:qualitative_baselines_strong}. Each panel includes a zoomed inset of the plume region marked by the red rectangle.}
\label{fig:qualitative_baselines_weak}
\end{figure*}

\paragraph{Detailed Weak-Plume Case} Figure~\ref{fig:qualitative_baselines_weak} shows the same comparison on a tile where the methane signal is too faint to be reliably distinguished from radiance variations in the score map alone. MAG1C-tile leaves the plume region nearly indistinguishable from the surrounding clutter. UNet+MAG1C-SAS produces a strong but spatially diffuse response over the plume area, and the resulting mask spreads well beyond the annotated region. UNet with the MobileNetV2 encoder under-predicts the plume body, while SegFormer with the ConvUp decoder produces only fragmented activations that fail to capture the elongated plume shape. \flame{} yields a focused score that is largely confined to the plume region, and the corresponding mask preserves the elongated morphology visible in the ground truth.

\clearpage

\begin{figure*}[t!]
\centering
\includegraphics[width=0.70\textwidth]{outputs/figures/fig_shared_colormap.pdf}
\vspace{2pt}

\setlength{\rowlabelw}{0.024\textwidth}

\newcommand{\rowlabel}[1]{%
  \makebox[\rowlabelw][c]{%
    \rotatebox[origin=c]{90}{\scriptsize #1}%
  }%
}

\setlength{\tabcolsep}{1pt}
\renewcommand{\arraystretch}{0.95}

\begin{tabular}{@{}c@{\hspace{1pt}}*{5}{c}@{}}
&
\scriptsize RGB &
\scriptsize Probability map &
\scriptsize Final binary mask &
\scriptsize GT &
\scriptsize TP/FP/FN
\\[2pt]

\adjustbox{valign=c}{\rowlabel{MAG1C-tile}} &
\adjustbox{valign=c}{\includegraphics[width=0.1700\textwidth]{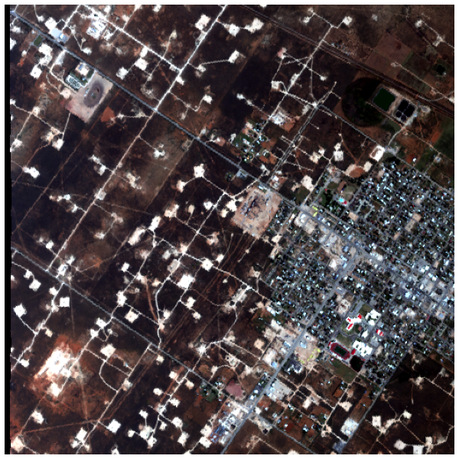}} &
\adjustbox{valign=c}{\includegraphics[width=0.1700\textwidth]{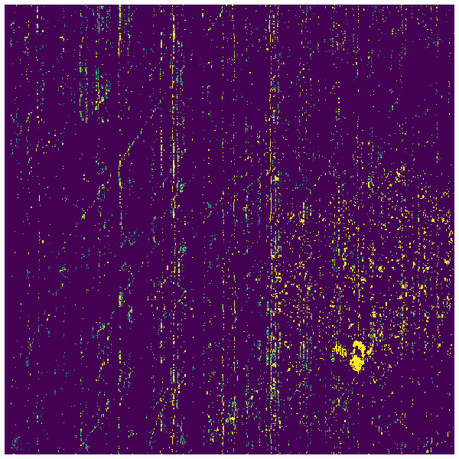}} &
\adjustbox{valign=c}{\includegraphics[width=0.1700\textwidth]{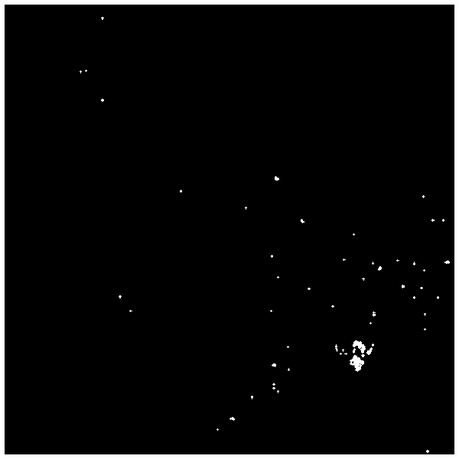}} &
\adjustbox{valign=c}{\includegraphics[width=0.1700\textwidth]{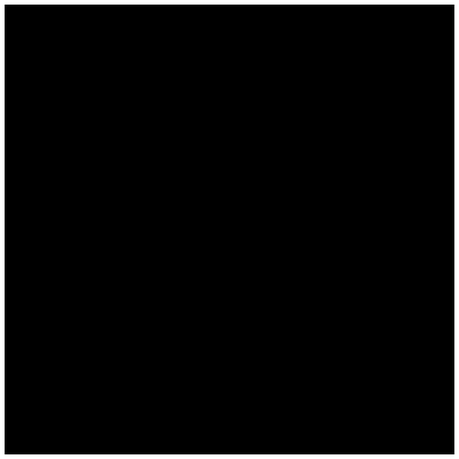}} &
\adjustbox{valign=c}{\includegraphics[width=0.1700\textwidth]{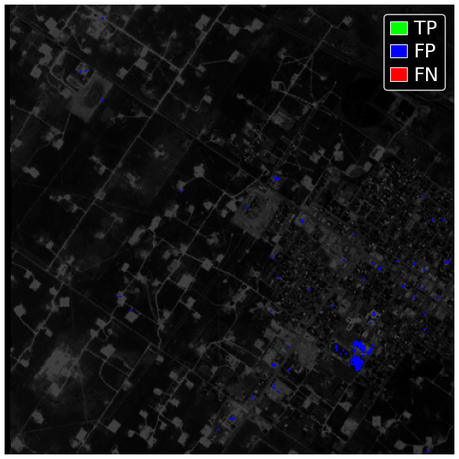}}
\\[2pt]

\adjustbox{valign=c}{\rowlabel{UNet+MAG1C-SAS}} &
\adjustbox{valign=c}{\includegraphics[width=0.1700\textwidth]{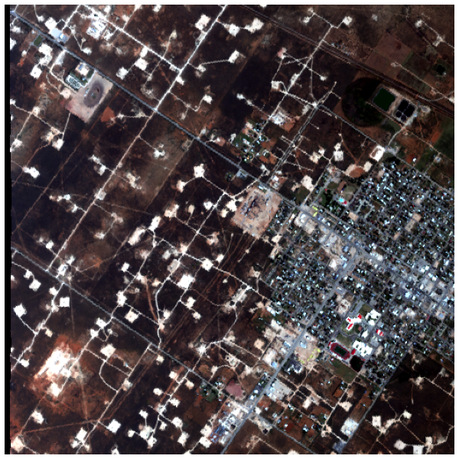}} &
\adjustbox{valign=c}{\includegraphics[width=0.1700\textwidth]{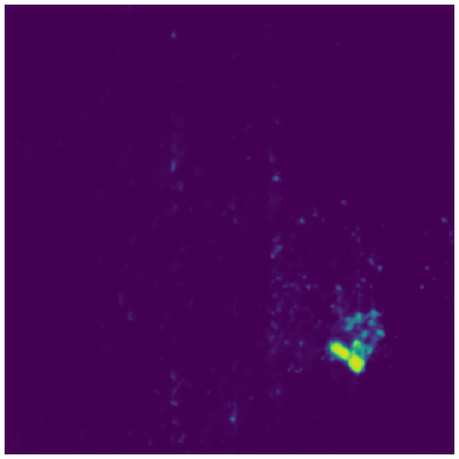}} &
\adjustbox{valign=c}{\includegraphics[width=0.1700\textwidth]{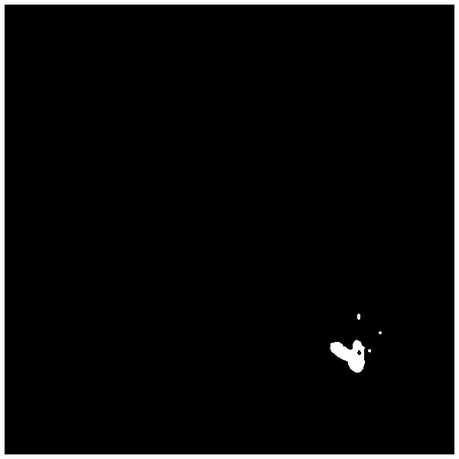}} &
\adjustbox{valign=c}{\includegraphics[width=0.1700\textwidth]{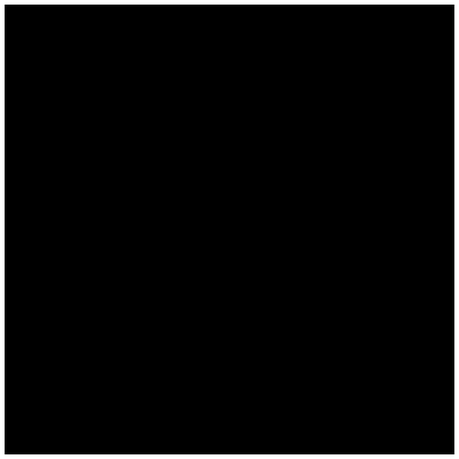}} &
\adjustbox{valign=c}{\includegraphics[width=0.1700\textwidth]{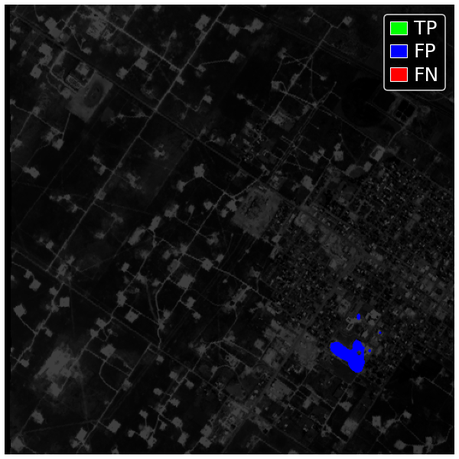}}
\\[2pt]

\adjustbox{valign=c}{\rowlabel{UNet (MNv2)}} &
\adjustbox{valign=c}{\includegraphics[width=0.1700\textwidth]{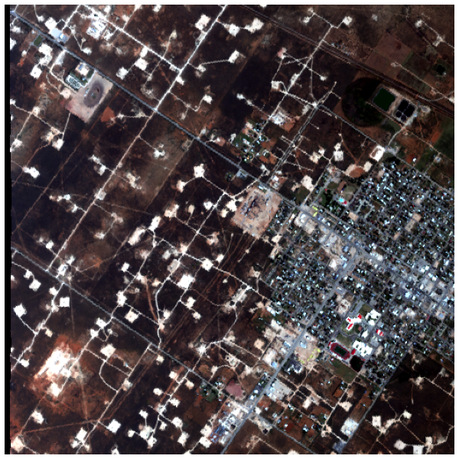}} &
\adjustbox{valign=c}{\includegraphics[width=0.1700\textwidth]{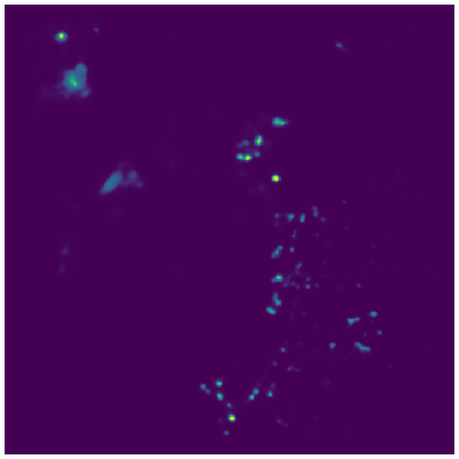}} &
\adjustbox{valign=c}{\includegraphics[width=0.1700\textwidth]{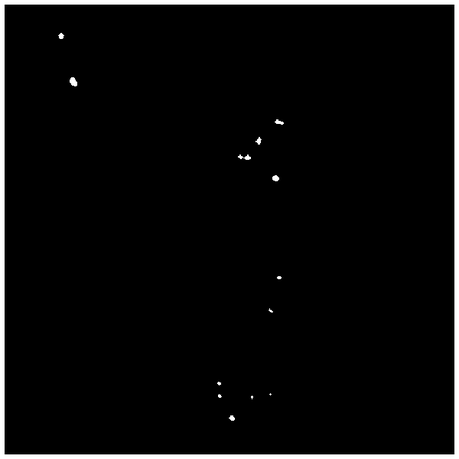}} &
\adjustbox{valign=c}{\includegraphics[width=0.1700\textwidth]{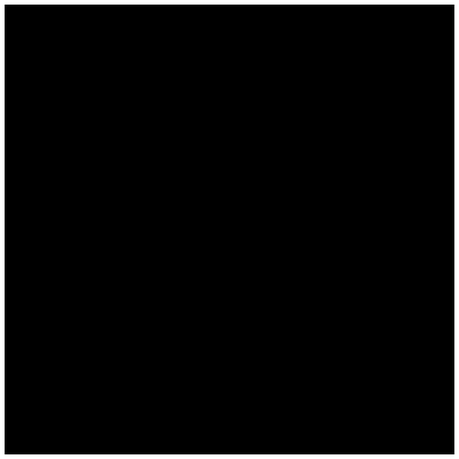}} &
\adjustbox{valign=c}{\includegraphics[width=0.1700\textwidth]{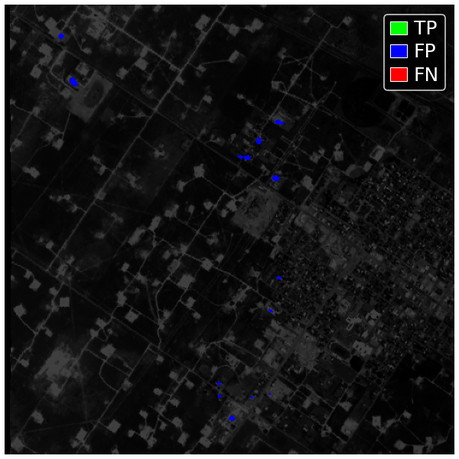}}
\\[2pt]

\adjustbox{valign=c}{\rowlabel{SegFormer (CU)}} &
\adjustbox{valign=c}{\includegraphics[width=0.1700\textwidth]{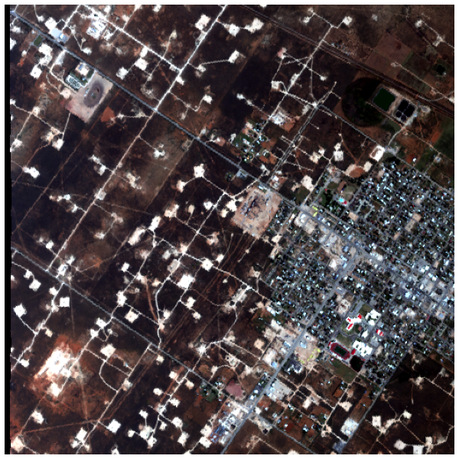}} &
\adjustbox{valign=c}{\includegraphics[width=0.1700\textwidth]{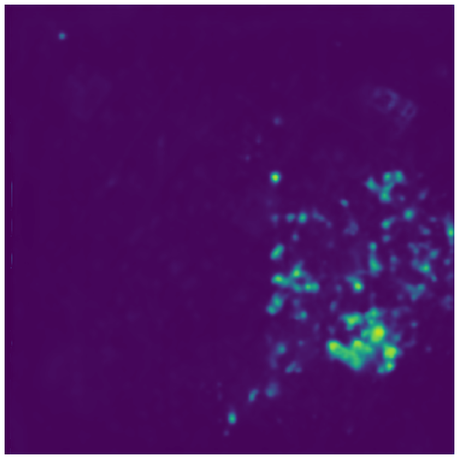}} &
\adjustbox{valign=c}{\includegraphics[width=0.1700\textwidth]{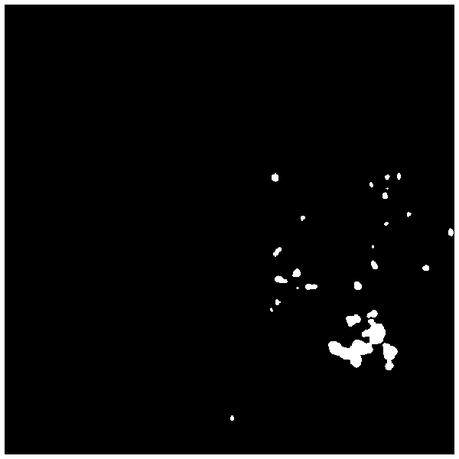}} &
\adjustbox{valign=c}{\includegraphics[width=0.1700\textwidth]{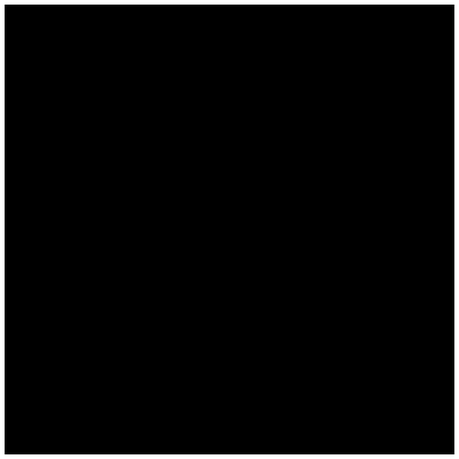}} &
\adjustbox{valign=c}{\includegraphics[width=0.1700\textwidth]{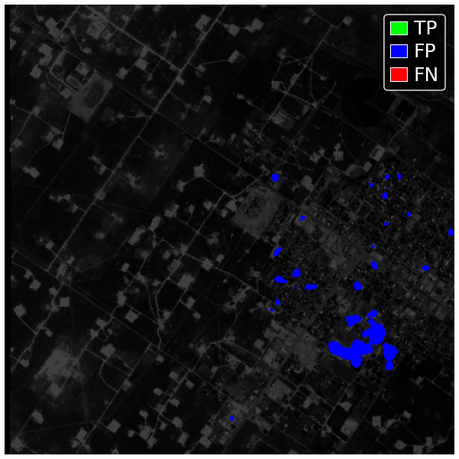}}
\\[2pt]

\adjustbox{valign=c}{\rowlabel{\flame{} (Ours)}} &
\adjustbox{valign=c}{\includegraphics[width=0.1700\textwidth]{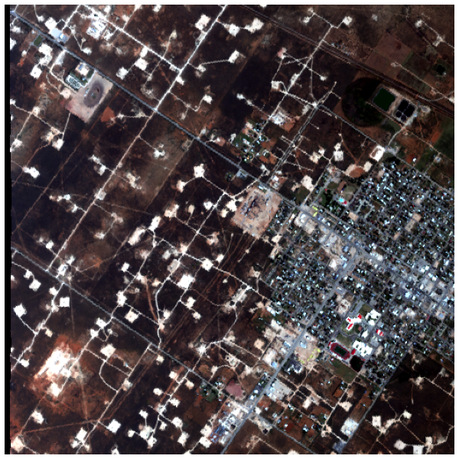}} &
\adjustbox{valign=c}{\includegraphics[width=0.1700\textwidth]{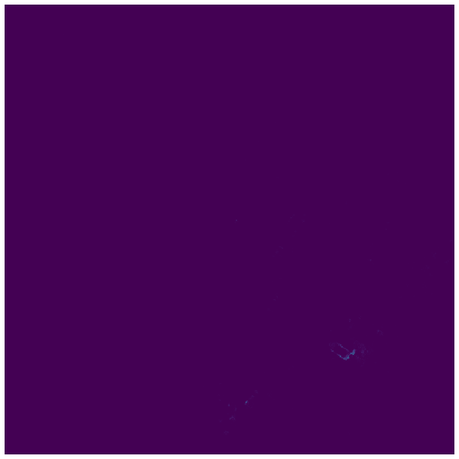}} &
\adjustbox{valign=c}{\includegraphics[width=0.1700\textwidth]{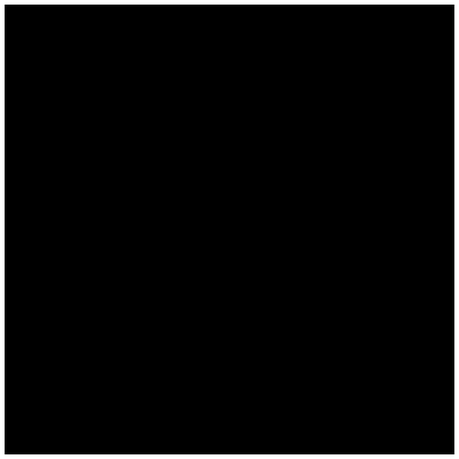}} &
\adjustbox{valign=c}{\includegraphics[width=0.1700\textwidth]{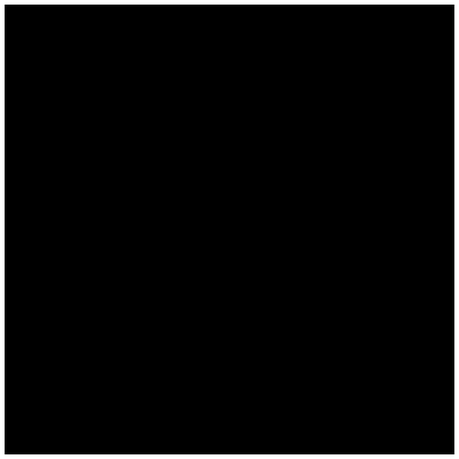}} &
\adjustbox{valign=c}{\includegraphics[width=0.1700\textwidth]{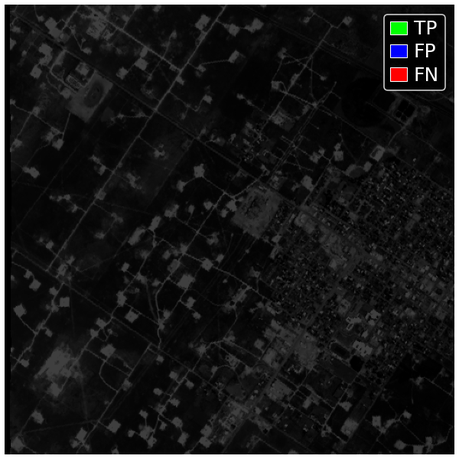}}
\\
\end{tabular}

\caption{Per-pixel outputs on a plume-free test tile. Row and column layouts follow Figure~\ref{fig:qualitative_baselines_strong}. The ground-truth column contains no plume pixels, and the TP/FP/FN overlay shows false positives in blue on a darkened RGB image.}
\label{fig:qualitative_baselines_non}
\end{figure*}
\paragraph{Detailed Plume-Free Case} Figure~\ref{fig:qualitative_baselines_non} shows the same comparison on a tile that contains no annotated plume but a high density of urban surface clutter. MAG1C-tile produces strong directional responses along streets and rooftops, and these responses propagate into the binary mask of UNet+MAG1C-SAS. UNet with the MobileNetV2 encoder and SegFormer with the ConvUp decoder both generate scattered detections in regions where surface materials share spectral cues with methane absorption. The \flame{} score remains close to zero across the entire tile, and its binary mask is empty, which is consistent with the lower pixel false positive rate reported in Section~\ref{sec:fpr_analysis}.


\end{document}